\documentclass{article}

\PassOptionsToPackage{numbers, compress}{natbib}

\usepackage[final]{neurips_2025}

\usepackage[utf8]{inputenc} %
\usepackage[T1]{fontenc}    %
\usepackage{hyperref}       %
\usepackage{url}            %
\usepackage{booktabs}       %
\usepackage{amsfonts}       %
\usepackage{nicefrac}       %
\usepackage{microtype}      %
\usepackage{xcolor}         %

\usepackage{amsmath}
\usepackage{dsfont}
\usepackage{subcaption} %
\usepackage{multirow} 

\usepackage{booktabs,subcaption,adjustbox,xcolor,colortbl}
\usepackage{amsthm}
\theoremstyle{plain}
\newtheorem{theorem}{Theorem}

\newtheorem{definition}{Definition}

\usepackage{enumitem}

\usepackage[colorinlistoftodos]{todonotes}

\usepackage{algorithm}
\usepackage{algorithmic}
\usepackage{wrapfig}

\DeclareMathOperator{\Var}{Var}
\DeclareMathOperator{\Cov}{Cov}

\usepackage{pifont}

\title{HyperMARL: Adaptive Hypernetworks for Multi-Agent RL}

\author{%
  Kale-ab Abebe Tessera\textsuperscript{1} \quad
  Arrasy Rahman\textsuperscript{2} \quad
  Amos Storkey\textsuperscript{1} \quad
  Stefano V. Albrecht\textsuperscript{3} \\
  \textsuperscript{1}School of Informatics, University of Edinburgh, Edinburgh, UK \\
  \textsuperscript{2}School of Computer Science, University of Texas at Austin, Austin, TX, USA \\
  \textsuperscript{3}DeepFlow, London, UK \\
  \texttt{\{k.tessera,a.storkey\}@ed.ac.uk, arrasy@utexas.edu}
}

\begin{document}

\maketitle

\begin{abstract}
Adaptive cooperation in multi-agent reinforcement learning (MARL) requires policies to express homogeneous, specialised, or mixed behaviours, yet achieving this adaptivity remains a critical challenge. While parameter sharing (PS) is standard for efficient learning, it notoriously suppresses the behavioural diversity required for specialisation. This failure is largely due to cross-agent gradient interference, a problem we find is surprisingly exacerbated by the common practice of \emph{coupling agent IDs with observations}. Existing remedies typically add complexity through altered objectives, manual preset diversity levels, or sequential updates -- raising a fundamental question: \emph{can shared policies adapt without these intricacies?} We propose a solution built on a key insight: an agent-conditioned hypernetwork can generate agent-specific parameters and \emph{decouple} observation- and agent-conditioned gradients, directly countering the interference from coupling agent IDs with observations. Our resulting method, \emph{HyperMARL}, avoids the complexities of prior work and empirically reduces policy gradient variance. Across diverse MARL benchmarks (22 scenarios, up to 30 agents), HyperMARL achieves performance competitive with six key baselines while preserving behavioural diversity comparable to non-parameter sharing methods, establishing it as a versatile and principled approach for adaptive MARL. The code is publicly available at \url{https://github.com/KaleabTessera/HyperMARL}.
\end{abstract}

\section{Introduction}

Specialist and generalist behaviours are critical to collective intelligence, enhancing performance and adaptability in both natural and artificial systems \citep{woolley2015collective,smith2008genetic,surowiecki2004wisdom,kassen2002experimental,williams1998demography}. In Multi-Agent Reinforcement Learning (MARL) \cite{albrecht2024multi}, this translates to a critical need for policies that can flexibly exhibit specialised, homogeneous, or mixed behaviours to meet diverse task demands \cite{li2021celebrating,bettini2024controlling}.
\begin{figure}[t]
    \centering\includegraphics[width=0.55\linewidth]{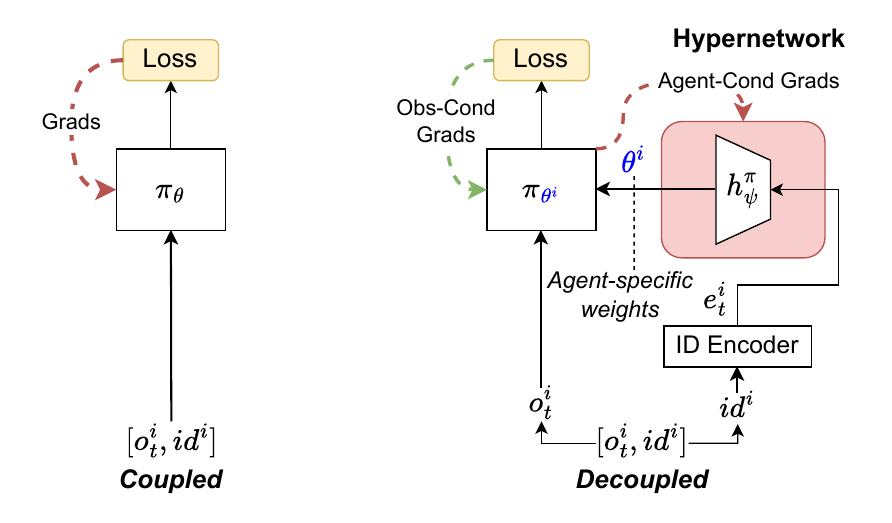}
    \caption{\textit{HyperMARL Policy Architecture.} Common agent-ID conditioned shared MARL policy (FuPS+ID, left) vs HyperMARL (right), which uses an agent-conditioned hypernetwork to generate agent-specific weights and \emph{decouples} observation- and agent-conditioned gradients.}
    \label{fig:hypermarl}
    \vspace{-12pt}
\end{figure}

Optimal MARL performance thus hinges on being able to represent the required behaviours. While No Parameter Sharing (NoPS) \citep{lowe2017multi} enables specialisation by using distinct per-agent networks, it suffers from significant computational overhead and sample inefficiency~\citep{christianos2021scaling}. Conversely, Full Parameter Sharing (FuPS)~\citep{tan1993multi,gupta2017cooperative,foerster2016learning}, which trains a single shared network, improves efficiency but typically struggles to foster the behavioural diversity necessary for many complex tasks~\citep{kim2023parameter,fu2022revisiting,li2021celebrating}.

This failure of FuPS, particularly for diverse behaviours, was hypothesised to be gradient interference among agents, whereby their updates negatively impact each other’s learning~\citep{christianos2021scaling,JMLR:v25:23-0488}. We not only empirically validate this hypothesis but also demonstrate a critical insight: this conflict is significantly exacerbated by the common practice of \emph{coupling observations with agent IDs} within a shared network (Fig.~\ref{fig:hypermarl} for coupling, Sec.~\ref{subsec:grad_conflicts} for results).

Balancing FuPS efficiency with the capacity for diverse behaviours therefore remains a central open problem in MARL. Prior works have explored intrinsic-rewards~\citep{li2021celebrating,pmlr-v139-jiang21g}, role-based allocations~\citep{wang2020roma,wang2020rode}, specialised architectures~\citep{kim2023parameter,li2024kaleidoscope,bettini2024controlling}, sequential updates~\citep{JMLR:v25:23-0488}, and sharing parameters within clusters of agents \cite{christianos2021scaling}. However, these remedies introduce their own intricacies: they often alter the learning objective, require prior knowledge of optimal diversity levels, necessitate maintaining agent-specific parameters or require sequential updates. This raises a fundamental question: \emph{Can we design a shared MARL architecture that flexibly supports both specialised and homogeneous behaviours—without altered learning objectives, manual preset diversity levels, or sequential updates?}

Guided by our observation-ID coupling insight, we propose \emph{HyperMARL}, a novel agent-conditioned hypernetwork~\citep{ha2016hypernetworks} architecture. HyperMARL generates per-agent weights on the fly (Fig.~\ref{fig:hypermarl}) and explicitly \emph{decouples} observation- and agent-conditioned gradients (Sec.~\ref{subsec:decoupling_grads}). This choice is motivated by hypernetworks' proven effectiveness at resolving gradient conflicts in multi-task RL~\citep{navon2020learning} and continual learning~\citep{Oswald2020Continual}. Our work establishes their effectiveness for the problem of cross-agent interference in MARL. Indeed, HyperMARL empirically attains lower policy gradient variance than FuPS, and we show this decoupling is critical for specialisation (Sec.~\ref{subsec:specialised_policy_learning},~\ref{sec:ablations}), confirming its role in mitigating interference.

We validate HyperMARL on diverse MARL benchmarks -- including Dispersion and Navigation (VMAS)~\citep{bettini2022vmas}, Multi-Agent MuJoCo (MAMuJoCo)~\citep{peng2021facmac}, SMAX~\citep{rutherford2024jaxmarl}, and Blind-Particle Spread (BPS)~\citep{christianos2021scaling} -- across environments with two to thirty agents that require homogeneous, heterogeneous, or mixed behaviours. HyperMARL consistently matches or outperforms NoPS, FuPS, and diversity-promoting methods such as Diversity Control (DiCo)~\citep{bettini2024controlling}, Heterogeneous-Agent Proximal Policy
Optimisation (HAPPO)~\citep{JMLR:v25:23-0488}, Kaleidoscope~\citep{li2024kaleidoscope} and Selective Parameter Sharing (SePS)~\citep{christianos2021scaling}, while achieving NoPS-level behavioural diversity while using a shared architecture. 

Our contributions are as follows: 
\begin{itemize}[leftmargin=*] 
\item We identify that cross-agent gradient interference in shared policies is critically exacerbated by the common practice of coupling agent IDs with observations (Sec.~\ref{subsec:grad_conflicts}). 
\item We propose \textit{HyperMARL} (Sec.~\ref{sec:hypermarl}), an agent-conditioned hypernetwork architecture, to test the hypothesis that explicitly \emph{decoupling} these gradients enables adaptive (diverse, homogeneous, or mixed) behaviours without the complexities of prior remedies (e.g., altered objectives, preset diversity levels, or sequential updates).
\item Our extensive evaluation (Sec.~\ref{sec:experiments}) across 22 diverse scenarios (up to 30 agents) shows HyperMARL achieves competitive returns against six strong baselines, while achieving NoPS-level behavioural diversity. We further show this decoupling is empirically linked to reduced policy gradient variance and is critical for specialisation (Sec.~\ref{subsec:specialised_policy_learning}; Sec.~\ref{sec:ablations}).
\end{itemize}

\section{Background}

We formulate the fully cooperative multi-agent systems addressed in our work as a Dec-POMDP~\citep{oliehoek2016concise}. A Dec-POMDP is a tuple, $\langle \mathbb{I}, \mathbb{S},\{\mathbb{A}^i\}_{i\in \mathbb{I}},R,\{\mathbb{O}^i\}_{i\in \mathbb{I}}, O, T, \rho_0, \gamma \rangle$, where $\mathbb{I}$ is the set of agents of size $n=|\mathbb{I}|$, $\mathbb{S}$ is the set of global states with an initial state distribution $\rho_0$, $\mathbb{A}^i$ is the action space for agent $i$ where $\mathbb{A}=\times_i \mathbb{A}^i$ is the joint action space, $R: \mathbb{S} \times \mathbb{A} \to \mathbb{R}$ is the shared reward function, $\mathbb{O}^i$ is the observation space for agent $i$ with the joint observation space $\mathbb{O}=\times_i \mathbb{O}^i$, $O:\mathbb{O} \times \mathbb{A} \times \mathbb{S} \to [0,1]$ is the probability of joint observation $\mathbf{o} \in \mathbb{O}$, i.e. $O(\mathbf{o},\mathbf{a},s) = \Pr(\mathbf{o_t}|s_t,\mathbf{a_{t-1}})$, $T: \mathbb{S} \times \mathbb{A} \times \mathbb{S} \to [0,1]$ is the state transition function i.e. $T\left(s, \mathbf{a}, s^{\prime}\right)=\Pr\left(s_{t+1} \mid s_t, \mathbf{a}_t\right)$ and $\gamma$ is the discount factor.

In this setting, each agent $i$ receives a partial observation $o_t^i \in \mathbb{O}^i$. These observations are accumulated into an action-observation history $h_t^i=\left(o_0^i, a_0^i, \ldots, o_{t-1}^i,a_{t-1}^i,o_{t}^i\right)$. Each agent $i$ acts based on their decentralised policies $\pi^i(a^i|h^i)$. The joint history and joint action are defined as follows $\mathbf{h_t}=(h_t^1,\ldots,h_t^n)$ and $\mathbf{a_t}=(a_t^1,\ldots,a_t^n)$. The goal is to learn an optimal joint policy $\boldsymbol{\pi^*}=(\pi^{1*},...,\pi^{n*})$ that maximizes the expected discounted return as follows,\footnote{We use simplified notation here omitting explicit dependence on state transitions and distributions for brevity.} $\boldsymbol{\pi}^* = \arg\max_{\boldsymbol{\pi}}\mathbb{E}_{s_0 \sim \rho_0,\; \mathbf{h} \sim \boldsymbol{\pi}}\left[ G(\mathbf{h}) \right]$, where $G(\mathbf{h})=\sum_{t=0}^\infty \gamma^t R(s_t, \mathbf{a}_t)$.

\textbf{Specialised Policies and Environments.}
We say an environment is \emph{specialised} if its optimal
joint policy contains at least two distinct, non-interchangeable
agent policies ( Def.~\ref{def:specialised_env} in App.~\ref{app:spec_policy_env}).
Under this mild condition, tasks such as Dispersion (\ref{results:dispersion}) or our
Specialisation Game (\ref{sec:spec_game}) require agents to learn complementary
roles rather than identical behaviours.

\section{Are Independent or Fully Shared Policies Enough?}\label{sec:motivation}

Standard independent (\textbf{NoPS}) and fully parameter-shared (\textbf{FuPS}) policies face inherent trade-offs in MARL. NoPS allows for uninhibited agent specialisation but can be sample inefficient and computationally expensive. FuPS, often conditioned with an agent ID (FuPS+ID), is more efficient but can struggle when agents must learn diverse behaviours \cite{christianos2021scaling}. This section investigates the limitations of these common policy architectures.

\begin{figure}[t]
  \centering
  \setlength{\tabcolsep}{3pt}
  \setlength{\abovecaptionskip}{4pt}
  \setlength{\belowcaptionskip}{6pt}
  \captionsetup[subfigure]{font=small,labelfont=bf,justification=centering}

  \begin{subfigure}[t]{0.48\textwidth}
    \centering\scriptsize
    \renewcommand{\arraystretch}{0.9}
    \begin{minipage}[b]{0.45\linewidth}
      \centering
      \begin{tabular}{cc|cc}
        \multicolumn{2}{c}{} & \multicolumn{2}{c}{\bfseries P2} \\
        \cmidrule(lr){3-4}
        &      & A                  & B                \\
        \cmidrule(lr){2-4}
        \multirow{2}{*}{\rotatebox[origin=c]{90}{\bfseries P1}}
          & A & \cellcolor{gray!10}(0.5,0.5) & \cellcolor{blue!20}(1,1)\\
        \cmidrule(lr){2-4}
          & B & \cellcolor{blue!20}(1,1)     & \cellcolor{gray!10}(0.5,0.5)\\
        \bottomrule
      \end{tabular}
      \caption*{Two‐Player Payoff matrix}
    \end{minipage}
    \hfill
    \begin{minipage}[b]{0.45\linewidth}
      \centering
      \includegraphics[height=2.8cm]{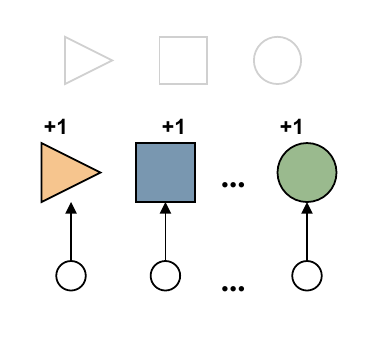}
      \caption*{$N$‐player Interaction}
    \end{minipage}
    \subcaption{Specialisation Game}
    \label{fig:spec-game}
  \end{subfigure}%
  \hfill
  \begin{subfigure}[t]{0.48\textwidth}
    \centering\scriptsize
    \renewcommand{\arraystretch}{0.9}
    \begin{minipage}[b]{0.45\linewidth}
      \centering
      \begin{tabular}{cc|cc}
        \multicolumn{2}{c}{} & \multicolumn{2}{c}{\bfseries P2} \\
        \cmidrule(lr){3-4}
        &      & A                & B                 \\
        \cmidrule(lr){2-4}
        \multirow{2}{*}{\rotatebox[origin=c]{90}{\bfseries P1}}
          & A & \cellcolor{blue!20}(1,1)     & \cellcolor{gray!10}(0.5,0.5)\\
        \cmidrule(lr){2-4}
          & B & \cellcolor{gray!10}(0.5,0.5) & \cellcolor{blue!20}(1,1)    \\
        \bottomrule
      \end{tabular}
      \caption*{Two‐Player Payoff matrix}
    \end{minipage}
    \hfill
    \begin{minipage}[b]{0.45\linewidth}
      \centering
      \includegraphics[height=2.8cm]{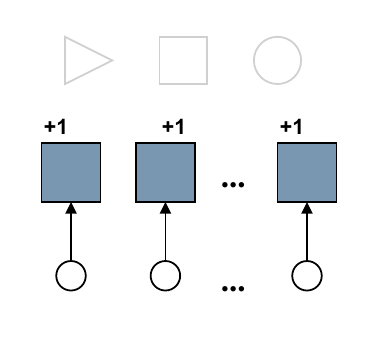}
      \caption*{$N$‐player Interaction}
    \end{minipage}
    \subcaption{Synchronisation Game}
    \label{fig:sync-game}
  \end{subfigure}

  \caption{\emph{Specialisation and Synchronisation Games.} The Specialisation game (left), which encourages \emph{distinct} actions, and the Synchronisation game (right), where rewards encourage \emph{identical} actions. Depicted are their two-player payoff matrices (pure Nash equilibria in blue) and $N$-player interaction schematics. While simple in form, these games are challenging MARL benchmarks due to non-stationarity and exponentially scaling observation spaces (temporal version).}
  \label{fig:game-setup}
   \vspace{-12pt}
\end{figure}

To probe these limitations, we introduce two illustrative environments: the \textbf{Specialisation Game}, rewarding \emph{distinct} actions, and the \textbf{Synchronisation Game}, rewarding \emph{identical} actions. Both are inspired by prior work~\citep{christianos2023pareto,fu2022revisiting,bettini2022vmas,osborne1994course} and extended here to $N$-agent and temporal settings where agents observe prior joint actions (see Appendix~\ref{appen:spec_game} for full definitions).

\subsection{Limitations of Fully Shared and Independent Policies}

\begin{table}[t]
\centering
\scriptsize
\setlength{\tabcolsep}{4pt}
\caption{Average evaluation reward (mean $\pm$ 95\% CI) for \textit{temporal} Specialisation vs. Synchronisation using REINFORCE (10 seeds). \textbf{Bold}: highest mean, no CI overlap. Neither fully shared nor independent policies consistently achieve the highest mean reward.}
\label{tab:spec_sync}
\begin{tabular}{ccccccc}
\toprule
& \multicolumn{3}{c|}{Specialisation} & \multicolumn{3}{c}{Synchronisation} \\
\cmidrule{2-7}
\#Ag & NoPS & FuPS & FuPS+ID & NoPS & FuPS & FuPS+ID \\
\midrule
2 & \textbf{0.88}$\pm$0.09 & 0.50$\pm$0.00 & 0.64$\pm$0.10 & 0.83$\pm$0.12 & \textbf{1.00}$\pm$0.00 & 0.91$\pm$0.09 \\
4 & \textbf{0.74}$\pm$0.08 & 0.25$\pm$0.00 & 0.40$\pm$0.07 & 0.32$\pm$0.03 & \textbf{1.00}$\pm$0.00 & 0.67$\pm$0.15 \\
8 & \textbf{0.68}$\pm$0.02 & 0.12$\pm$0.00 & 0.25$\pm$0.03 & 0.14$\pm$0.00 & \textbf{1.00}$\pm$0.00 & 0.54$\pm$0.10 \\
16 & \textbf{0.64}$\pm$0.01 & 0.06$\pm$0.00 & 0.13$\pm$0.02 & 0.07$\pm$0.00 & \textbf{1.00}$\pm$0.00 & 0.55$\pm$0.14 \\
\bottomrule
\end{tabular}
\end{table}

\begin{figure*}[t]
  \centering
  \captionsetup[subfigure]{justification=centering,singlelinecheck=false,skip=3pt}

  \begin{subfigure}[t]{0.48\textwidth}
    \centering %
    \includegraphics[width=\linewidth,height=3.5cm,keepaspectratio]{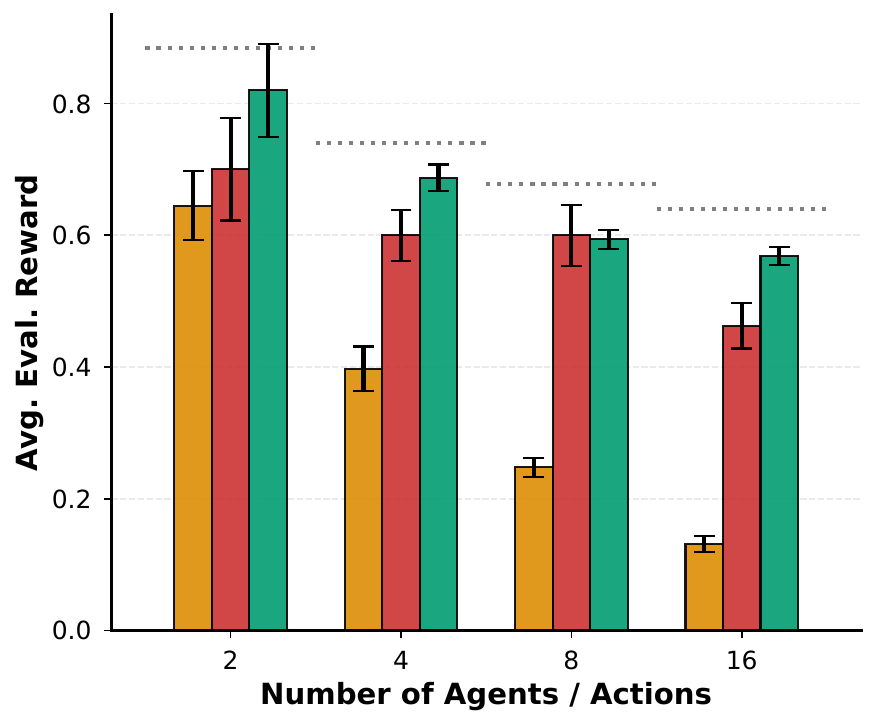}
    \caption{Avg.\ evaluation reward}
    \label{subfig:rewards}
  \end{subfigure}\hfill
  \begin{subfigure}[t]{0.48\textwidth}
    \centering %
    \includegraphics[width=\linewidth,height=3.5cm,keepaspectratio]{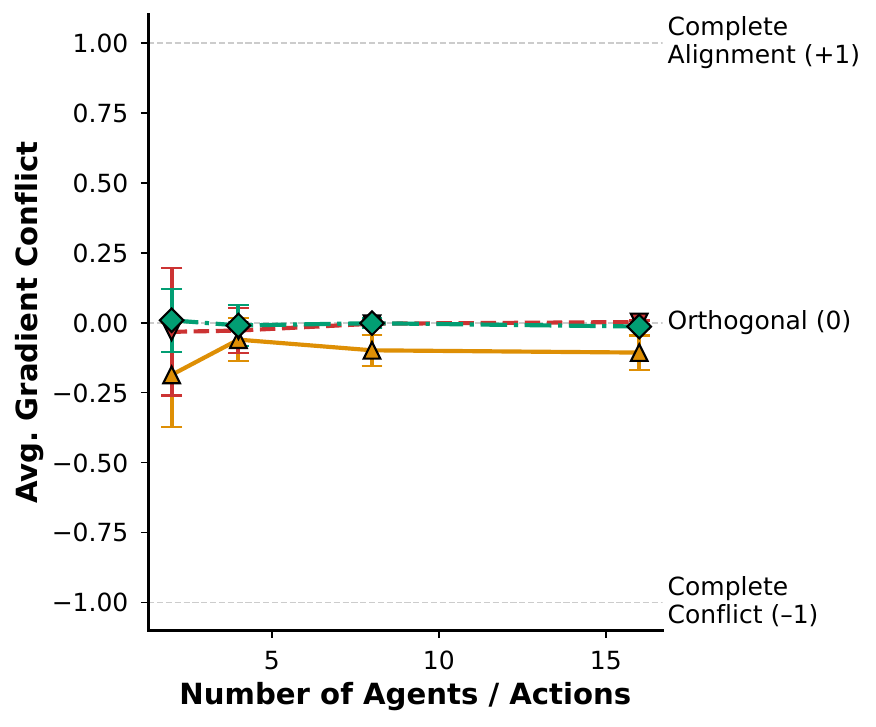}
    \caption{Avg.\ gradient conflict}
    \label{subfig:grad_conflict}
  \end{subfigure}

  \begin{subfigure}[t]{0.70\textwidth}
    \centering
    \includegraphics[width=\linewidth]{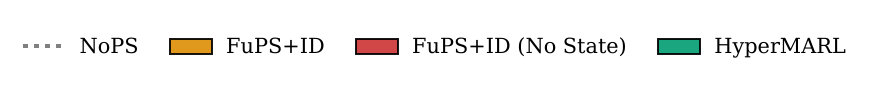}
  \end{subfigure}

   \caption{\textit{Multi-agent policy gradient methods in the Specialisation environment.} The FuPS+ID (No State) ablation outperforms FuPS+ID, showing near-orthogonal gradients (\subref{subfig:grad_conflict}), indicating that observation–ID decoupling is important. HyperMARL (MLP) enables this decoupling while leveraging state information, and achieves better performance and reduced gradient conflict than FuPS+ID.}
  \label{fig:fups_failure}
  \vspace{-12pt}
\end{figure*}

 FuPS \emph{without} agent IDs provably cannot recover optimal pure Nash equilibria in the non-temporal 2-player Specialisation Game (Proof~\ref{proof:fups_failure}, App.~\ref{appen:spec_game}). In practice, however, FuPS is often conditioned \emph{with} agent IDs, and MARL policies must handle complexities beyond static, two-player interactions. We therefore evaluate standard architectures in the temporal $n$-player versions of these games\footnote{Results for non-temporal (normal-form) variants are in App.~\ref{app:normal_form}.} We compare three standard architectures trained with REINFORCE~\citep{williams1992simple}: 
1) \textbf{NoPS}: independent policies ($\pi_{\theta^i}(a^i|o^i)$); 
2) \textbf{FuPS}: a single shared policy ($\pi_{\theta}(a^i|o^i)$); and 
3) \textbf{FuPS+ID}: a shared policy incorporating a one-hot agent ID ($\pi_{\theta}(a^i|o^i,id^i)$). 
All use single-layer networks, 10-step episodes and $10,000$ training steps (further details in Table~\ref{tab:hyperparameters_spec_syn_game} in App.~\ref{append:hyperparams}).

\textbf{Empirical Performance.} Table~\ref{tab:spec_sync} shows that neither NoPS nor FuPS consistently achieves the highest mean evaluation rewards. NoPS excels in the Specialisation Game but is outperformed by FuPS (optimal) and FuPS+ID in the Synchronisation Game. Furthermore, the performance gaps widen as the number of agents increases (notably at $n=8$ and $n=16$), highlighting the scalability challenges of both fully independent and fully shared policies.

\subsection{Why FuPS+ID Fails to Specialise: The Problem of Gradient Conflict}\label{subsec:grad_conflicts}

Despite being a universal approximator~\citep{hornik1989multilayer}, FuPS+ID often struggles to learn diverse policies in practice (Table~\ref{tab:spec_sync}, \citep{christianos2021scaling,JMLR:v25:23-0488}). A key reason is \emph{gradient conflict}: when a single network processes both observation $o$ and agent ID $id^i$, updates intended to specialise agent $i$ (based on $id^i$) can conflict with updates for agent $j$ (based on $id^j$), particularly if they share similar observations but require different actions. This obstructs the emergence of specialised behaviours (conflict measured via inter-agent gradient cosine similarity, App.~\ref{appen:gradient_metric_definitions}).

\textbf{Importance of Observation and ID Decoupling.} To investigate the effect of entangled observation and ID inputs, we introduce an ablation: \emph{FuPS+ID (No State)}, where the policy $\pi_\theta(a^i \mid id^i)$ conditions \emph{only} on the agent ID, ignoring observations. Surprisingly, \emph{FuPS+ID (No State) outperforms standard FuPS+ID} in the Specialisation Game for all tested $N$ (Figure~\ref{subfig:rewards}), even when $N \le 4$ (where observation spaces are small, suggesting the issue is not merely observation size). Figure~\ref{subfig:grad_conflict} reveals why: FuPS+ID (No State) shows near-zero gradient conflict (nearly orthogonal gradients), whereas standard FuPS+ID exhibits negative cosine similarities (conflicting gradients).

These results show that naively coupling observation and ID inputs in shared networks can lead to destructive interference, hindering specialisation. While discarding observations is not a general solution (most tasks require state information), this finding motivates designing architectures that can leverage both state and agent IDs, while minimising interference. Section~\ref{sec:hypermarl} introduces HyperMARL (Figure~\ref{fig:hypermarl}), which explicitly \emph{decouples observation- and agent-conditioned} gradients through agent-conditioned hypernetworks, leading to improved performance over FuPS variants and reduced gradient conflict compared to standard FuPS+ID (Figure~\ref{fig:fups_failure}). 

\section{HyperMARL}\label{sec:hypermarl}

We introduce \textit{HyperMARL}, an approach that uses agent-conditioned hypernetworks to learn diverse or homogeneous policies \emph{end-to-end}, without modifying the standard RL objective or requiring manual preset diversity levels. By operating within a fully shared paradigm, HyperMARL leverages shared gradient information while enabling specialisation through the decoupling of observation- and agent-conditioned gradients. We present the pseudocode in Sec.~\ref{subsec:pseudocode}, with additional scaling (\ref{sec:scaling}) and runtime (\ref{sec:runtime}) details.

\subsection{Hypernetworks for MARL}

As illustrated in Figure \ref{fig:hypermarl}, for any agent $i$ with context $e^i$ (i.e., either a one-hot encoded ID or a learned embedding), the hypernetworks generate the agent-specific parameters:
\begin{equation}
\theta^i = h_{\psi}^{\pi}(e^i), \quad \phi^i = h_{\varphi}^{V} (e^i),
\end{equation}
where $h_{\psi}^{\pi}$ and $h_{\varphi}^{V}$ are the hypernetworks for the policy and critic, respectively. The parameters $\theta^i$ and $\phi^i$ define each agent's policy $\pi_{\theta^i}$ and critic $V_{\phi^i}$, dynamically enabling either specialised or homogeneous behaviours as required by the task.

\textbf{Linear Hypernetworks}\label{subsec:linear_hnets} Given a one-hot agent ID, $\mathds{1}^i \in \mathbb{R}^{1 \times n}$, a linear hypernetwork $h_{\psi}^{\pi}$ generates agent-specific parameters $\theta^i$ as follows\footnote{For conciseness we only show the policy parameters in this section.}:
\begin{equation}
    \theta^i = h_{\psi}^{\pi}(\mathds{1}^i) = \mathds{1}^i \cdot W + b
\end{equation}
where $W \in \mathbb{R}^{n \times m}$ is the weight matrix (with $m$ the per-agent parameter dimensionality and $n$ is the number of agents), and $b \in \mathbb{R}^{1 \times m}$ is the bias vector. Since $\mathds{1}^i$ is one-hot encoded, each $\theta^i$ corresponds to a specific row of $W$ plus the shared bias $b$. If there is no shared bias term, this effectively replicates training of separate policies for each task (in our case, for each agent) \citep{beck2023hypernetworks}, since there are no shared parameters and gradient updates are independent.

 \textbf{MLP Hypernetworks for Expressiveness} \label{subsec:mlp_hnets} To enhance expressiveness, MLP Hypernetworks incorporate hidden layers and non-linearities:

\begin{equation}
    \theta^i = h_{\psi}^{\pi}(e^i) = f_{\psi_1}^{\pi}\bigl(g_{\psi_2}^{\pi}(e^i)\bigr)
\end{equation}
where $g_{\psi_2}^{\pi}$ is an MLP processing the agent context $e^i$, and $f_{\psi_1}^{\pi}$ is a final linear output layer.

Unlike linear hypernetworks with one-hot agent IDs, MLP hypernetworks do not guarantee distinct weights for each agent. Additionally, they increase the total number of trainable parameters, necessitating a careful balance between expressiveness and computational overhead.

\subsection{Agent Embeddings and Initialisation}
\label{subsec:agent_context_and_init}

The agent embedding $e^i$ is a one-hot encoded ID for Linear Hypernetworks. For MLP Hypernetworks, we use learned agent embeddings, orthogonally initialised and optimised end-to-end with the hypernetwork. HyperMARL's hypernetworks are themselves initialised such that the generated agent-specific parameters ($\theta^i, \phi^i$) initially match the distribution of standard direct initialisation schemes (e.g., orthogonal for PPO, preserving fan in/out), promoting stable learning.

\subsection{Gradient Decoupling in HyperMARL}\label{subsec:decoupling_grads}

A core difficulty in FuPS is cross-agent gradient interference~\citep{christianos2021scaling,JMLR:v25:23-0488}. HyperMARL mitigates this
by generating each agent’s parameters through a shared hypernetwork,
thereby \emph{decoupling} \emph{agent-conditioned} and \emph{observation‑conditioned} components of the gradient.

\textbf{Hypernetwork gradients.} Consider a fully cooperative MARL setting with a centralised critic, we can formulate the policy gradient for agent $i$ as follows~\citep{albrecht2024multi,kuba2021settling}: 
$$
\nabla_{\theta^i} J(\theta^i) \;=\;
\mathbb{E}_{\mathbf{h_t},\mathbf{a_t} \sim \boldsymbol{\pi}}
\Bigl[
  A(\mathbf{h_t}, \mathbf{a_t})
  \,\nabla_{\theta^i} \log \pi_{\theta^i}(a_t^i \mid h_t^i)
\Bigr],
$$
where $\mathbf{h_t}$ and $\mathbf{a_t}$ are the joint histories and joint actions for all agents, $\theta^i$ denotes the parameters of agent $i$, and $A(\mathbf{h_t}, \mathbf{a_t}) = Q(\mathbf{h_t}, \mathbf{a_t}) - V(\mathbf{h_t})$ is the advantage function.

\textbf{Decoupling.} In HyperMARL each agent’s policy weights are produced by the hypernetwork $h_{\psi}^{\pi}$: $\theta^{i}=h_{\psi}^{\pi}(e^{i})$, so we optimise a \emph{single}
parameter vector~$\psi$. Applying the chain rule and re-ordering the expectations: 
\begin{equation}
\label{eq:hmarl_split}
\nabla_{\!\psi}J(\psi)
=
\sum_{i=1}^{I}
\underbrace{\nabla_{\!\psi}h_{\psi}^{\pi}(e^{i})}_{\mathbf J_i\;\text{(agent-conditioned)}}\;
\underbrace{
  \mathbb{E}_{\mathbf h_t,\mathbf a_t\sim\boldsymbol\pi}
  \bigl[
     A(\mathbf h_t,\mathbf a_t)\,
     \nabla_{\theta^{i}}\log\pi_{\theta^{i}}(a_t^{i}\mid h_t^{i})
  \bigr]}_{Z_i\;\text{(observation-conditioned)}} .
\end{equation}

\begin{itemize}[leftmargin=*]
   \item \textbf{Agent-conditioned factor $\mathbf J_i$.}  
      This Jacobian depends only on the fixed embedding $e^i$ and the
      hypernetwork weights $\psi$, therefore, it is \emph{deterministic} with respect to mini-batch samples (as $e^i$ and $\psi$ are fixed per gradient step), separating agent identity from trajectory noise.
      \item  \textbf{Observation-conditioned factor $Z_i$.} The expectation averages trajectory noise \emph{per agent~$i$} for its policy component $\pi_{\theta^i}$, prior to transformation by $\mathbf{J}_i$ and aggregation. 
      
\end{itemize}

The crucial structural decoupling in Equation~\eqref{eq:hmarl_split} ensures HyperMARL first averages noise per agent (via factor $Z_i$) before applying the deterministic agent-conditioned transformation $\mathbf{J}_i$. This mitigates gradient interference common in FuPS+ID ~\citep{christianos2021scaling,JMLR:v25:23-0488}, where observation noise and agent identity become entangled (see Equation~\eqref{eq:variance_split} in App.~\ref{app:gradient_details}). This is the MARL analogue of the task/state decomposition studied by \citep[Eq.~18]{sarafian2021recomposing} in Meta-RL. Section \ref{subsubsec:grad_var} empirically verifies the predicted variance drop, and ablations in Section \ref{sec:ablations} demonstrate that disabling decoupling degrades performance, underscoring its critical role.

\section{Experiments}\label{sec:experiments}

Our empirical evaluation of HyperMARL assesses whether agent-conditioned hypernetworks can enable adaptive (specialised or homogeneous) policies without altered RL objectives, preset diversity levels, or sequential updates. We structure our experiments to directly answer two key research questions: \textbf{Q1: Specialised Policy Learning:} Can \textit{HyperMARL} effectively learn \textit{specialised policies} via a shared hypernetwork? \textbf{Q2: Effectiveness in Homogeneous Tasks}: Is \textit{HyperMARL} competitive in environments that necessitate homogeneous behaviours?

To address these questions, HyperMARL is evaluated against six representative modern baselines across a carefully selected suite of MARL benchmarks (22 scenarios, up to 30 agents). All experiments use at least 5 seeds (details in App.~\ref{subsec:train_and_eval}).

\subsection{Experimental Setup}

\begin{wraptable}{r}{0.45\columnwidth} %
    \centering
    \caption{MARL environments for evaluating \textit{HyperMARL}.}
    \label{tab:env_summary}
    \vspace{-0.5em} %
    \resizebox{\linewidth}{!}{%
        \begin{tabular}{lccc}
        \toprule
        \textbf{Env.} & \textbf{Agents} & \textbf{Action} & \textbf{Behaviour} \\ %
        \midrule
        Dispersion & 4 & Discrete & Hetero. \\ %
        Navigation & 2,4,8 & Continuous & Homo., Hetero., Mixed \\ %
        MAMuJoCo & 2--17 & Continuous & Hetero. \\
        SMAX & 2--20 & 
        Discrete & Homo. \\
        BPS & 15--30 & Discrete & Hetero. \\
        \bottomrule
        \end{tabular}%
    }
    \vspace{-1em} %
\end{wraptable}

\textbf{Environments.} HyperMARL is evaluated across 22 scenarios from five diverse MARL environments (Dispersion~\citep{bettini2022vmas}, Navigation~\citep{bettini2022vmas}, MAMuJoCo~\citep{peng2021facmac}, SMAX~\citep{rutherford2024jaxmarl}, BPS~\citep{christianos2021scaling}) (Table~\ref{tab:env_summary}). These were specifically chosen to rigorously test performance across varying complexities, agent counts (2 to 30), and distinct behaviours (heterogeneous, homogeneous, or mixed). Full details in Appendix~\ref{subsec:environments}.

\textbf{Baselines.} We evaluate HyperMARL against modern parameter sharing (PS) and diversity-promoting baselines. Core PS comparisons use \textbf{FuPS+ID} and \textbf{NoPS}. For specialisation, we include \emph{privileged} baselines: \textbf{DiCo}~\citep{bettini2024controlling} (shared and non-shared weights, preset diversity levels); \textbf{HAPPO}~\citep{JMLR:v25:23-0488} (shared critic, sequential actor updates); \textbf{Kaleidoscope}~\citep{li2024kaleidoscope} (learnable masks, critic ensembles, diversity loss); and \textbf{SePS} (pre-training phase, agent clustering). We use IPPO/MAPPO~\citep{de2020independent,yu2022surprising} as the underlying algorithm for all methods except Kaleidoscope and SePS (see App.~\ref{app:kaleidoscope_offpolicy} and App.~\ref{app:seps} for these results).

Adhering to best evaluation practices~\citep{patterson2024empirical}, we ~\textit{use original codebases/hyperparameters and environments for which baselines were tuned}. HyperMARL uses identical observations and generates architectures of equivalent capacity to baselines. Training and evaluation (App.~\ref{subsec:train_and_eval}) and hyperparameters (App.~\ref{append:hyperparams}) follow each baseline's original setup. We detail our baseline and environment selection criteria in Table~\ref{tab:baseline_selection}, with architecture details in App.~\ref{app:arch_details}.

\begin{figure*}[tb]
  \centering
  \includegraphics[width=0.6\linewidth]{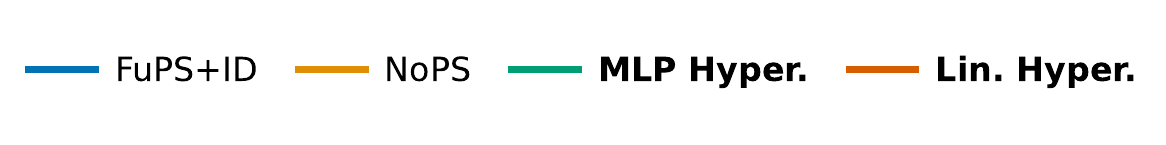}
  
  \vspace{1ex}
  
  \begin{subfigure}[b]{0.23\textwidth}
    \centering
    \includegraphics[width=\linewidth]{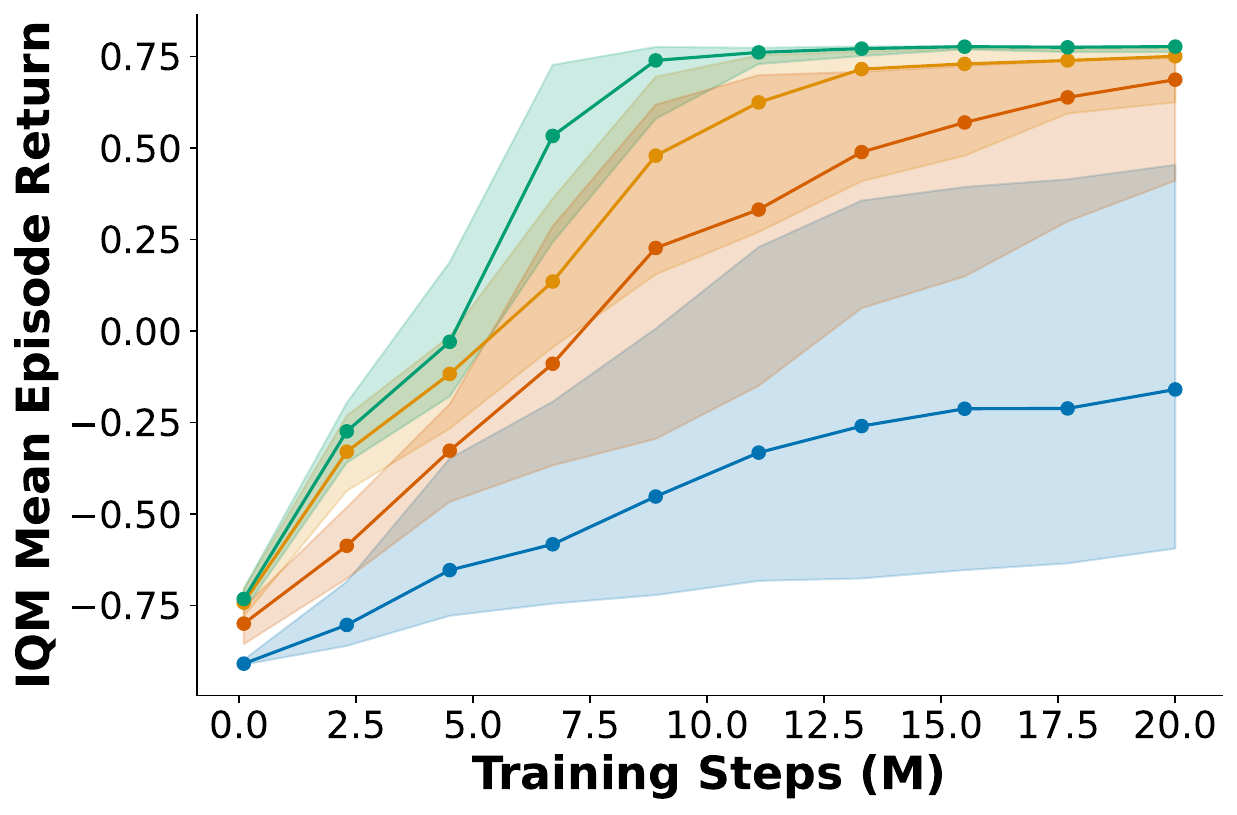}
    \caption{IPPO}
    \label{fig:ippo-sample-efficiency}
  \end{subfigure}\hfill
  \begin{subfigure}[b]{0.23\textwidth}
    \centering
    \includegraphics[width=\linewidth]{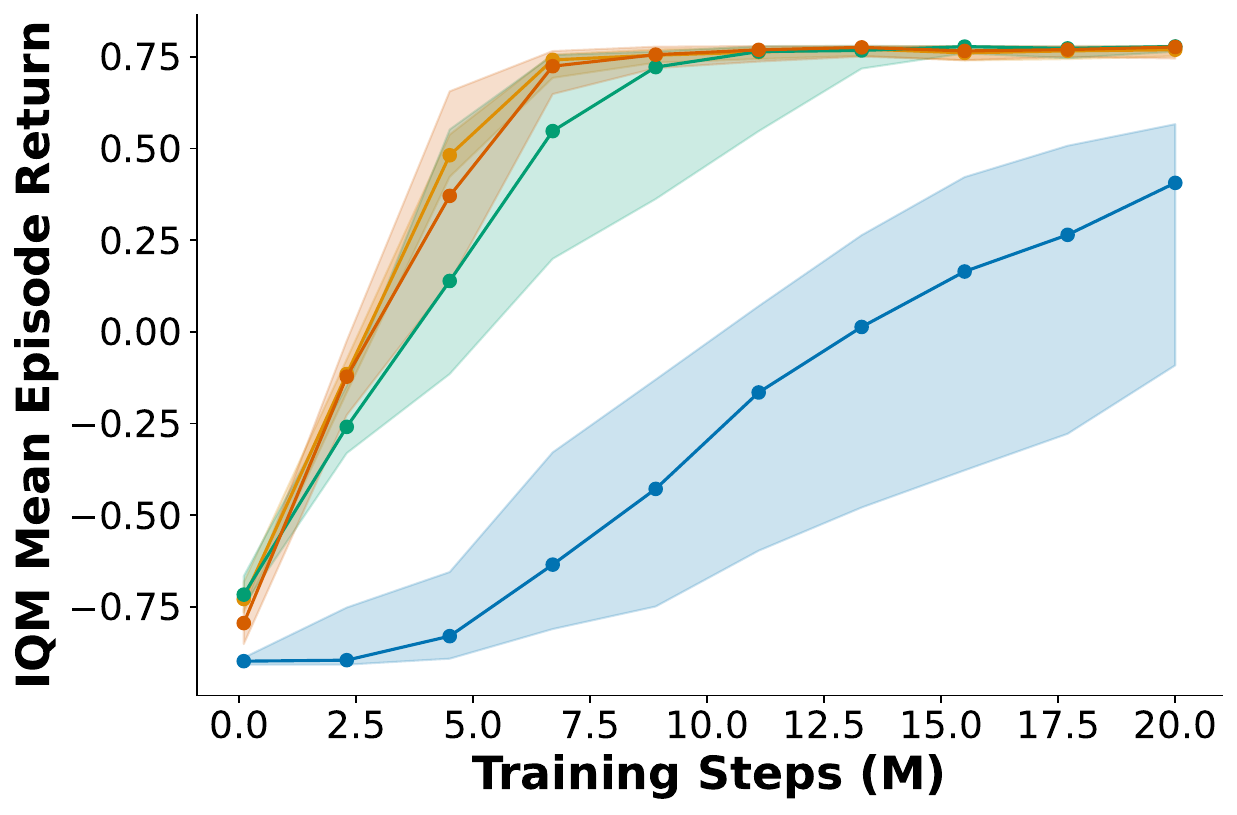}
    \caption{MAPPO}
    \label{fig:mappo-sample-efficiency}
  \end{subfigure}\hfill
  \begin{subfigure}[b]{0.23\textwidth}
    \centering
    \includegraphics[width=\linewidth]{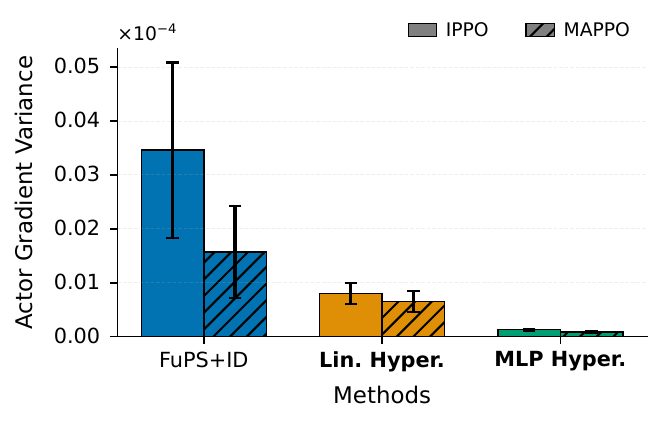}
    \caption{Actor Gradient Var.}
    \label{fig:actor_grad_var}
  \end{subfigure}\hfill
  \begin{subfigure}[b]{0.23\textwidth}
    \centering
    \includegraphics[width=\linewidth]{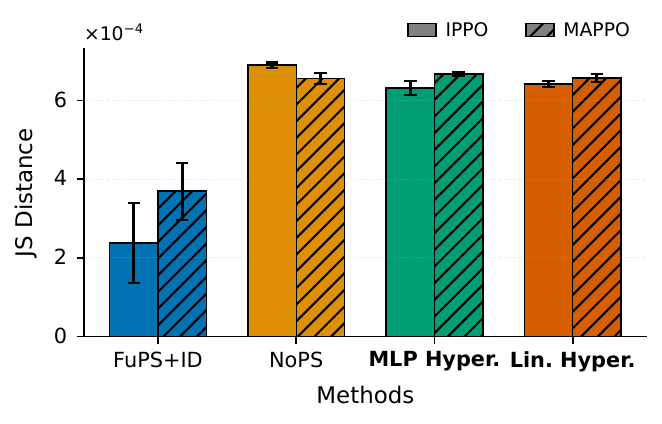}
    \caption{\small Policy Diversity}
    \label{fig:snd_plot}
  \end{subfigure}
  \caption{\textit{Performance and gradient analysis.}  
    \textbf{(a,b)} IPPO and MAPPO on Dispersion (20M timesteps) - IQM of Mean Episode Return with 95\% bootstrap CIs: Hypernetworks match NoPS performance while FuPS struggle with specialisation. Interval estimates in App.~\ref{append:dispersion_dynamics}.   
    \textbf{(c)} Actor gradient variance: Hypernetworks achieve lower gradient variance than FuPS+ID.  
    \textbf{(d)} Policy diversity (SND with Jensen–Shannon distance): Hypernetworks achieve NoPS-level diversity while sharing parameters.
  }
  \label{fig:performance-comparison}
  \vspace{-12pt}
\end{figure*}

\textbf{Measuring Policy Diversity.}\label{subsubsec:policy_diversity}  
To measure the diversity of the policies we System Neural Diversity (SND)~\citep{bettini2023system} (Equation \ref{eq:snd}) with Jensen-Shannon distance (details in App.~\ref{app:measure_policy_diversity}).

\subsection{Q1: Specialised Policy Learning}\label{subsec:specialised_policy_learning}

\definecolor{FuPS}{HTML}{0173b2}
\definecolor{NoPS}{HTML}{de8f05}
\definecolor{LinearHnet}{HTML}{d55e00} 
\definecolor{MLPHnet}{HTML}{029e73}

\textbf{Learning Diverse Behaviour (Dispersion)} \label{results:dispersion} Figures~\ref{fig:ippo-sample-efficiency} and \ref{fig:mappo-sample-efficiency} show that FuPS variants (IPPO-FuPS, MAPPO-FuPS -- (\textcolor{FuPS}{$\bullet$})) can struggle to learn the diverse policies required by Dispersion (even when running for longer - Fig.~\ref{fig:run_longer}), while their NoPS counterparts (IPPO-NoPS, MAPPO-NoPS--(\textcolor{NoPS}{$\bullet$})) converge to the optimal policy, corroborating standard FuPS limitations to learn diverse behaviour. In contrast, HyperMARL (both linear and MLP variants) (\textcolor{LinearHnet}{$\bullet$}, \textcolor{MLPHnet}{$\bullet$}) match NoPS performance, suggesting that a shared hypernetwork can effectively enable agent specialisation. \emph{SND policy diversity measurements}  (Fig.~\ref{fig:snd_plot}) confirm FuPS variants achieve lower behavioural diversity than NoPS, while HyperMARL notably matches NoPS-level diversity.

\begin{wrapfigure}[18]{r}{0.35\textwidth}
    \vspace{-8pt}
    \centering
    \includegraphics[width=0.95\linewidth]{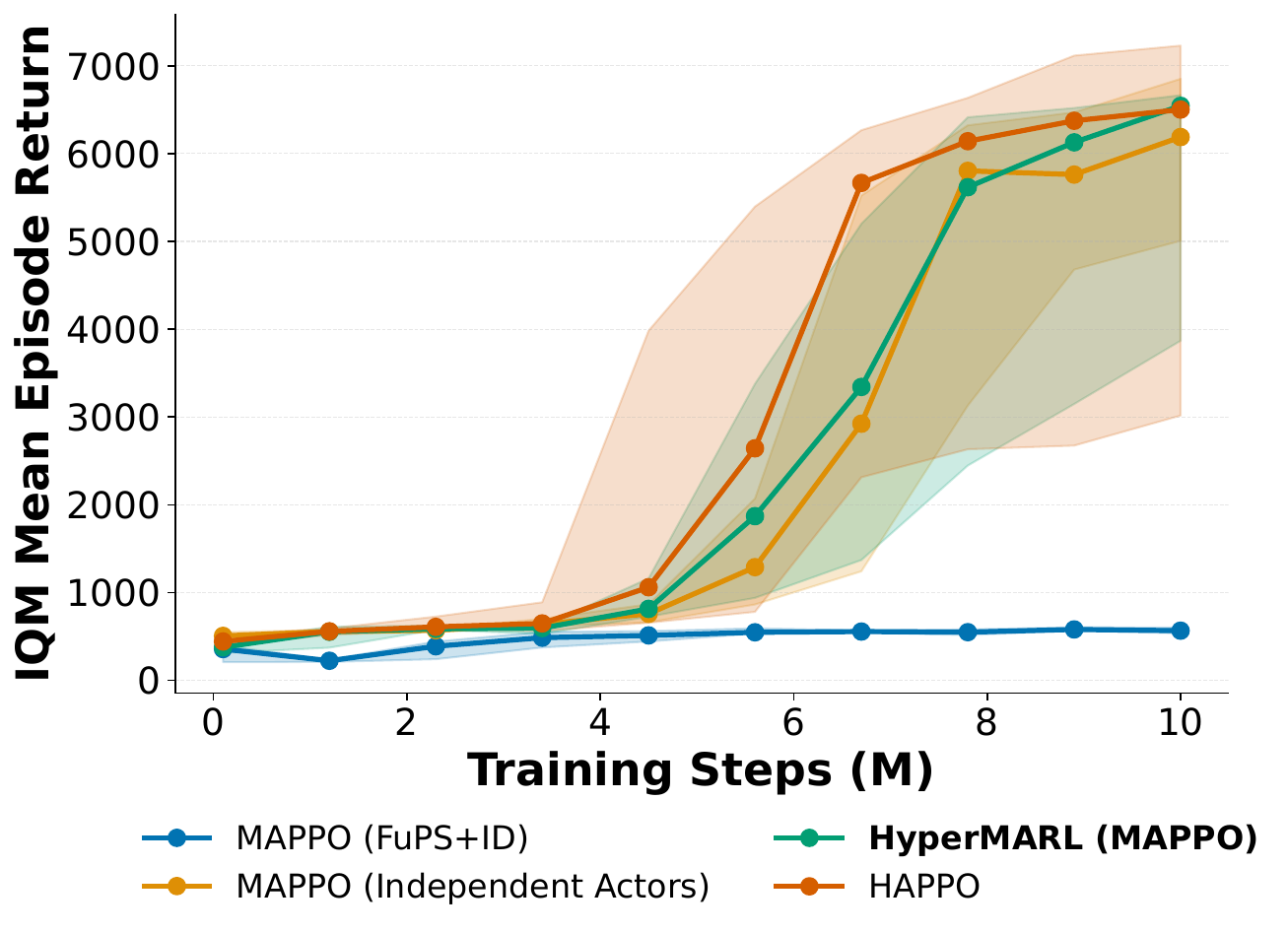}
    \vspace{-4pt}
    \caption{\small \textit{17-agent Humanoid learning dynamics (IQM, 95\% CI).} HyperMARL, utilising a shared actor architecture, outperforms MAPPO-FuPS (non-overlapping CIs) and matches the performance of methods employing non-shared or sequential actors. This challenging environment is recognised for its high variance in outcomes across different methods~\citep{JMLR:v25:23-0488}.}
    \label{fig:humanoid-result}
    \vspace{-8pt}
\end{wrapfigure}

\textbf{Gradient Variance.}\label{subsubsec:grad_var} HyperMARL (IPPO and MAPPO variants) also exhibits lower mean policy gradient variance than FuPS+ID across actor parameters (Fig.~\ref{fig:actor_grad_var}). This aligns with their ability to learn diverse behaviours and supports the hypothesis that its gradient decoupling mechanism (Sec.~\ref{subsec:decoupling_grads}) enhances training stability.

\textbf{Diversity at Complexity and Scale (MAMuJoCo).}\label{sec:mamujoco_results} In the challenging MAMuJoCo heterogeneous control tasks (Table~\ref{tab:mamujoco_results}), HyperMARL (MLP variant) is broadly competitive. Notably, unlike HAPPO and MAPPO (independent actors), HyperMARL uses a shared actor and parallel updates, and yet manages strong performance, even in the  17-agent Humanoid-v2 notoriously difficult heterogeneous task\citep{JMLR:v25:23-0488} (Fig.~\ref{fig:humanoid-result}), matching methods that employ independent actors and sequential updates.

\begin{table*}[tb] %
\small
\centering
\caption{\textit{Mean episode return in MAMuJoCo for MAPPO variants(IQM, 95\% CI).}
    HyperMARL achieves the highest IQM in 3/4 scenarios (bold), and is the only method with shared actors to demonstrate stable learning in the notoriously difficult 17-agent Humanoid environment (see Figure~\ref{fig:humanoid-result} for learning dynamics). * indicates CI overlap with the top score.
}
\label{tab:mamujoco_results}
\resizebox{0.90\textwidth}{!}{%
\begin{tabular}{lcccc}
\toprule
\textbf{Scenario} & \textbf{HAPPO} & \textbf{FuPS+ID} & \textbf{Ind. Actors} & \textbf{HyperMARL (Ours)} \\
\midrule
Humanoid-v2 17x1    & 6501.15\textsuperscript{*} (3015.88, 7229.79) & 566.12 (536.36, 603.01) & 6188.46\textsuperscript{*} (5006.13, 6851.74) & \textbf{6544.10} (3868.00, 6664.89) \\
Walker2d-v2 2x3     & 4748.06\textsuperscript{*} (4366.94, 6230.81) & 4574.39\textsuperscript{*} (4254.21, 5068.32) & 4747.05\textsuperscript{*} (3974.76, 6249.58) & \textbf{5064.86} (4635.10, 5423.42) \\
HalfCheetah-v2 2x3 & 6752.40\textsuperscript{*} (6130.42, 7172.98) & 6771.21\textsuperscript{*} (6424.94, 7228.65) & 6650.31\textsuperscript{*} (5714.68, 7229.61) & \textbf{7063.72} (6696.30, 7325.36) \\
Ant-v2 4x2         & 6031.92\textsuperscript{*} (5924.32, 6149.22) & \textbf{6148.58} (5988.63, 6223.88) & 6046.23\textsuperscript{*} (5924.62, 6216.57) & 5940.16\textsuperscript{*} (5485.77, 6280.59) \\
\bottomrule
\end{tabular}%
}
\end{table*}

\textbf{Adaptability (Navigation).}\label{results:navigation} Navigation tasks~\citep{bettini2022vmas} assess adaptability to homogeneous, heterogeneous, and \textit{mixed} goals (some agents have the same goals, others different). We compare HyperMARL with baselines including DiCo~\cite{bettini2024controlling}. While using DiCo's optimal preset diversity for n=2 agents, we note that identifying appropriate diversity levels for DiCo with larger teams ($n>2$) via hyperparameter sweeps proved challenging (see Tables~\ref{tab:dico-snd-des} and~\ref{tab:param_sweeps_navigation}).

\begin{figure}[tb]
  \centering

  \begin{subfigure}[b]{0.48\columnwidth}
    \centering
    \includegraphics[width=\linewidth]{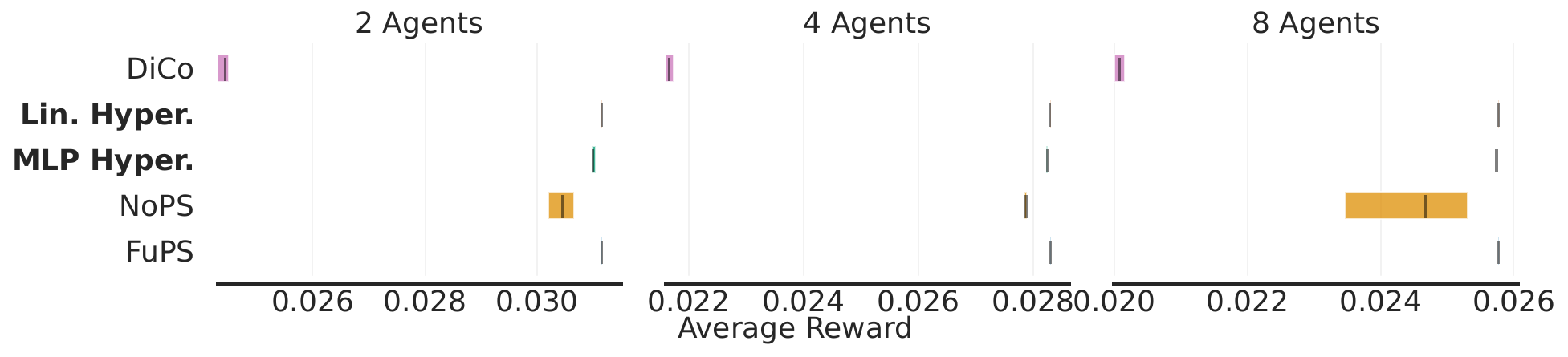}
    \caption{Shared Goals}
    \label{fig:dico_same_goals}
  \end{subfigure}\hfill
  \begin{subfigure}[b]{0.48\columnwidth}
    \centering
    \includegraphics[width=\linewidth]{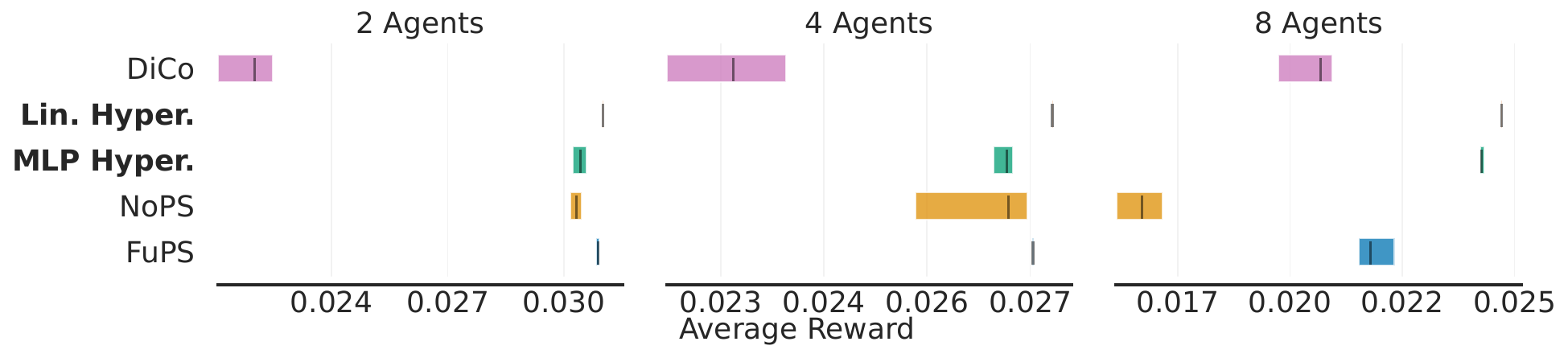}
    \caption{Unique Goals} 
    \label{fig:dico_diff_goals}
  \end{subfigure}

  \vspace{0.5em} %

  \begin{subfigure}[b]{0.35\columnwidth} %
    \centering
    \includegraphics[
      width=\linewidth,
      keepaspectratio %
    ]{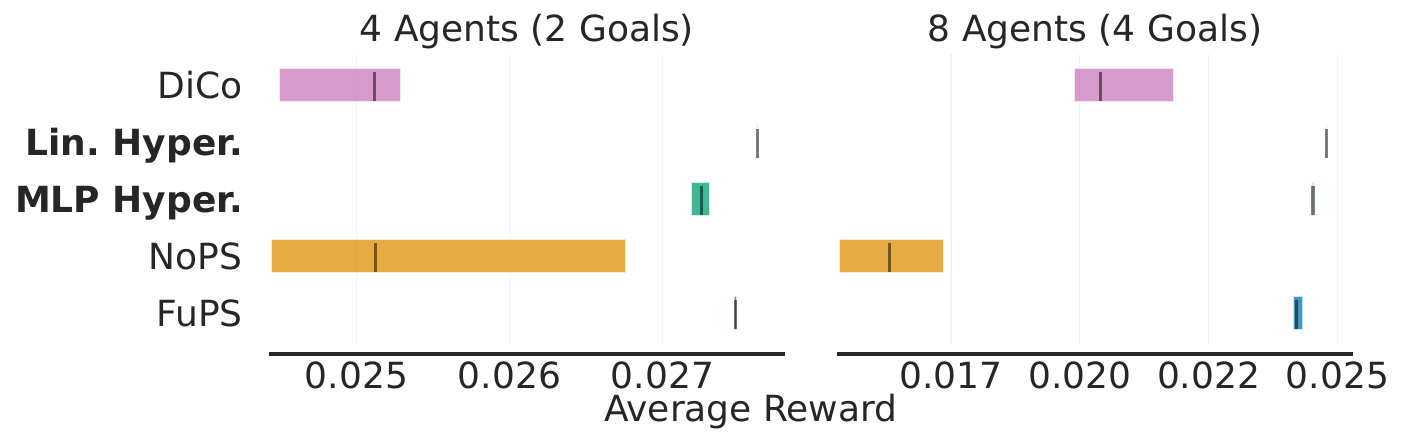}
    \caption{Mixed Goals (Half Agents Share Goals)} 
    \label{fig:dico_mix_goals}
  \end{subfigure}
  \caption{\small \textit{Average Reward (IQM, 95\% CI) in Navigation for IPPO Variants.} HyperMARL adapts robustly across goal configurations—(a) shared, (b) unique, and (c) mixed. Both linear and MLP versions consistently match or outperform IPPO baselines and DiCo, with the margin widening as the number of agents grows. Sample-efficiency curves appear in App.~\ref{app:dico_results}.}\label{fig:dico_results}
  \vspace{-10pt} %
\end{figure}

Across all tested goal configurations (shared, unique, and mixed), HyperMARL consistently achieves strong performance (Figure~\ref{fig:dico_results}). It generally matches or outperforms NoPS and FuPS, and outperforming DiCo. Interestingly, unlike in sparse-reward tasks like Dispersion, FuPS remains competitive with NoPS and HyperMARL in Navigation scenarios requiring diverse behaviours for smaller teams ($n \in \{2, 4\}$), likely due to Navigation's dense rewards. However, HyperMARL distinguishes itself as the strongest method for n=8 agents, highlighting its effectiveness in handling more complex coordination challenges.

\begin{figure*}[tb]
    \centering
    \includegraphics[width=0.8\linewidth]{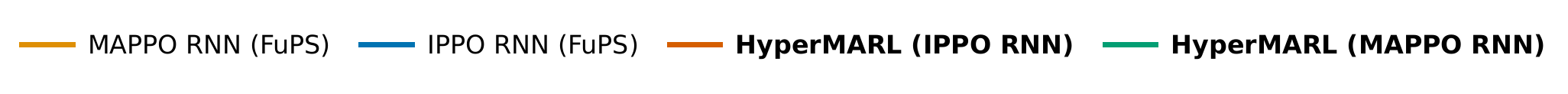}
    \begin{subfigure}[b]{0.24\textwidth}
        \centering
        \includegraphics[width=\textwidth]{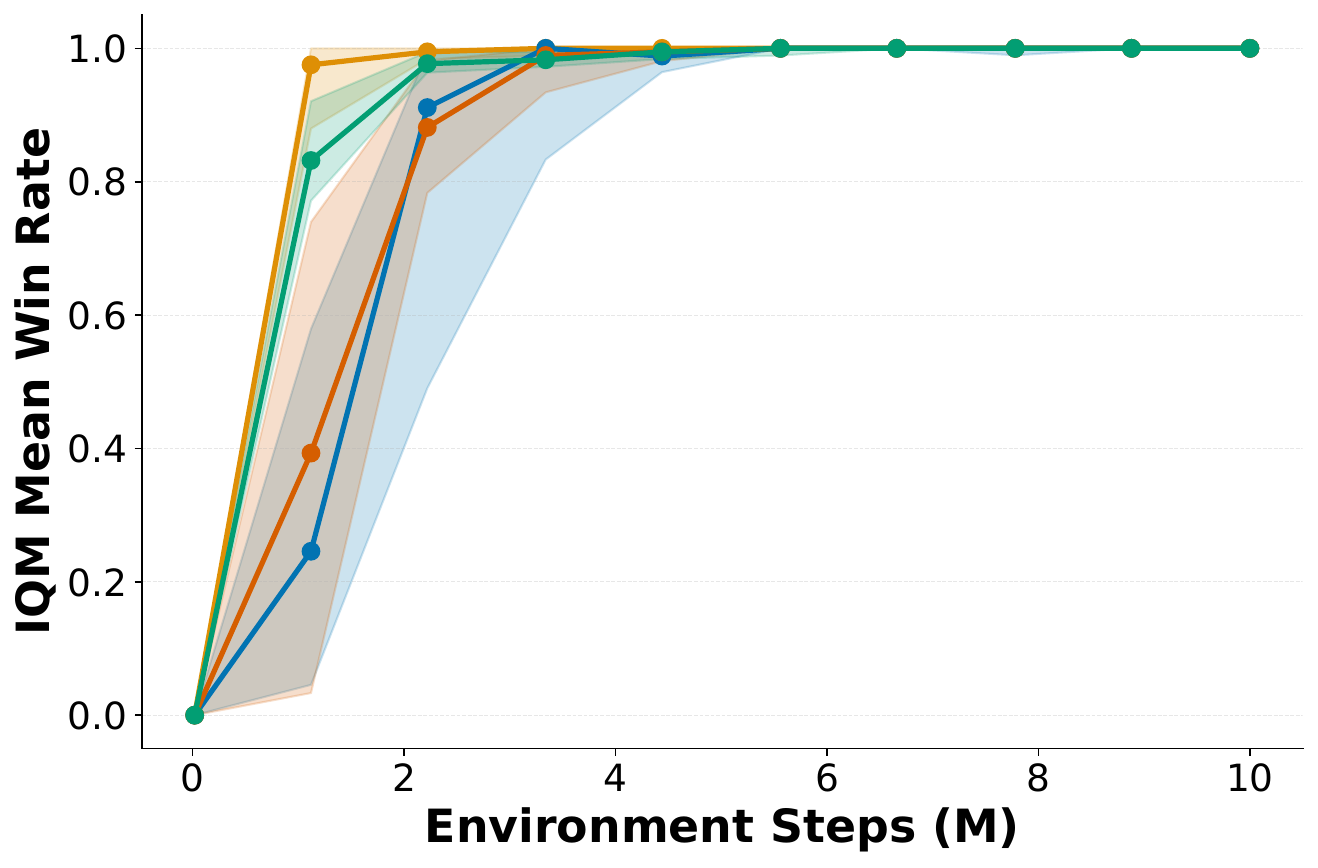}
        \caption{2s3z}
        \label{fig:smax_2s3z}
    \end{subfigure}
    \hfill
    \begin{subfigure}[b]{0.24\textwidth}
        \centering
        \includegraphics[width=\textwidth]{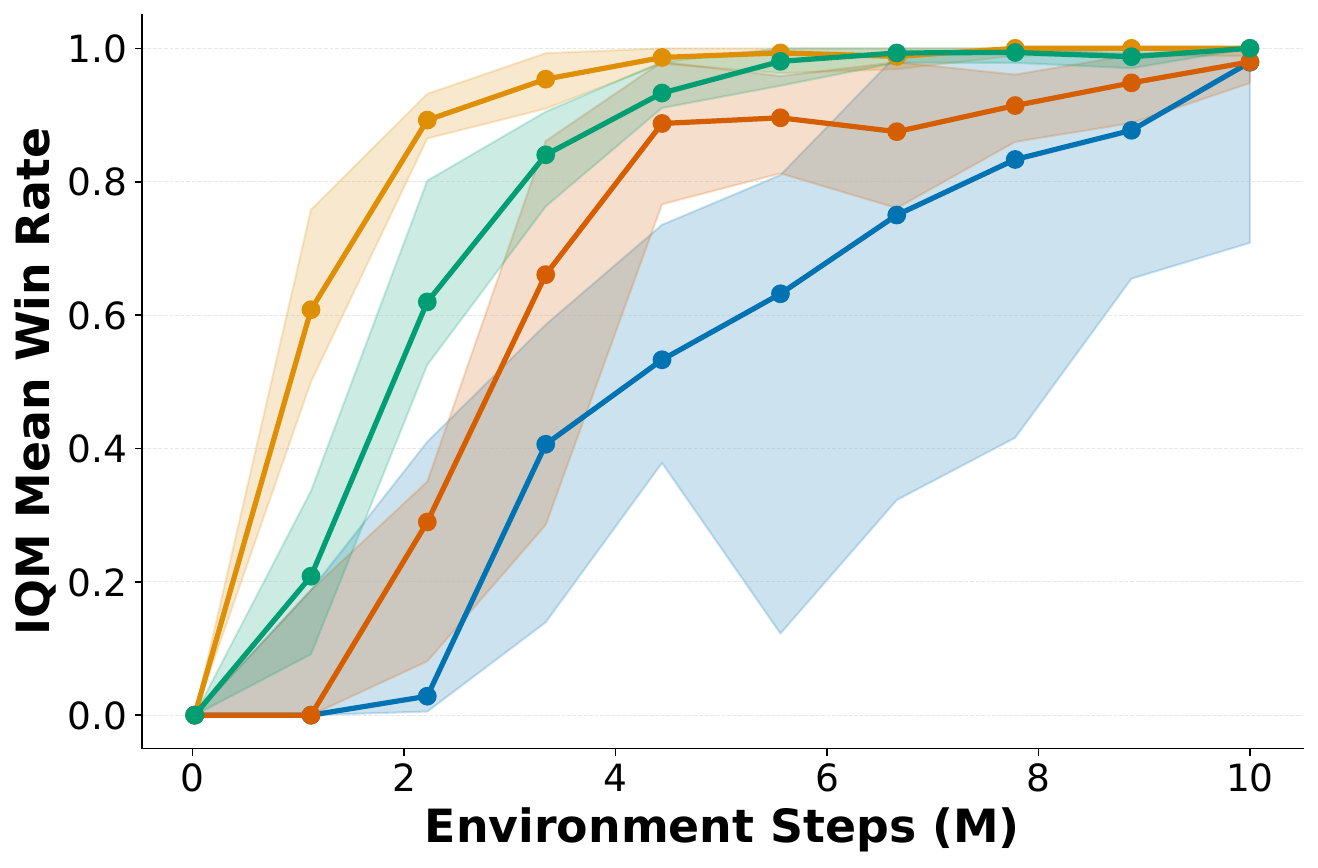}
        \caption{3s5z}
        \label{fig:smax_3s5z}
    \end{subfigure}
    \hfill
    \begin{subfigure}[b]{0.24\textwidth}
        \centering
        \includegraphics[width=\textwidth]{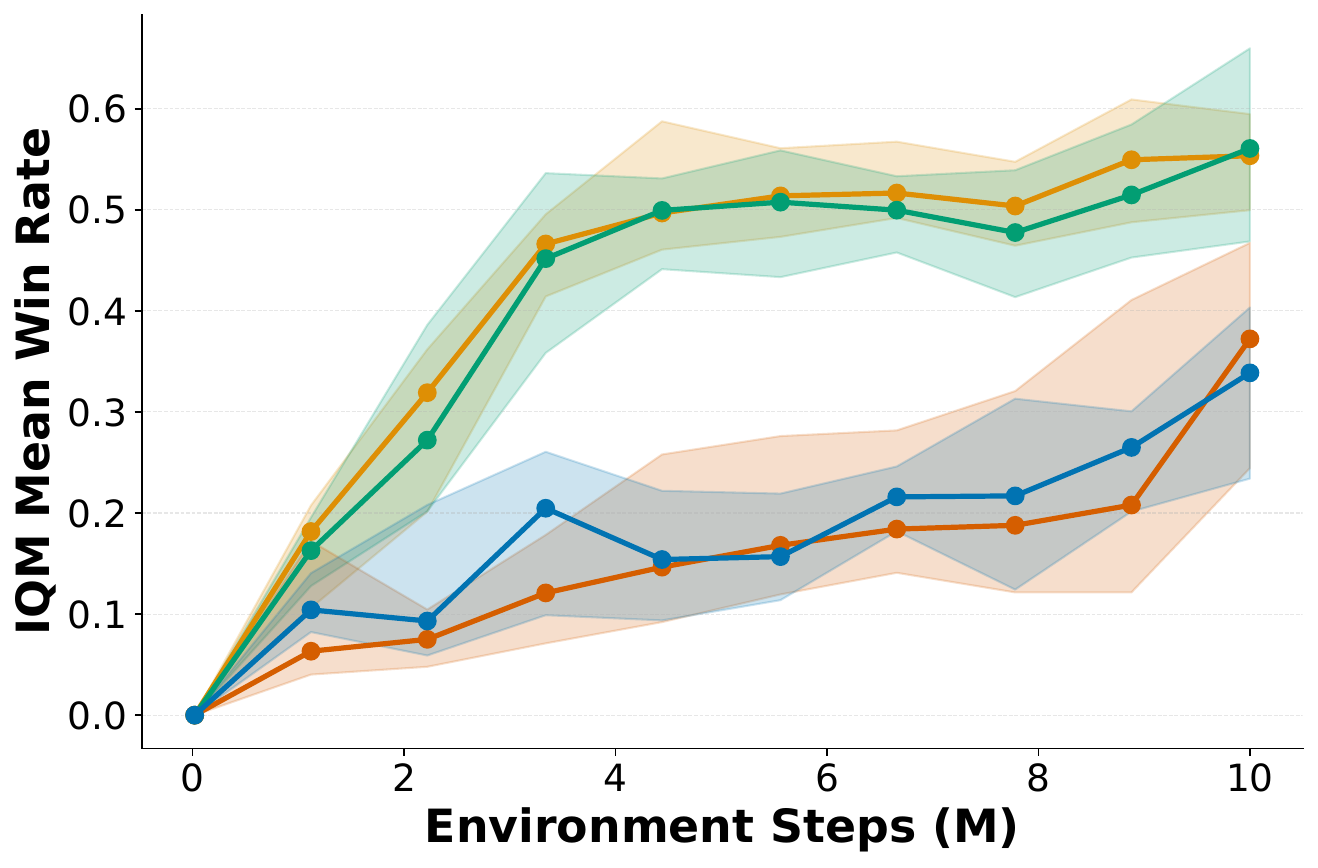}
        \caption{SMACv2 10 Units}
        \label{fig:smax_10_units}
    \end{subfigure}
    \hfill 
    \begin{subfigure}[b]{0.24\textwidth}
        \centering
        \includegraphics[width=\textwidth]{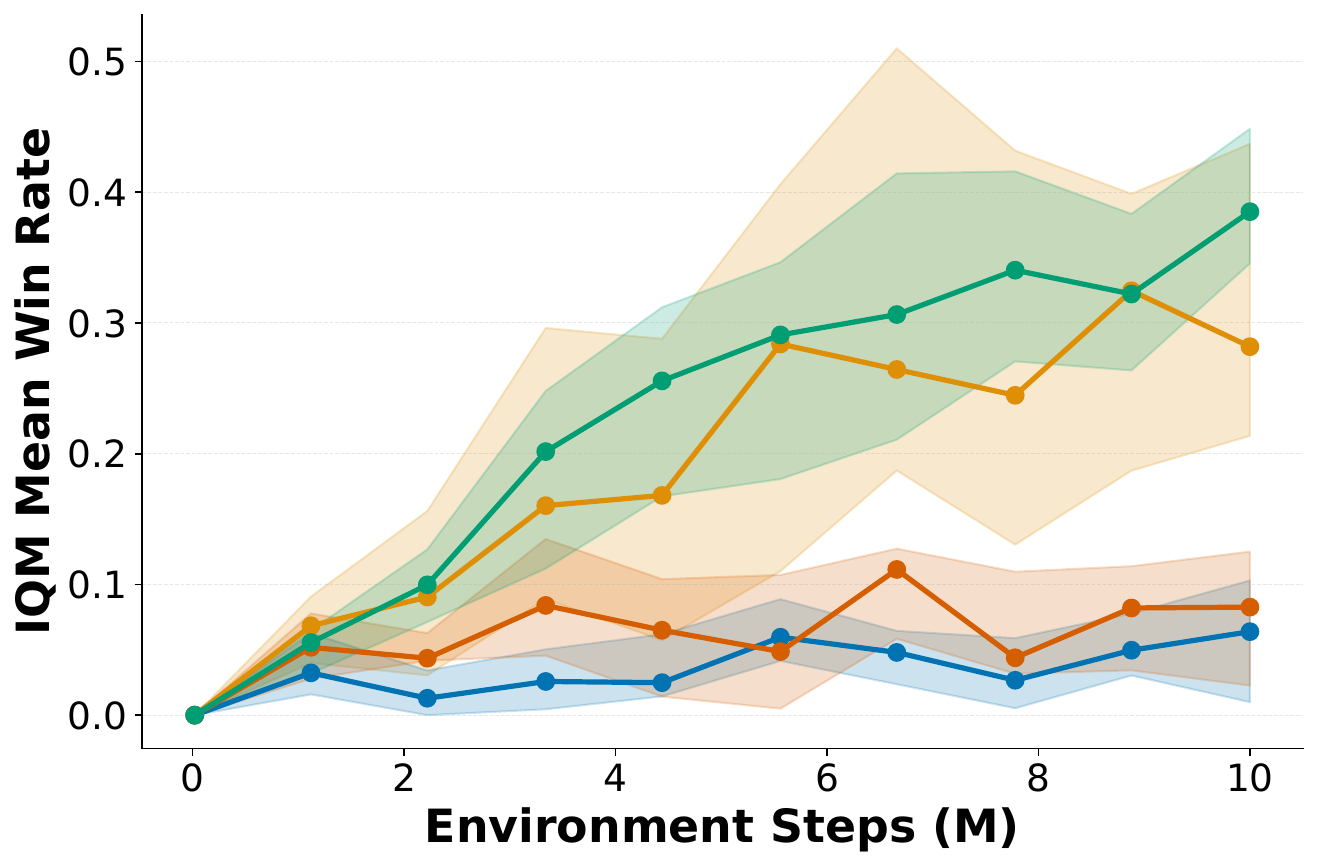}
        \caption{SMACv2 20 Units}
        \label{fig:smax_20_units}
    \end{subfigure}
    \caption{\textit{IQM and 95\% CI of mean win rate in SMAX.} Performance of FuPS Recurrent IPPO and MAPPO and HyperMARL (MLP) on SMAX. HyperMARL performs comparably to FuPS baselines across all environments, demonstrating its effectiveness in tasks requiring homogeneous behaviours and using recurrency. Interval estimates in Fig.~\ref{fig:smax_interval_estimates} in App.~\ref{app:interval_est_smax}.}
    \label{fig:performance_smax}
\end{figure*}

\subsection{Q2: Effectiveness in Homogeneous Tasks}\label{results:smax}

\textbf{SMAX.} Finally, we evaluate HyperMARL (MLP) on SMAX, where recurrent FuPS is the established baseline~\citep{rutherford2024jaxmarl,yu2022surprising,fu2022revisiting}. Figure \ref{fig:performance_smax} shows while some FuPS variants might exhibit marginally faster initial convergence on simpler maps, HyperMARL achieves comparable final performance on all maps, using the same GRU backbone for partial observability. These results highlight two points: (i) HyperMARL is fully compatible with recurrent architectures essential under partial observability, and (ii) it has no intrinsic bias toward specialisation and can converge to homogeneous behaviour when it is optimal (also shown with strong same-goal Navigation performance (Fig.~\ref{fig:dico_same_goals})), even with large observation spaces and many agents.

\textbf{Summary.} Our empirical results confirm HyperMARL effectively addresses both research questions. For \textbf{Q1  (Specialisation)}, across Dispersion, MAMuJoCo, and Navigation, HyperMARL learned specialised policies, matched NoPS-level diversity and performance where FuPS+ID struggled, and scaled to complex, high-agent-count heterogeneous tasks. For \textbf{Q2 (Homogeneity)}, HyperMARL demonstrated competitive performance against strong FuPS baselines in SMAX and shared-goal Navigation, confirming its versatility.

\section{Ablations and Embedding Analysis}

\subsection{Ablations: Gradient Decoupling and Initialisation Scaling}\label{sec:ablations}
We ablate two critical components of HyperMARL: \emph{gradient decoupling} (Sec.~\ref{subsec:decoupling_grads}) and initialisation scaling (Sec.~\ref{subsec:agent_context_and_init}). In \textit{HyperMARL w/o GD}, the hypernetwork is conditioned on $[o_t, e^i_t]$, coupling observation and agent-ID gradients. In \textit{HyperMARL w/o RF}, we remove the reset fan-in/out scaling that aligns the scale of generated parameters $(\theta^i,\phi^i)$ with standard initialisers.

\textbf{Gradient decoupling is essential; initialisation scaling grows with complexity.} Figure~\ref{fig:ablations} shows that removing GD consistently degrades performance across both \textit{Humanoid-v2} (17 agents) and \textit{Dispersion}, showing that GD is an essential component of HyperMARL. Removing RF reveals a task-dependent effect: it is critical on \textit{Humanoid-v2}, consistent with hypernetwork initialisation results~\citep{Chang2020Principled}, but has a minor impact on \textit{Dispersion}, indicating that principled initialisation becomes more vital with increased complexity. We provide additional ablations in App.~\ref{app:additional_ablations}.

\begin{figure}[tb]
  \centering
  \begin{minipage}{\textwidth}
    \centering
    \begin{subfigure}[b]{0.24\textwidth}
      \centering
      \includegraphics[width=\textwidth]{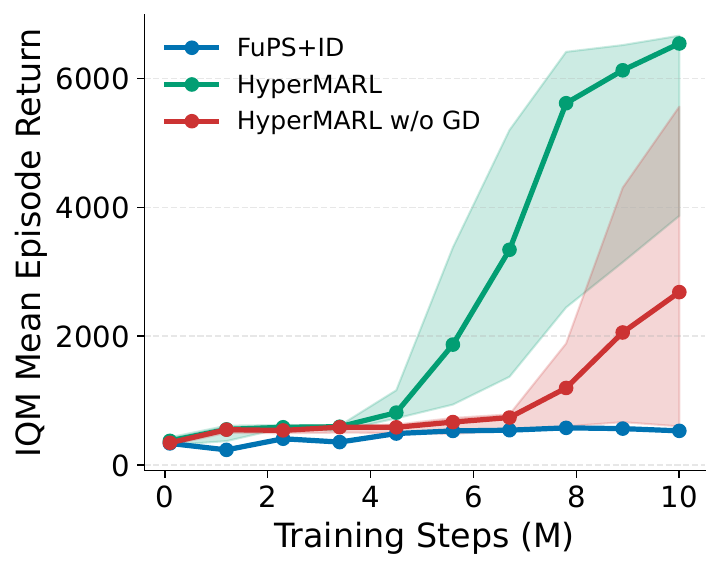}
      \caption{Humanoid w/o GD}
      \label{fig:ablation-humanoid-gd}
    \end{subfigure}
    \begin{subfigure}[b]{0.24\textwidth}
      \centering
      \includegraphics[width=\textwidth]{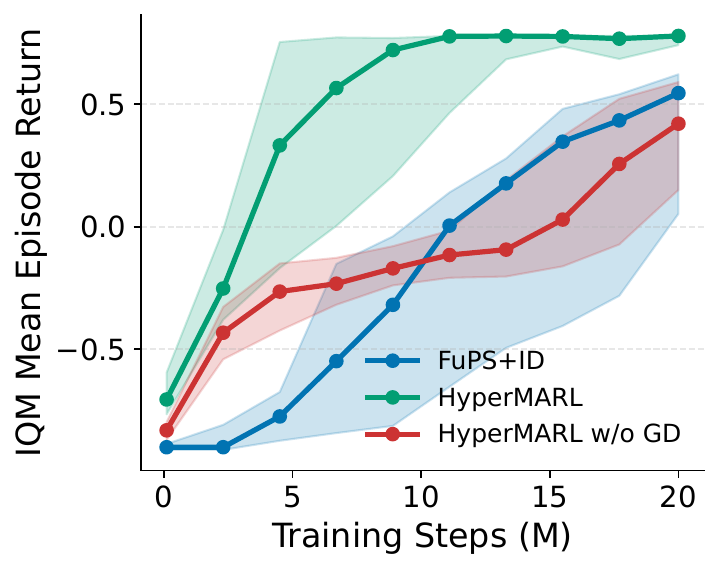}
      \caption{Dispersion w/o GD}
      \label{fig:ablation-dispersion-gd}
    \end{subfigure}
    \begin{subfigure}[b]{0.24\textwidth}
      \centering
      \includegraphics[width=\textwidth]{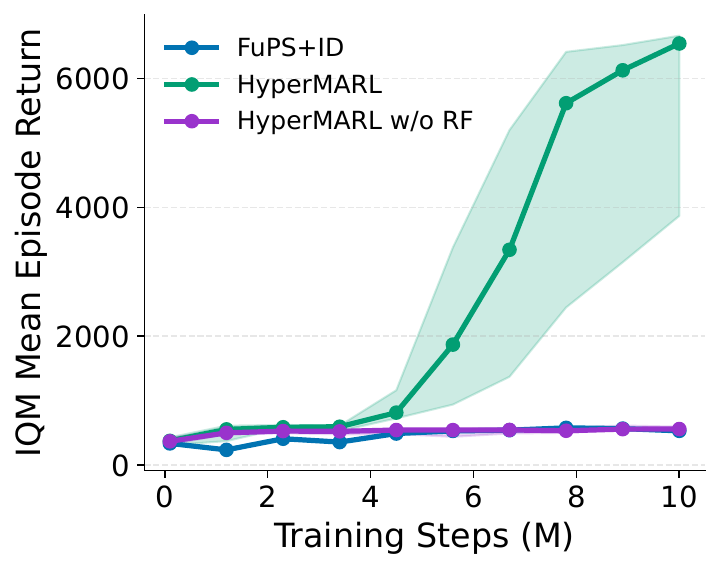}
      \caption{Humanoid w/o RF}
      \label{fig:ablation-humanoid-rf}
    \end{subfigure}
    \begin{subfigure}[b]{0.24\textwidth}
      \centering
      \includegraphics[width=\textwidth]{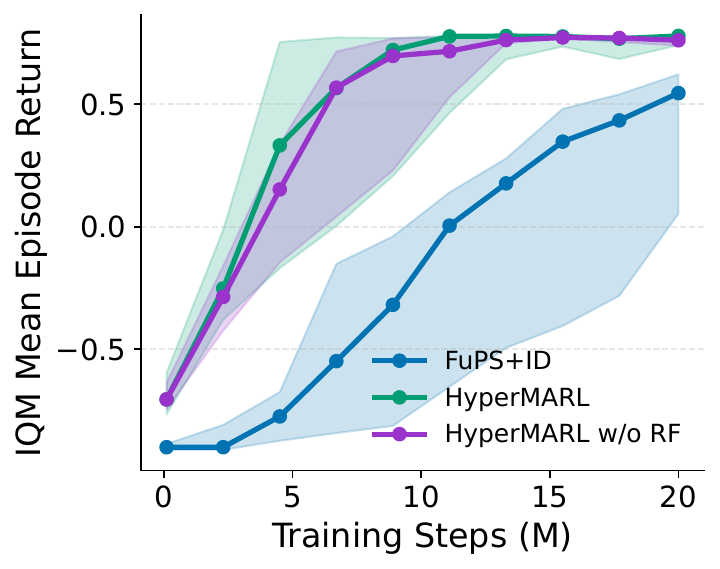}
      \caption{Dispersion w/o RF}
      \label{fig:ablation-dispersion-rf}
    \end{subfigure}
    \caption{\textit{Ablation results comparing HyperMARL to variants without gradient decoupling (w/o GD) and without reset fan in/out initialisation (w/o RF) across environments.} Gradient decoupling (a,b) is consistently critical across both environments, while initialisation scaling (c,d) shows greater importance in the complex Humanoid task but less impact in the simpler Dispersion environment.}
    \label{fig:ablations}
  \end{minipage}
  \vspace{-10pt}
\end{figure}

\begin{wrapfigure}[15]{r}{0.35\textwidth}
  \centering
  \vspace{-30pt} %
  \includegraphics[width=\linewidth]{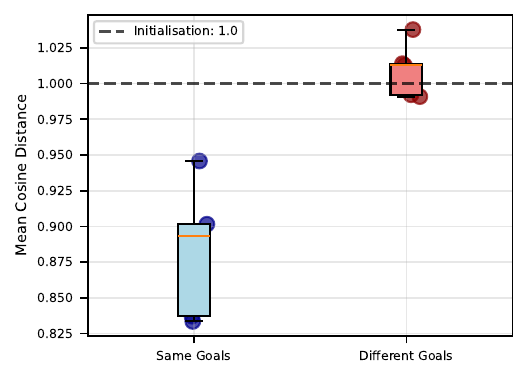}
  \caption{\textbf{Embedding similarity reflects task demands.} Mean pairwise cosine distance (dashed line = 1.0). Same goal: contraction ($0.882\pm0.042$). Different goals: near-orthogonal ($1.010\pm0.017$). Embeddings contract when a shared policy is optimal and maintain separation for specialisation.}
  \label{fig:embedding_distance}
\end{wrapfigure}

\subsection{Analysis of Learned Agent Embeddings}
\label{subsec:embedding_analysis}

Recall from \ref{subsec:agent_context_and_init} that each agent $i$ possesses an embedding $e^i$, which serves as input to the hypernetworks. For the MLP hypernetworks, these embeddings are learned end-to-end and are \emph{orthogonally initialised}. Consequently, at initialisation (step 0), the pairwise cosine distance between any two distinct agent embeddings is $1.0$, mirroring the separability of one-hot IDs.

We probe how these embeddings $e^i$ adapt in the 4-agent Navigation task (same task as in Figure \ref{app:dico_results}) under two objectives with identical dynamics: (i) \emph{same goal} (all agents navigate to a single shared target) and (ii) \emph{different goals} (each agent navigates to its own unique target). At the end of training, we compute the mean pairwise cosine distance between agents' embeddings (lower values imply greater similarity) and compare it to the orthogonal baseline of $1.0$. (Cosine distance $= 1 - \text{cosine similarity}$; $0$ = identical direction, $1$ = orthogonal, $2$ = opposite). Figure~\ref{fig:embedding_distance} shows the per-seed distributions, with the $1.0$ initialisation value as a dashed line.

\textbf{Agent embeddings adapt to task demands.} When behaviour should be homogeneous (\emph{same goal}), embedding directions become significantly more aligned, the mean pairwise cosine distance \emph{contracts} to $0.882 \pm 0.042$ (one-sample $t$-test vs $1.0$: $p=0.0079$). Conversely, when behaviour must differ (\emph{different goals}), the embeddings \emph{maintain their separation}, remaining near their orthogonal initialisation at $1.010 \pm 0.017$. These results demonstrate that the hypernetwork actively modulates the agent embeddings based on the task, promoting alignment for homogeneity while preserving separability for specialisation. We also conduct a sensitivity analysis on HyperMARL's hyperparameters in App~\ref{app:sensitivity} and see that agent embedding size can be an important hyperparameter, one that could correspond to the task's diversity requirements.

\section{Related Work}\label{sec:related_work}

\textbf{Hypernetworks in RL and MARL.} Hypernetworks are effective in single-agent RL for meta-learning, multi-task learning, and continual learning~\citep{beck2023hypernetworks,beck2024recurrent,sarafian2021recomposing,huang2021continual}. In MARL, QMIX~\citep{rashid2020monotonic} used hypernetworks (conditioned on a global state) to mix per-agent Q-values; however, each agent’s own network remained a standard GRU. Parallel work, CASH~\citep{fu2025learning}, conditions a hypernetwork on local observations and predefined capability descriptors to target zero-shot generalisation across heterogeneous action spaces. By contrast, we use agent-conditioned hypernetworks in homogeneous action spaces, conditioning only on agent IDs/learned embeddings, and we explicitly decouple agent-conditioned from observation-conditioned gradients -- a mechanism absent in CASH -- which we find critical for specialisation.

\textbf{Variants of Parameter Sharing.} While Full Parameter Sharing (FuPS) is the most common approach, several other variants exist. Selective Parameter Sharing (SePS)~\citep{christianos2021scaling} shares weights between similar groups of agents, identified by clustering agent trajectories during a pre-training phase. Pruning methods (SNP-PS, Kaleidoscope)~\citep{kim2023parameter,li2024kaleidoscope} split a single network into agent-specific subnetworks using learned agent masks. AdaPS~\citep{li2024adaptive} combines clustering and pruning masks to create shared policies for different groups of agents. Concurrent to our work, GradPS~\citep{qingradps} identifies neurons with conflicting gradient updates, clones these neurons, and assigns each clone to a group of agents with low gradient conflict. Unlike these works, HyperMARL does not rely on pre-training (SePS), clustering algorithms requiring a preset number of clusters (SePS, GradPS, AdaPS), intricate pruning hyperparameters (SNP-PS, Kaleidoscope), auxiliary diversity losses (Kaleidoscope), or gradient conflict thresholds (GradPS).

\textbf{Learning Diverse Policies.} Shared parameters often limit policy diversity~\citep{christianos2021scaling,kim2023parameter,fu2022revisiting,li2021celebrating}.
Proposed solutions include: maximising mutual information between agent IDs and trajectories~\citep{li2021celebrating}, role-based methods that assign distinct roles to agents~\citep{wang2020roma,wang2020rode}, best-response~\citep{rahman2023BRDiv} methods and approaches that use structural modifications or constraints to induce diversity~\citep{kim2023parameter,bettini2024controlling,li2024kaleidoscope,li2024adaptive}.
Outside FuPS/NoPS, HAPPO~\citep{JMLR:v25:23-0488} uses a non-shared centralised critic with individual actors updated sequentially to learn heterogeneous behaviours. In contrast to these works, HyperMARL does not alter the learning objective, use sequential updates or require preset diversity levels.

\section{Conclusion}

We investigated why standard parameter sharing fails at behavioural diversity, identifying that cross-agent gradient interference is critically exacerbated by coupling agent IDs with observations. We hypothesised that explicitly \emph{decoupling} these gradients would enable adaptivity without prior complexities, and confirmed this using our \emph{HyperMARL} approach. Our results show this decoupling enables adaptive behaviours (up to 30 agents) and is linked to reduced policy gradient variance. These findings establish gradient decoupling via HyperMARL as a versatile, principled approach for adaptive MARL. We discuss limitations in App.~\ref{app:limiations}, most notably parameter count, which can be remedied by parameter-efficient hypernetworks (e.g., chunked variants~\citep{Oswald2020Continual,pmlr-v238-chauhan24a}).

\section{Acknowledgements}
We would like to thank Samuel Garcin, Max Tamborski, Dave Abel, Timothy Hospedales, Trevor Mcinroe, Elliot Fosong, and Aris Filos-Ratsikas for fruitful discussions on early versions of this work. We also wish to acknowledge the anonymous reviewers for their constructive comments and feedback that helped strengthen this work.

{
\small

\bibliographystyle{plain}
\bibliography{ref}

\begin{thebibliography}{10}

\bibitem{agarwal2021deep}
Rishabh Agarwal, Max Schwarzer, Pablo~Samuel Castro, Aaron Courville, and Marc~G Bellemare.
\newblock Deep reinforcement learning at the edge of the statistical precipice.
\newblock {\em Advances in Neural Information Processing Systems}, 2021.

\bibitem{albrecht2024multi}
Stefano~V Albrecht, Filippos Christianos, and Lukas Sch{\"a}fer.
\newblock {\em Multi-agent reinforcement learning: Foundations and modern approaches}.
\newblock MIT Press, 2024.

\bibitem{beck2023hypernetworks}
Jacob Beck, Matthew~Thomas Jackson, Risto Vuorio, and Shimon Whiteson.
\newblock Hypernetworks in meta-reinforcement learning.
\newblock In {\em Conference on Robot Learning}, pages 1478--1487. PMLR, 2023.

\bibitem{beck2024recurrent}
Jacob Beck, Risto Vuorio, Zheng Xiong, and Shimon Whiteson.
\newblock Recurrent hypernetworks are surprisingly strong in meta-rl.
\newblock {\em Advances in Neural Information Processing Systems}, 36, 2024.

\bibitem{bettini2022vmas}
Matteo Bettini, Ryan Kortvelesy, Jan Blumenkamp, and Amanda Prorok.
\newblock Vmas: A vectorized multi-agent simulator for collective robot learning.
\newblock {\em The 16th International Symposium on Distributed Autonomous Robotic Systems}, 2022.

\bibitem{bettini2024controlling}
Matteo Bettini, Ryan Kortvelesy, and Amanda Prorok.
\newblock Controlling behavioral diversity in multi-agent reinforcement learning.
\newblock In {\em Forty-first International Conference on Machine Learning}, 2024.

\bibitem{bettini2023system}
Matteo Bettini, Ajay Shankar, and Amanda Prorok.
\newblock System neural diversity: Measuring behavioral heterogeneity in multi-agent learning.
\newblock {\em arXiv preprint arXiv:2305.02128}, 2023.

\bibitem{Chang2020Principled}
Oscar Chang, Lampros Flokas, and Hod Lipson.
\newblock Principled weight initialization for hypernetworks.
\newblock In {\em International Conference on Learning Representations}, 2020.

\bibitem{pmlr-v238-chauhan24a}
Vinod~Kumar Chauhan, Jiandong Zhou, Ghadeer Ghosheh, Soheila Molaei, and David A~Clifton.
\newblock Dynamic inter-treatment information sharing for individualized treatment effects estimation.
\newblock In Sanjoy Dasgupta, Stephan Mandt, and Yingzhen Li, editors, {\em Proceedings of The 27th International Conference on Artificial Intelligence and Statistics}, volume 238 of {\em Proceedings of Machine Learning Research}, pages 3529--3537. PMLR, 02--04 May 2024.

\bibitem{christianos2023pareto}
Filippos Christianos, Georgios Papoudakis, and Stefano~V. Albrecht.
\newblock Pareto actor-critic for equilibrium selection in multi-agent reinforcement learning.
\newblock {\em Transactions on Machine Learning Research (TMLR)}, 2023.

\bibitem{christianos2021scaling}
Filippos Christianos, Georgios Papoudakis, Muhammad~A Rahman, and Stefano~V Albrecht.
\newblock Scaling multi-agent reinforcement learning with selective parameter sharing.
\newblock In {\em International Conference on Machine Learning}, pages 1989--1998. PMLR, 2021.

\bibitem{christianos2020shared}
Filippos Christianos, Lukas Sch\"afer, and Stefano~V. Albrecht.
\newblock Shared experience actor-critic for multi-agent reinforcement learning.
\newblock In {\em 34th Conference on Neural Information Processing Systems}, 2020.

\bibitem{de2020independent}
Christian~Schroeder De~Witt, Tarun Gupta, Denys Makoviichuk, Viktor Makoviychuk, Philip~HS Torr, Mingfei Sun, and Shimon Whiteson.
\newblock Is independent learning all you need in the starcraft multi-agent challenge?
\newblock {\em arXiv preprint arXiv:2011.09533}, 2020.

\bibitem{endres2003new}
Dominik~Maria Endres and Johannes~E Schindelin.
\newblock A new metric for probability distributions.
\newblock {\em IEEE Transactions on Information theory}, 49(7):1858--1860, 2003.

\bibitem{foerster2016learning}
Jakob Foerster, Ioannis~Alexandros Assael, Nando De~Freitas, and Shimon Whiteson.
\newblock Learning to communicate with deep multi-agent reinforcement learning.
\newblock {\em Advances in neural information processing systems}, 29, 2016.

\bibitem{fu2025learning}
Kevin Fu, Pierce Howell, Shalin Jain, and Harish Ravichandar.
\newblock Learning flexible heterogeneous coordination with capability-aware shared hypernetworks.
\newblock {\em arXiv preprint arXiv:2501.06058}, 2025.

\bibitem{fu2022revisiting}
Wei Fu, Chao Yu, Zelai Xu, Jiaqi Yang, and Yi~Wu.
\newblock Revisiting some common practices in cooperative multi-agent reinforcement learning.
\newblock In {\em International Conference on Machine Learning}, pages 6863--6877. PMLR, 2022.

\bibitem{gupta2017cooperative}
Jayesh~K Gupta, Maxim Egorov, and Mykel Kochenderfer.
\newblock Cooperative multi-agent control using deep reinforcement learning.
\newblock In {\em Autonomous Agents and Multiagent Systems: AAMAS 2017 Workshops, Best Papers, S{\~a}o Paulo, Brazil, May 8-12, 2017, Revised Selected Papers 16}, pages 66--83. Springer, 2017.

\bibitem{ha2016hypernetworks}
David Ha, Andrew Dai, and Quoc~V Le.
\newblock Hypernetworks.
\newblock {\em arXiv preprint arXiv:1609.09106}, 2016.

\bibitem{hornik1989multilayer}
Kurt Hornik, Maxwell Stinchcombe, and Halbert White.
\newblock Multilayer feedforward networks are universal approximators.
\newblock {\em Neural networks}, 2(5):359--366, 1989.

\bibitem{huang2021continual}
Yizhou Huang, Kevin Xie, Homanga Bharadhwaj, and Florian Shkurti.
\newblock Continual model-based reinforcement learning with hypernetworks.
\newblock In {\em 2021 IEEE International Conference on Robotics and Automation (ICRA)}, pages 799--805. IEEE, 2021.

\bibitem{pmlr-v139-jiang21g}
Jiechuan Jiang and Zongqing Lu.
\newblock The emergence of individuality.
\newblock In Marina Meila and Tong Zhang, editors, {\em Proceedings of the 38th International Conference on Machine Learning}, volume 139 of {\em Proceedings of Machine Learning Research}, pages 4992--5001. PMLR, 18--24 Jul 2021.

\bibitem{kassen2002experimental}
R~Kassen.
\newblock The experimental evolution of specialists, generalists, and the maintenance of diversity.
\newblock {\em Journal of evolutionary biology}, 15(2):173--190, 2002.

\bibitem{kim2023parameter}
Woojun Kim and Youngchul Sung.
\newblock Parameter sharing with network pruning for scalable multi-agent deep reinforcement learning.
\newblock {\em arXiv preprint arXiv:2303.00912}, 2023.

\bibitem{kuba2021settling}
Jakub~Grudzien Kuba, Muning Wen, Linghui Meng, Haifeng Zhang, David Mguni, Jun Wang, Yaodong Yang, et~al.
\newblock Settling the variance of multi-agent policy gradients.
\newblock {\em Advances in Neural Information Processing Systems}, 34:13458--13470, 2021.

\bibitem{li2021celebrating}
Chenghao Li, Tonghan Wang, Chengjie Wu, Qianchuan Zhao, Jun Yang, and Chongjie Zhang.
\newblock Celebrating diversity in shared multi-agent reinforcement learning.
\newblock {\em Advances in Neural Information Processing Systems}, 34:3991--4002, 2021.

\bibitem{li2024adaptive}
Dapeng Li, Na~Lou, Bin Zhang, Zhiwei Xu, and Guoliang Fan.
\newblock Adaptive parameter sharing for multi-agent reinforcement learning.
\newblock In {\em ICASSP 2024-2024 IEEE International Conference on Acoustics, Speech and Signal Processing (ICASSP)}, pages 6035--6039. IEEE, 2024.

\bibitem{li2024kaleidoscope}
Xinran Li, Ling Pan, and Jun Zhang.
\newblock Kaleidoscope: Learnable masks for heterogeneous multi-agent reinforcement learning.
\newblock {\em arXiv preprint arXiv:2410.08540}, 2024.

\bibitem{lin1991divergence}
Jianhua Lin.
\newblock Divergence measures based on the shannon entropy.
\newblock {\em IEEE Transactions on Information theory}, 37(1):145--151, 1991.

\bibitem{lowe2017multi}
Ryan Lowe, Yi~I Wu, Aviv Tamar, Jean Harb, OpenAI Pieter~Abbeel, and Igor Mordatch.
\newblock Multi-agent actor-critic for mixed cooperative-competitive environments.
\newblock {\em Advances in neural information processing systems}, 30, 2017.

\bibitem{mckee2022quantifying}
Kevin~R McKee, Joel~Z Leibo, Charlie Beattie, and Richard Everett.
\newblock Quantifying the effects of environment and population diversity in multi-agent reinforcement learning.
\newblock {\em Autonomous Agents and Multi-Agent Systems}, 36(1):21, 2022.

\bibitem{mnih2016asynchronous}
Volodymyr Mnih, Adria~Puigdomenech Badia, Mehdi Mirza, Alex Graves, Timothy Lillicrap, Tim Harley, David Silver, and Koray Kavukcuoglu.
\newblock Asynchronous methods for deep reinforcement learning.
\newblock In {\em International conference on machine learning}, pages 1928--1937. PmLR, 2016.

\bibitem{navon2020learning}
Aviv Navon, Aviv Shamsian, Gal Chechik, and Ethan Fetaya.
\newblock Learning the pareto front with hypernetworks.
\newblock {\em arXiv preprint arXiv:2010.04104}, 2020.

\bibitem{oliehoek2016concise}
Frans~A Oliehoek and Christopher Amato.
\newblock {\em A concise introduction to decentralized POMDPs}.
\newblock Springer, 2016.

\bibitem{osborne1994course}
Martin~J Osborne and Ariel Rubinstein.
\newblock {\em A course in game theory}.
\newblock MIT press, 1994.

\bibitem{patterson2024empirical}
Andrew Patterson, Samuel Neumann, Martha White, and Adam White.
\newblock Empirical design in reinforcement learning.
\newblock {\em Journal of Machine Learning Research}, 25(318):1--63, 2024.

\bibitem{peng2021facmac}
Bei Peng, Tabish Rashid, Christian Schroeder~de Witt, Pierre-Alexandre Kamienny, Philip Torr, Wendelin B{\"o}hmer, and Shimon Whiteson.
\newblock Facmac: Factored multi-agent centralised policy gradients.
\newblock {\em Advances in Neural Information Processing Systems}, 34:12208--12221, 2021.

\bibitem{qingradps}
Haoyuan Qin, Zhengzhu Liu, Chenxing Lin, Chennan Ma, Songzhu Mei, Siqi Shen, and Cheng Wang.
\newblock Gradps: Resolving futile neurons in parameter sharing network for multi-agent reinforcement learning.
\newblock In {\em Forty-second International Conference on Machine Learning}, 2025.

\bibitem{rahman2023BRDiv}
Arrasy Rahman, Elliot Fosong, Ignacio Carlucho, and Stefano~V. Albrecht.
\newblock Generating teammates for training robust ad hoc teamwork agents via best-response diversity.
\newblock {\em Transactions on Machine Learning Research (TMLR)}, 2023.

\bibitem{rashid2020monotonic}
Tabish Rashid, Mikayel Samvelyan, Christian~Schroeder De~Witt, Gregory Farquhar, Jakob Foerster, and Shimon Whiteson.
\newblock Monotonic value function factorisation for deep multi-agent reinforcement learning.
\newblock {\em Journal of Machine Learning Research}, 21(178):1--51, 2020.

\bibitem{rutherford2024jaxmarl}
Alexander Rutherford, Benjamin Ellis, Matteo Gallici, Jonathan Cook, Andrei Lupu, Gar{\dh}ar Ingvarsson, Timon Willi, Akbir Khan, Christian Schroeder~de Witt, Alexandra Souly, et~al.
\newblock Jaxmarl: Multi-agent rl environments and algorithms in jax.
\newblock In {\em Proceedings of the 23rd International Conference on Autonomous Agents and Multiagent Systems}, pages 2444--2446, 2024.

\bibitem{sarafian2021recomposing}
Elad Sarafian, Shai Keynan, and Sarit Kraus.
\newblock Recomposing the reinforcement learning building blocks with hypernetworks.
\newblock In {\em International Conference on Machine Learning}, pages 9301--9312. PMLR, 2021.

\bibitem{smith2008genetic}
Chris~R Smith, Amy~L Toth, Andrew~V Suarez, and Gene~E Robinson.
\newblock Genetic and genomic analyses of the division of labour in insect societies.
\newblock {\em Nature Reviews Genetics}, 9(10):735--748, 2008.

\bibitem{surowiecki2004wisdom}
James Surowiecki.
\newblock {\em The Wisdom of Crowds}.
\newblock Doubleday, New York, 2004.

\bibitem{tan1993multi}
Ming Tan.
\newblock Multi-agent reinforcement learning: Independent vs. cooperative agents.
\newblock In {\em Proceedings of the tenth international conference on machine learning}, pages 330--337, 1993.

\bibitem{vaserstein1969markov}
Leonid~Nisonovich Vaserstein.
\newblock Markov processes over denumerable products of spaces, describing large systems of automata.
\newblock {\em Problemy Peredachi Informatsii}, 5(3):64--72, 1969.

\bibitem{Oswald2020Continual}
Johannes von Oswald, Christian Henning, Benjamin~F. Grewe, and João Sacramento.
\newblock Continual learning with hypernetworks.
\newblock In {\em International Conference on Learning Representations}, 2020.

\bibitem{wang2020roma}
Tonghan Wang, Heng Dong, Victor Lesser, and Chongjie Zhang.
\newblock Roma: multi-agent reinforcement learning with emergent roles.
\newblock In {\em Proceedings of the 37th International Conference on Machine Learning}, pages 9876--9886, 2020.

\bibitem{wang2020rode}
Tonghan Wang, Tarun Gupta, Anuj Mahajan, Bei Peng, Shimon Whiteson, and Chongjie Zhang.
\newblock Rode: Learning roles to decompose multi-agent tasks.
\newblock {\em arXiv preprint arXiv:2010.01523}, 2020.

\bibitem{williams1998demography}
Katherine~Y Williams and Charles~A O'Reilly~III.
\newblock Demography and.
\newblock {\em Research in organizational behavior}, 20:77--140, 1998.

\bibitem{williams1992simple}
Ronald~J Williams.
\newblock Simple statistical gradient-following algorithms for connectionist reinforcement learning.
\newblock {\em Machine learning}, 8:229--256, 1992.

\bibitem{woolley2015collective}
Anita~Williams Woolley, Ishani Aggarwal, and Thomas~W Malone.
\newblock Collective intelligence and group performance.
\newblock {\em Current Directions in Psychological Science}, 24(6):420--424, 2015.

\bibitem{yu2022surprising}
Chao Yu, Akash Velu, Eugene Vinitsky, Jiaxuan Gao, Yu~Wang, Alexandre Bayen, and Yi~Wu.
\newblock The surprising effectiveness of ppo in cooperative multi-agent games.
\newblock {\em Advances in Neural Information Processing Systems}, 35:24611--24624, 2022.

\bibitem{JMLR:v25:23-0488}
Yifan Zhong, Jakub~Grudzien Kuba, Xidong Feng, Siyi Hu, Jiaming Ji, and Yaodong Yang.
\newblock Heterogeneous-agent reinforcement learning.
\newblock {\em Journal of Machine Learning Research}, 25(32):1--67, 2024.

\end{thebibliography}

}

\newpage 
\appendix 

\section{Limitations}\label{app:limiations}
Hypernetworks generate weights for target networks, which can lead to high-dimensional outputs and many parameters for deep target networks, particularly when using MLP-based hypernetworks. While HyperMARL uses more parameters than NoPS and FuPS for few agents, it scales almost constantly with agent count, an attractive property for large-scale MARL. Parameter efficiency could be improved through chunking techniques \citep{Oswald2020Continual,pmlr-v238-chauhan24a}, or low-rank weight approximations. This parameter overhead is often acceptable in RL/MARL given typically smaller actor-critic networks, and HyperMARL's favorable agent scaling (see App.~\ref{sec:scaling}).

\section{Broader Impact}\label{app:braoder_impact}
This paper presents work whose goal is to advance the field of Multi-Agent Reinforcement Learning. There are many potential societal consequences of our work, none which we feel must be specifically highlighted here.
\section{Specialised Policies and Environments}\label{app:spec_policy_env}

Specialisation plays a key role in MARL, yet remains under-defined, so we define \textit{specialised environments} and \textit{specialised policies}.

\begin{definition} 
\label{def:specialised_env}
   \raggedright An environment is \textit{specialised} if the following both hold:
    \begin{enumerate}[leftmargin=*]
        \item \textbf{Distinct Agent Policies.} The optimal joint policy $\boldsymbol{\pi}^*$ consists of at least two distinct agent policies, i.e., $\exists i, j \in \mathbb{I}$ such that  $\pi^i \neq \pi^j$.
        
     \item \textbf{Non-Interchangeability.} Any permutation $\sigma$ of the policies in $\boldsymbol{\pi}^*$, denoted as $\boldsymbol{\pi}^\sigma$, results in a weakly lower expected return:
        $$
        \mathbb{E}_{\mathbf{h} \sim \boldsymbol{\pi}^\sigma}[G(\mathbf{h})] \leq \mathbb{E}_{\mathbf{h} \sim \boldsymbol{\pi}^*}[G(\mathbf{h})],
        $$
        with strict inequality if the joint policies are \emph{non-symmetric} (i.e., swapping any individual policy degrades performance).
    \end{enumerate}
\end{definition}
For example, consider a \textit{specialised environment} such as a football game, optimal team performance typically requires players in distinct roles (e.g., "attackers," "defenders"). Permuting these roles (i.e., exchanging their policies) would typically lead to suboptimal results. Here, agents develop \textit{specialised policies} by learning distinct, complementary behaviours essential for an optimal joint policy. While agents with heterogeneous capabilities (e.g., different action spaces) are inherently specialised, homogeneous agents can also learn distinct policies. Such environments are analysed in Sections \ref{sec:spec_game} and \ref{results:dispersion}.
\section{Measuring Behavioural Diversity}\label{app:beh_diversity}
\subsection{Quantifying Team Diversity}
We quantify policy diversity using System Neural Diversity (SND)~\citep{bettini2023system}, which measures behavioural diversity based on differences in policy outputs: 
\begin{equation} \label{eq:snd}
\operatorname{SND}\left(\left\{\pi^i\right\}_{i \in \mathbb{I}}\right) = 
\frac{2}{n(n-1)|\mathcal{O}|} \sum_{i=1}^n \sum_{j=i+1}^n \sum_{o \in \mathcal{O}} 
D\left(\pi^i(o), \pi^j(o)\right).
\end{equation}
where $n$ is the number of agents, $\mathcal{O}$ is a set of observations typically collected via policy rollouts, $\pi^{i}(o_t)$ and $\pi^{j}(o_t)$ are the outputs of policies $i$ and $j$ for observation $o_t$, and $D$ is our distance function between two probability distributions. 

In contrast to ~\cite{bettini2023system}, we use Jensen-Shannon Distance (JSD) \citep{endres2003new,lin1991divergence} as $D$, rather than the Wasserstein metric~\citep{vaserstein1969markov}. As shown in Appendix \ref{appen:distance_function}, JSD is a robust metric for both continuous and discrete cases, and provides a more reliable measure of policy distance. 
\subsection{Finding a Suitable Distance Function for Policy Diversity} \label{appen:distance_function}

The choice of distance function $D$ in Equation \ref{eq:snd} is crucial for accurately measuring policy diversity. In MARL, policies are often represented as probability distributions over actions, making the choice of distance function non-trivial. 

\cite{bettini2024controlling} use the Wasserstein metric for continuous policies \citep{vaserstein1969markov} as distance function $D$, while \cite{mckee2022quantifying} use the total variation distance for discrete policies. For discrete policies, Wasserstein distance would require a cost function representing the cost of changing from one action to another, which might not be trivial to come up with. On the other hand, although well-suited for discrete policies, TVD might miss changes in action probabilities because it measures the largest difference assigned to an event (i.e. action) between two probability distributions. 

We consider a simple example to illustrate this point. Suppose we have two policies $\pi^1$ and $\pi^2$ with action probabilities as shown in Figure \ref{fig:tvd_vs_sqrt_shannon}. $\pi^1$ stays constant, while $\pi^2$ changes gradually over timesteps. We see that even as $\pi^2$ changes over time, the $TVD(\pi^1,\pi^2)$ between $\pi^1$ and $\pi^2$ remains constant. This is because TVD only measures the largest difference between the two distributions, and does not consider the overall difference between them. On the other hand, the Jensen-Shannon distance (JSD) \citep{endres2003new}, which is the square root of the Jensen-Shannon divergence, does not have this problem as it is a smooth distance function. Furthermore, it satisfies the conditions for being a metric -- it is non-negative, symmetric, and it satisfies the triangle inequality. 

For continuous policies, as shown in Figure \ref{fig:wasserstein_vs_sqrt_shannon}, JSD exhibits similar trends to the Wasserstein distance. Since JSD is a reasonable metric for both continuous and discrete probability distributions, we will use it as the distance metric for all experiments and propose it as a suitable distance function for measuring policy diversity in MARL. 

We also summarise the properties of the various distance metrics in Table \ref{tab:distance_metrics}.

\begin{figure}[h]
    \centering
    \includegraphics[width=0.9\textwidth]{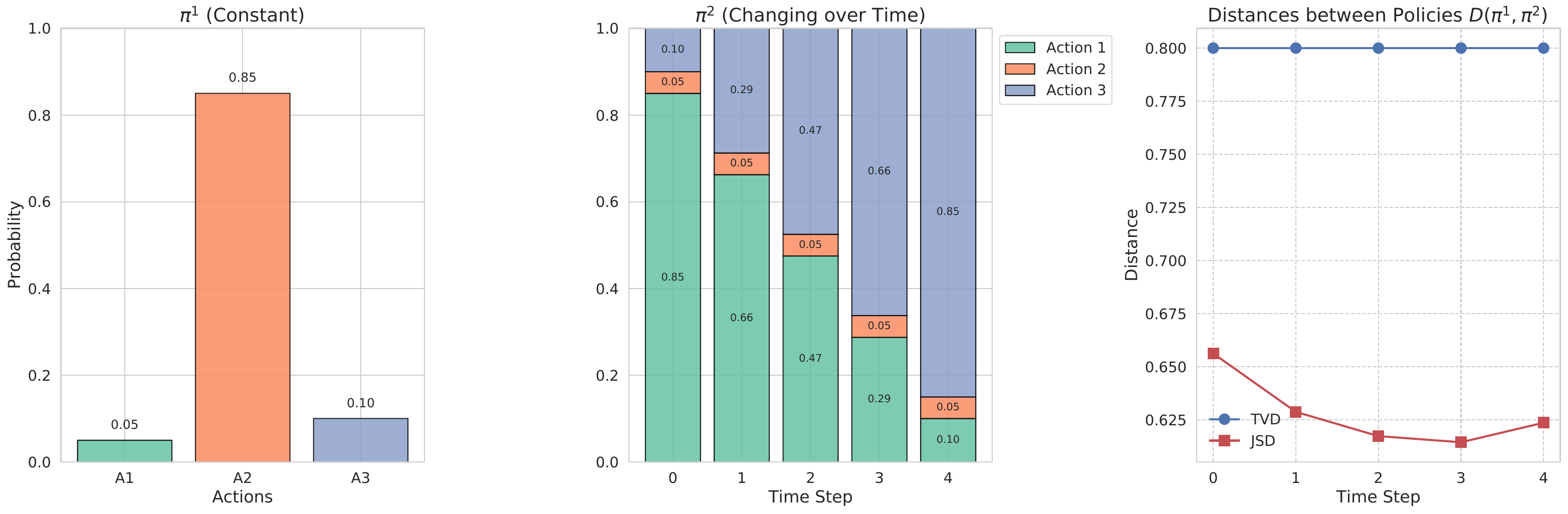}
    \caption{Gradual changes in $\pi^2$, result in gradual changes in the Jensen-Shannon distance (JSD), while the Total Variation Distance (TVD) can miss changes in action probabilities.}
    \label{fig:tvd_vs_sqrt_shannon}
\end{figure}

\begin{figure}[h]
    \centering
    \includegraphics[width=0.9\textwidth]{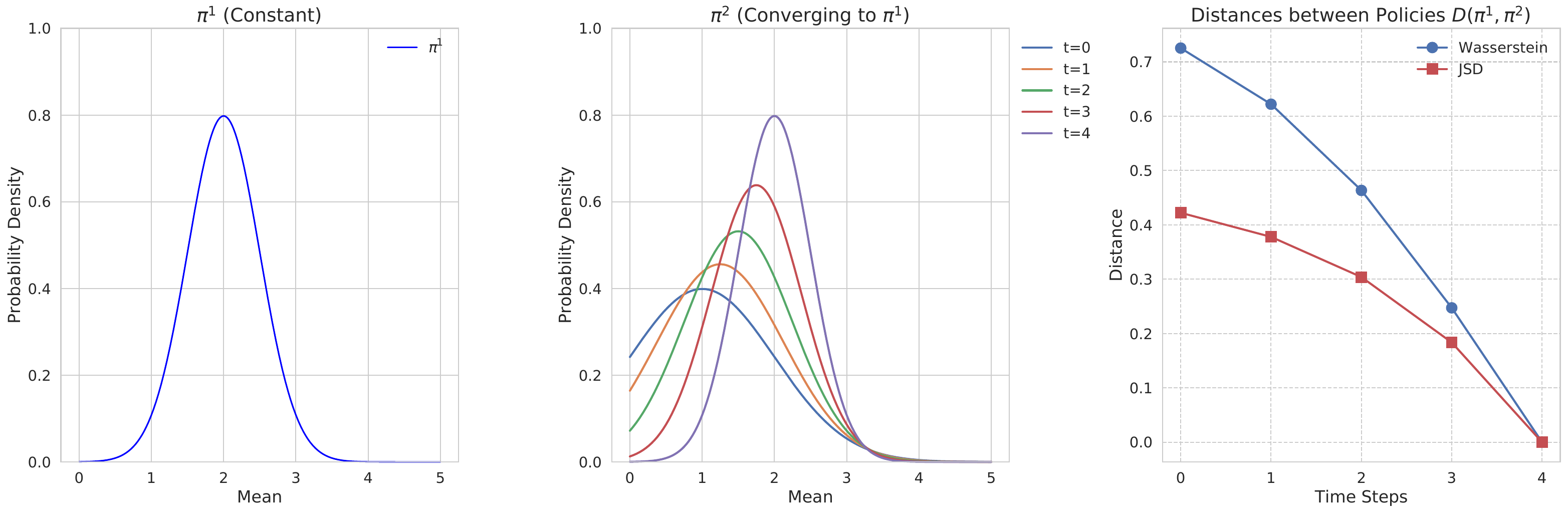}
    \caption{Jensen-Shannon distance (JSD) trends similarly to Wasserstein distance when we have continuous policies.}
    \label{fig:wasserstein_vs_sqrt_shannon}
\end{figure}

\clearpage %
\begin{table}[htbp]
    \centering
    \resizebox{1.1\textwidth}{!}{%
    \begin{tabular}{|c|c|c|c|c|}
        \hline
        \textbf{Method} & \textbf{Kinds of Actions} & \textbf{Metric} & \textbf{Smooth} & \textbf{Formula} \\
        \hline
        \text{Wasserstein Distance \citep{vaserstein1969markov}} & \text{Continuous*} & \text{Metric} & \text{Yes} & $W(p, q) = \left( \inf_{\gamma \in \Gamma(p, q)} \int_{\mathbb{R} \times \mathbb{R}} |x-y| \, d\gamma(x,y) \right)^{1/p}$ \\
        \hline
        \text{Total Variation Distance} & \text{Discrete} & \text{Metric} & \text{No} & $TV(p, q) = \frac{1}{2} \sum_{x} |p(x) - q(x)|$ \\
        \hline
        \text{Jensen-Shannon Divergence \citep{lin1991divergence}} & \text{Both} & \text{Divergence} & \text{Yes} & $JSD(p \parallel q) = \frac{1}{2} D_{KL}(p \parallel m) + \frac{1}{2} D_{KL}(q \parallel m), \, m = \frac{1}{2}(p + q)$ \\
        \hline
        \text{Jensen-Shannon Distance \citep{endres2003new}} & \text{Both} & \text{Metric} & \text{Yes} & $\sqrt{JSD(p \parallel q)}$ \\
        \hline
    \end{tabular}%
    }
    \caption{Measures for Policy Diversity}
    \label{tab:distance_metrics}
\end{table}

\clearpage %

\section{Specialisation and Synchronisation Games}\label{appen:spec_game}

To study the challenges of specialisation and coordination in an isolated setting, we introduce the Specialisation and Synchronisation Games, drawing inspiration from a version of the XOR game~\cite{fu2022revisiting}, VMAS's Dispersion~\cite{bettini2022vmas} and coordination and anti-coordination games in game theory~\citep{osborne1994course}. These environments encourage agents to take distinct actions (Specialisation) or take identical actions (Synchronisation). Despite their deceptively simple payoff structure, these games present substantial learning challenges -- non-stationary reward distributions driven by others’ adapting behaviours and in their temporal extension, the joint observation spaces grows exponentially with the number of agents. 
\subsection{Specialisation and Synchronisation Games Description} \label{sec:spec_game}

\textbf{Specialisation Game.} Agents are encouraged to choose \emph{distinct} actions. In the simplest setting, it is a two-player matrix game where each agent selects between two actions ($A$ or $B$). As shown in Figure~\ref{fig:spec-game}, agents receive a payoff of $1.0$ when their actions differ (creating two pure Nash equilibria on the anti-diagonal) and $0.5$ when they match. This structure satisfies Definition~\ref{def:specialised_env}, since optimal joint policies require complementary, non-identical strategies. There is also a symmetric mixed‐strategy equilibrium in which each agent plays $A$ and $B$ with probability $0.5$.

\textbf{Synchronisation Game.} Conversely, agents are encouraged to coordinate and choose \emph{identical} actions. The payoff matrix inverts the Specialisation game's structure, agents receive $1$ for matching actions and $0.5$ for differing ones. This creates two pure Nash equilibria along the diagonal of the payoff matrix (Figure~\ref{fig:sync-game}), and incentivises uniform behaviour across agents.

\textbf{$N$-Agent Extension.} Both games naturally scale to $n$ agents and $n$ possible actions. In Specialisation, unique actions receive a payoff of $1.0$, while selecting the same action receives payoffs of $\frac{1}{k}$, where $k$ is the number of agents choosing that action. In contrast, in Synchronisation, agents receive maximum payoffs ($1.0$) only when all actions match. For partial coordination, rewards follow a hyperbolic scale, $\frac{1}{n-k+1}$, encouraging agents to align their choices. Visualisations in Figure~\ref{fig:game-setup} and detailed reward profiles appear in Figure~\ref{fig:reward-profile-n5}.

\subsection{General-$n$ Payoff Definitions}
\label{app:general-n}

Both games generalise naturally to $n$ agents and $n$ possible actions. We show the reward profiles for $n=5$ agents in Figure \ref{fig:reward-profile-n5}. 

Let $
\mathbf{a} = (a^1, \dots, a^n)\in\{1,\dots,n\}^n$ 
and $k_{a} = \bigl|\{\,j : a^j = a\}\bigr|$
be the joint action profile and the count of agents choosing action \(a\), respectively.

\textbf{Temporal Extension.} To model sequential decision-making, we extend each normal‑form game into a repeated Markov game, where the state at time $t$ is the joint action at time $t-1$. At each step $t$ all agents observe $a^{t-1}$, choose $a_i^{t}$, and receive the original Specialisation or Synchronisation payoff. This repeated setup isolates how agents adapt based on past joint actions, exposing temporal patterns of specialisation and coordination.

\subsubsection{Specialisation Game}
The reward is formulated as follows:
\[
r_{\mathrm{spec}}^{\,i}(\mathbf{a})
= 
\begin{cases}
1.0, 
&\text{if }k_{a^{\,i}} = 1 
\quad\bigl(\text{unique action}\bigr);\\[6pt]
\displaystyle\frac{1}{k_{a^{\,i}}}, 
&\text{if }k_{a^{\,i}} > 1 
\quad\bigl(\text{shared action}\bigr).
\end{cases}
\]

\subsubsection{Synchronisation Game}
The reward is formulated as follows:
\[
r_{\mathrm{sync}}^{\,i}(\mathbf{a})
=\frac{1}{\,n - k_{a^{\,i}} + 1\,},
\]
so that \(r_{\mathrm{sync}}^{\,i}=1.0\) when \(k_{a^{\,i}}=n\) (all agents select the same action), and otherwise follows a hyperbolic scale encouraging consensus.

\begin{figure}[t]
    \centering
    \includegraphics[width=0.60\linewidth]{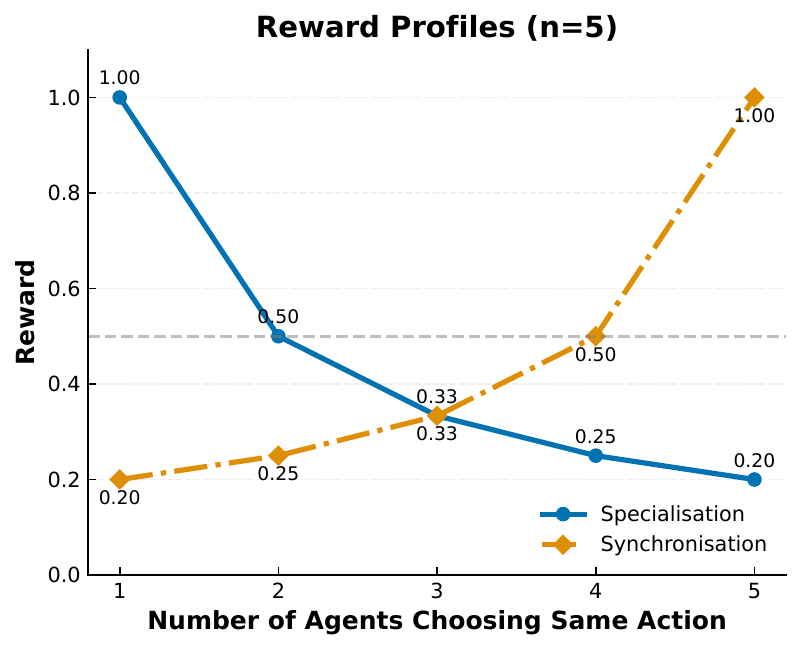}
    \caption{Reward profiles for the Specialisation (blue) and Synchronisation (orange) games with $n=5$ agents. In the Specialisation game, an agent’s payoff peaks when it selects a unique action, and then decays as when actions are shared. In the Synchronisation game, payoffs follow a hyperbolic scale $1/(n - k + 1)$, reaching maximum only under full consensus, thereby incentivising alignment.}
    \label{fig:reward-profile-n5}
\end{figure}

\subsection{Proof that FuPS cannot represent the optimal policy in the two-player Specialisation Game}\label{app:proof_fups_limitations}
\begin{theorem}
A stochastic, shared policy without agent IDs cannot learn the optimal behaviour for the two-player Specialisation Game. 
\end{theorem}

\begin{proof}\label{proof:fups_failure}
Let $\pi$ be a shared policy for both agents, and let $\alpha = \mathbb{P}(a_i = 0)$ represent the probability of any agent choosing action 0. The expected return of $\pi$ for each agent is:

\begin{align}
E[R(\pi)] &= \mathbb{P}(a_0 = 0, a_1 = 0) \cdot 0.5 + \mathbb{P}(a_0 = 0, a_1 = 1) \cdot 1 \\
&\quad + \mathbb{P}(a_0 = 1, a_1 = 0) \cdot 1 + \mathbb{P}(a_0 = 1, a_1 = 1) \cdot 0.5 \\ 
&= 0.5\alpha^2 + 2\alpha(1-\alpha) + 0.5(1-\alpha)^2 \\
&= -\alpha^2 + \alpha + 0.5 \\
&= -(\alpha - 0.5)^2 + 0.75 \label{eq:spec_game}
\end{align}
Thus, \(E[R(\pi)] \leq 0.75 < 1\) for all \(\alpha \in [0,1]\), with the maximum at \(\alpha = 0.5\). Therefore, a shared policy cannot achieve the optimal return of 1, confirming the need for specialised behaviour to optimise rewards.
\end{proof}

\subsection{Measuring Agent Gradient Conflict}\label{appen:gradient_metric_definitions}

We measure \emph{gradient conflicts}, via the cosine similarity between agents’ gradients
$\cos\bigl(g_t^{(i)},g_t^{(j)}\bigr)
=\frac{\langle g_t^{(i)},g_t^{(j)}\rangle}{\|g_t^{(i)}\|\|g_t^{(j)}\|}
$, where $g_t^{(i)}=\nabla_\theta\mathcal L^{(i)}(\theta_t)$.
\subsection{Empirical Results in N-player Specialisation and Synchronisation Normal-Form Game }\label{app:normal_form}

To assess this limitations of FuPS and NoPS in practice, we compare three REINFORCE~\cite{williams1992simple} 
variants in both games with $n = 2, 4, 8, 16, 32$ agents:
NoPS (No Parameter Sharing), 
FuPS (Fully Parameter Sharing), 
and FuPS+ID (FuPS with one-hot agent IDs).
All policies use single-layer neural networks with controlled parameter counts
(see Appendix \ref{append:hyperparams} for details).

\begin{figure}[h]
    \centering
    \includegraphics[width=\linewidth]{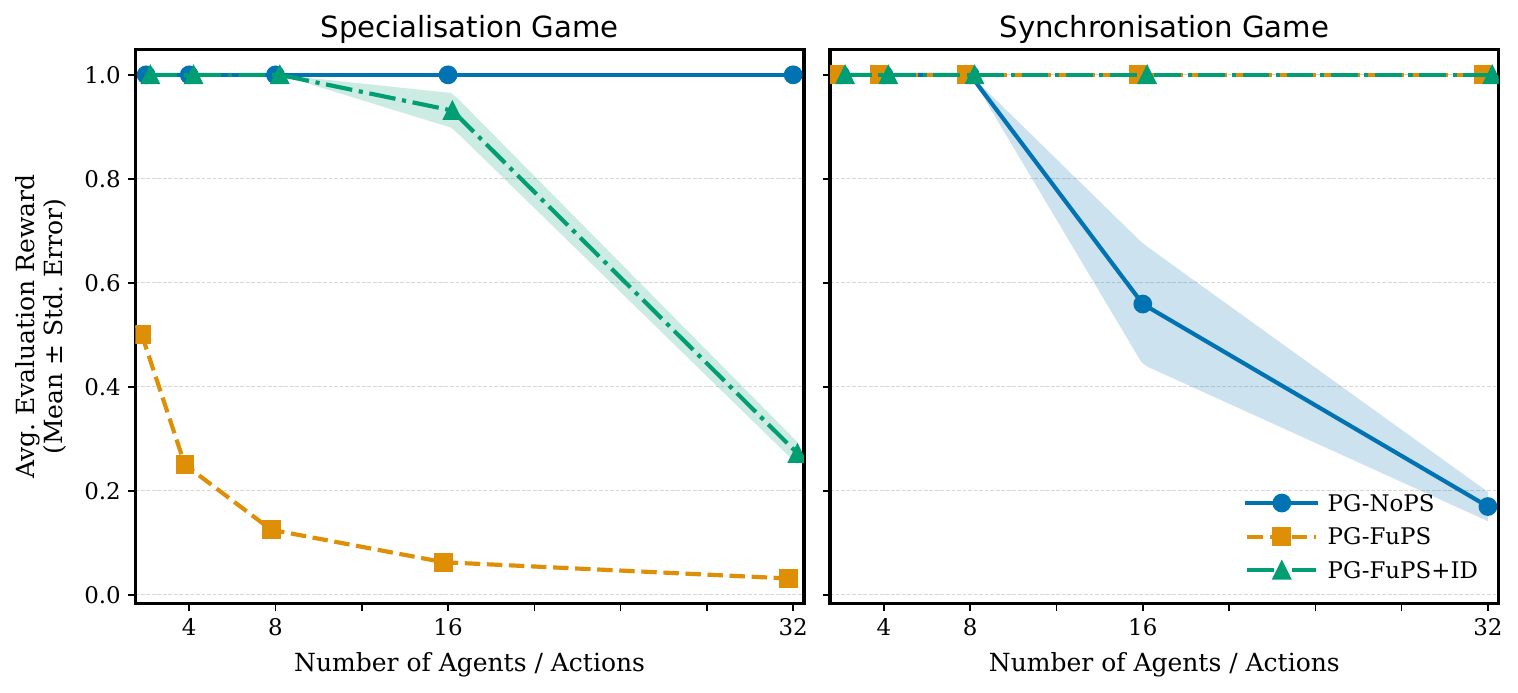}
    \caption{
        Mean evaluation reward (\textbf{mean ± standard error}) as a function of the number of agents/actions in the \textbf{Specialisation} (left) and \textbf{Synchronisation} (right) games. In the Specialisation game, vanilla policy gradients (PG, i.e. REINFORCE) with FuPS collapse as the team grows, whereas our identity-aware variant (PG-FuPS+ID) retains near-optimal performance. In the Synchronisation game, PG-NoPS performs well at small scales but degrades with more agents, while both PG-FuPS and PG-FuPS+ID remain at optimal reward across all scales.
    }
    \label{fig:specialisation-synchronisation-normal-form-results}
\end{figure}

\clearpage
\section{HyperMARL Details}
\subsection{HyperMARL Pseudocode}\label{subsec:pseudocode}
\definecolor{HyperMarlBlue}{RGB}{70, 130, 180}

In Algorithm \ref{alg:hypermarl}, we present the pseudocode for HyperMARL, with HyperMARL-specific steps highlighted in \textcolor{HyperMarlBlue}{blue}. HyperMARL leverages hypernetworks to dynamically generate the parameters of both actor and critic networks. The weights of the hypernetworks and the agent embeddings are automatically updated through automatic differentiation (autograd) based on the computed loss. Additionally, Figure \ref{fig:hypermarl} provides a visual representation of the HyperMARL architecture.
\begin{algorithm}[t]
   \caption{HyperMARL}
   \label{alg:hypermarl}
\begin{algorithmic}[1]
   \STATE {\bfseries Input:} Number of agents $n$, number of training iterations $K$, MARL algorithm parameters (e.g., MAPPO-specific hyperparameters)
   \STATE {\bfseries Initialise:}
   \STATE \quad \textcolor{HyperMarlBlue}{Hypernetwork parameters $\psi, \varphi$} \COMMENT{Ensure $\theta^i$ and $\phi^i$ follow standard initialization schemes, e.g., orthogonal}
   \STATE \quad \textcolor{HyperMarlBlue}{Agent embeddings $\{e^i\}_{i=1}^n$} \COMMENT{One-hot or orthogonally initialised learnable parameters}
   \STATE {\bfseries Output:} Optimized joint policy $\boldsymbol{\pi}$
   
   \FOR{each training iteration $k = 0, 1, \dots, K-1$}
       \FOR{each agent $i = 1, \dots, n$} 
       \STATE \textcolor{HyperMarlBlue}{$\theta^i \gets h_{\psi}^{\pi}(e^i)$} \COMMENT{Policy parameters}
       \STATE \textcolor{HyperMarlBlue}{$\phi^i \gets h_{\varphi}^{V}(e^i)$} \COMMENT{Critic parameters}
       \ENDFOR
       \STATE Interact with environment using $\{\pi_{\textcolor{HyperMarlBlue}{\theta^i}}\}_{i=1}^n$ to collect trajectories $\mathcal{D}$
       \STATE Compute shared loss $\mathcal{L}$ from $\mathcal{D}$, using $\{V_{\textcolor{HyperMarlBlue}{\phi^i}}\}_{i=1}^n$ \COMMENT{Standard RL loss function}
       \STATE \textcolor{HyperMarlBlue}{Update $\psi$, $\varphi$, and $\boldsymbol{e}$ by minimizing $\mathcal{L}$} \COMMENT{Optimise parameters.}
   \ENDFOR
   \STATE {\bfseries Return} $\boldsymbol{\pi} = (\pi^{1}, \dots, \pi^{n})$
\end{algorithmic}
\end{algorithm}

\subsection{Variance of the HyperMARL Gradient Estimator}
\label{app:variance_derivation}

\paragraph{Unbiased estimator.}
Following from Equation~\ref{eq:hmarl_split}, we can write the unbiased estimator for HyperMARLs gradient as follows:
\[
  \widehat g_{\mathrm{HM}}
  =
  \sum_{i=1}^{I}
    \underbrace{\nabla_{\!\psi}h_{\psi}^{\pi}(e^{i})}_{\mathbf J_i}
    \;\cdot\;
    \underbrace{\Bigl(
      \frac{1}{B}\sum_{b=1}^{B}\sum_{t=0}^{T-1}
        A(\mathbf h_t^{(b)},\mathbf a_t^{(b)})\,
        \nabla_{\theta^{i}}\log\pi_{\theta^{i}}(a_t^{i,(b)}\!\mid h_t^{i,(b)})
    \Bigr)}_{\widehat Z_i}
  =\sum_{i=1}^{I}\mathbf J_i\,\widehat Z_i,
  \tag{HM'}
\]
where \(B\) trajectories \(\{\tau^{(b)}\}_{b=1}^B\) are sampled i.i.d.\ and
\(\widehat Z_i\) is the empirical analogue of the observation-conditioned factor.

\paragraph{Assumptions.} (A1) trajectories are i.i.d.; (A2) all second
moments are finite; (A3) \(\psi,\theta,e^i\) are fixed during the
backward pass.

\paragraph{Variance expansion.}

Since each \(\mathbf J_i\) is deterministic under (A3), we may factor them outside the covariance:

\begin{align}
\Var\bigl(\widehat g_{\text{\tiny HM}}\bigr)
&=\Cov\!\Bigl(\sum_{i}\mathbf J_i\widehat Z_i,\;\sum_{j}\mathbf J_j\widehat Z_j\Bigr)
  \tag*{(by def.\ \(\Var(X)=\Cov(X,X)\))}\\
&=\sum_{i,j}\Cov\bigl(\mathbf J_i\widehat Z_i,\;\mathbf J_j\widehat Z_j\bigr)
  \tag*{(bilinearity of \(\Cov\))}\\
&=\sum_{i,j}\mathbf J_i\;\Cov(\widehat Z_i,\widehat Z_j)\;\mathbf J_j^{\!\top}
  \tag*{(pull deterministic matrices out of \(\Cov\))}\\
\label{eq:variance_split_detail}
\end{align}

Equation~\eqref{eq:variance_split_detail} makes explicit that all trajectory‐induced covariance is captured \(\Cov(\widehat Z_i,\widehat Z_j)\), while the agent-conditioned Jacobians \(\mathbf J_i\) remain trajectory noise-free.  

\textbf{Mini-batch update and covariance.}~\label{app:gradient_details} Let $\widehat Z_i$ be the unbiased mini-batch estimate of $Z_i$
and $\widehat g_{\text{\tiny HM}}=\sum_i\mathbf J_i\widehat Z_i$ the
stochastic update.  Because every $\mathbf J_i$ is deterministic (wrt. to mini-batch),

\begin{equation}
\label{eq:variance_split}
\Var\!\bigl(\widehat g_{\text{\tiny HM}}\bigr)
   =\sum_{i,j}\mathbf J_i\;
     \Cov(\widehat Z_i,\widehat Z_j)\;
     \mathbf J_j^{\!\top},
\end{equation}
(derivation in Appendix~\ref{app:variance_derivation}). Equation~\eqref{eq:variance_split} shows that HyperMARL first averages noise \emph{within each agent} (\(\widehat Z_i\)) and only then applies \(\mathbf J_i\).  FuPS+ID, by contrast, updates the shared weights \(\theta\) with every raw sample \(A\,\nabla_\theta\log\pi_\theta[h,\mathrm{id}]\), leaving observation noise and agent ID entangled and making it susceptible to gradient interference \citep{christianos2021scaling,JMLR:v25:23-0488}.

\subsection{Scalability and Parameter Efficiency} \label{sec:scaling}
Hypernetworks generate weights for the target network, which can lead to high-dimensional outputs and many parameters for deep target networks. This challenge is amplified in MLP-based hypernetworks, which include additional hidden layers. Figure \ref{fig:scale} shows scaling trends:
\begin{itemize}
    \item \textbf{NoPS} and \textbf{linear hypernetworks}: Parameter count grows linearly with the number of agents.
    \item \textbf{FuPS}: More efficient, as growth depends on one-hot vector size. 
    \item \textbf{MLP hypernetworks}: Scale better with larger populations, since they only require embeddings of fixed size for each new agent.
\end{itemize}

To reduce parameter count, techniques like shared hypernetworks, chunked hypernetworks \citep{Oswald2020Continual,pmlr-v238-chauhan24a}, or producing low-rank weight approximations, can be used. While naive implementations are parameter-intensive, this might be less critical in RL and MARL which commonly have smaller actor-critic networks. Moreover, HyperMARL’s near-constant scaling with agents suggests strong potential for large-scale MARL applications.

\begin{figure}[h]
    \centering
    \includegraphics[width=0.6\linewidth]{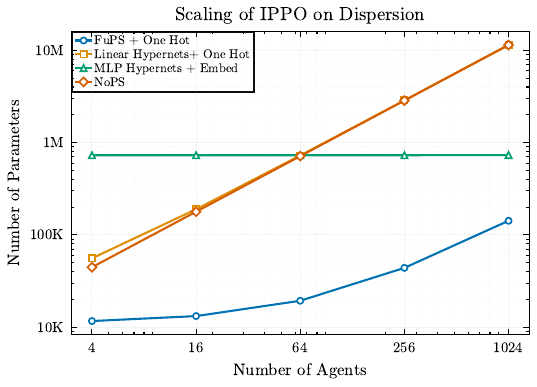}
     \caption{Parameter scaling for IPPO variants with increasing agents (4 to 1024). MLP Hypernetworks scale nearly constantly, while NoPS, Linear Hypernetworks, and FuPS+One-Hot grow linearly. Log scale on both axes.}
    \label{fig:scale}
\end{figure}

\begin{figure}[h]
    \centering
    \includegraphics[width=0.65\linewidth]{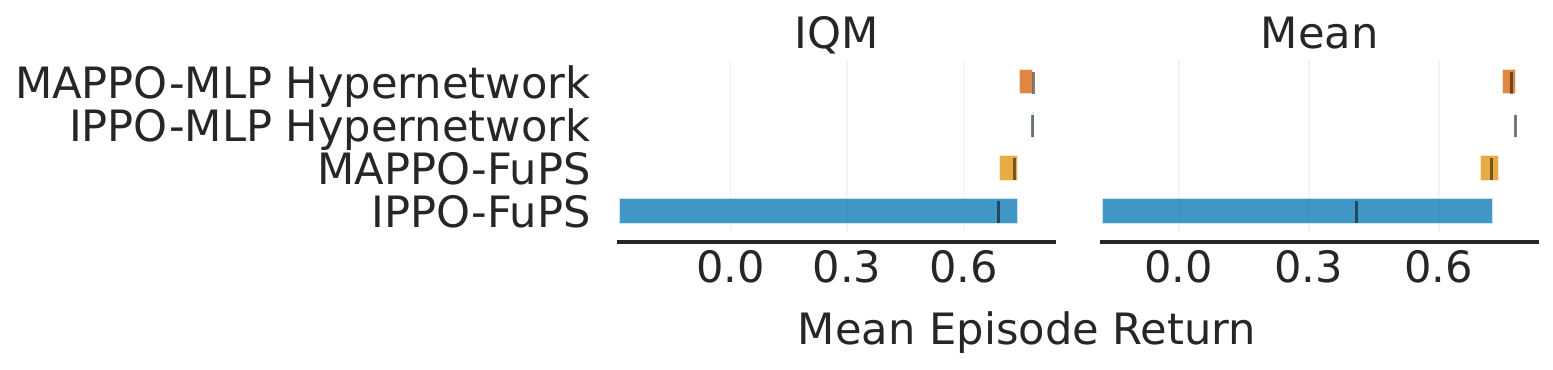}
    \caption{Dispersion performance with four agents. FuPS variants match HyperMARL in parameter count but still underperform.} 
    \label{fig:control_params}
\end{figure}

To isolate the effects of parameter count, we scaled the FuPS networks (Figure \ref{fig:control_params}) to match the number of trainable parameters in HyperMARL.  Despite generating 10x smaller networks, HyperMARL consistently outperforms FuPS variants, showing its advantages extend beyond parameter count.

\subsection{Speed and Memory Usage}\label{sec:runtime}

We examine the computational efficiency of HyperMARL compared to NoPS and FuPS by measuring inference speed (Figure~\ref{fig:forward_speed}) and GPU memory usage (Figure~\ref{fig:gpu_memory}) as we scale the number of agents. The benchmarks were conducted using JAX on a single NVIDIA GPU (T4) with a recurrent (GRU-based) policy architecture. All experiments used fixed network sizes (64-dimensional embeddings and hidden layers) with a batch size of 128 and 64 parallel environments, allowing us to isolate the effects of varying agent count. Each measurement represents the average of 100 forward passes per configuration, with operations repeated across 10 independent trials.

\begin{figure}[h]
    \centering
    \begin{subfigure}[b]{0.48\textwidth}
        \centering
        \includegraphics[width=\textwidth]{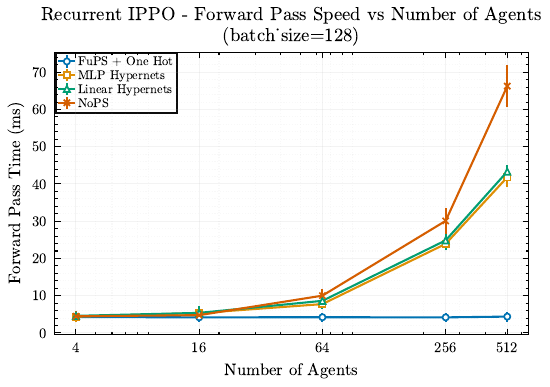}
        \caption{Forward pass inference time}
        \label{fig:forward_speed}
    \end{subfigure}
    \hfill
    \begin{subfigure}[b]{0.48\textwidth}
        \centering
        \includegraphics[width=\textwidth]{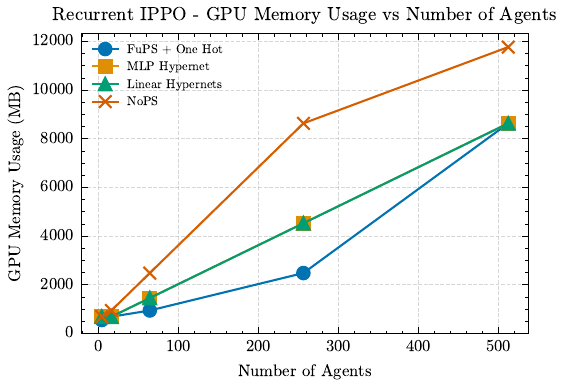}
        \caption{GPU memory usage}
        \label{fig:gpu_memory}
    \end{subfigure}
    \caption{Computational efficiency scaling with number of agents. HyperMARL offers a balance between NoPS and FuPS. Notably, in real-world deployments, NoPS incurs additional data transfer and synchronisation costs not reflected here, further widening the efficiency gap.}
    \label{fig:computational_scaling}
\end{figure}

The results demonstrate that HyperMARL offers a balance between the extremes represented by NoPS and FuPS. In practice, NoPS incurs additional data transfer and update costs, widening the efficiency gap.

\subsection{Sensitivity Analysis}\label{app:sensitivity}

We run a sensitivity analysis on the HyperMARL's hyperparameters in the 20-agent SMAX scenario with recurrent MAPPO, a challenging setting for optimisation. The results are shown Table~\ref{tab:sensitivity_analysis}.

\paragraph{Findings.}
\begin{itemize}
    \item Similarly to FuPS, and in general, deep RL methods like PPO, the learning rate is an important parameter for performance.
    \item \textbf{Agent embedding size is an important hyperparameter.} In these results, a smaller embedding size outperformed a larger one. This could be due to the homogeneous nature of some SMAX tasks, where with smaller embeddings, it could be easier to learn similar agent embeddings and hence similar behaviours. This suggests the embedding size could correspond to the task's diversity requirements.
    \item Other parameters, such as the width, have limited impact beyond a modest size.
\end{itemize}

\begin{table*}[h]
\centering
\caption{IQM and 95\% CI of mean win rate across 5 seeds for the 20-agent SMAX scenario with recurrent MAPPO. \textbf{Bold} indicates the highest IQM score; * indicates scores whose 95\% confidence intervals overlap with the highest score.}
\label{tab:sensitivity_analysis}
\resizebox{\textwidth}{!}{%
\begin{tabular}{@{}llrccc@{}}
\toprule
\textbf{Method} & \textbf{Embedding Dim} & \textbf{Hidden Dims} & \textbf{LR 0.0001} & \textbf{LR 0.0003} & \textbf{LR 0.0005} \\ \midrule
FuPS & N/A & N/A & 0.1067 (0.0513, 0.2128) & 0.3053 (0.2041, 0.4769) & *0.4213 (0.3623, 0.4800) \\ \midrule
\multirow{4}{*}{HyperMARL} & \multirow{4}{*}{4} & 16 & 0.0799 (0.0250, 0.1190) & 0.2946 (0.1143, 0.3778) & \textbf{0.4455} (0.3443, 0.5692) \\
 & & 64 & 0.1233 (0.0750, 0.2162) & 0.3338 (0.2692, 0.4286) & *0.3765 (0.3281, 0.4107) \\
 & & 128 & 0.1147 (0.0270, 0.1842) & *0.3787 (0.3509, 0.4603) & 0.3327 (0.2931, 0.3548) \\
 & & 64, 64 & 0.1781 (0.1628, 0.2075) & *0.3498 (0.2807, 0.4500) & 0.1944 (0.1268, 0.3077) \\ \midrule
\multirow{4}{*}{HyperMARL} & \multirow{4}{*}{64} & 16 & 0.1109 (0.0256, 0.2093) & 0.3198 (0.2353, 0.3934) & *0.3483 (0.2581, 0.4308) \\
 & & 64 & 0.1360 (0.0303, 0.2821) & 0.1191 (0.0517, 0.1964) & 0.1155 (0.0794, 0.1500) \\
 & & 128 & 0.1193 (0.0750, 0.1579) & 0.1036 (0.0517, 0.1594) & 0.1106 (0.0455, 0.1587) \\
 & & 64, 64 & 0.1353 (0.0541, 0.1818) & 0.1283 (0.0526, 0.2258) & 0.0860 (0.0345, 0.1343) \\ \bottomrule
\end{tabular}%
}
\end{table*}
\clearpage

\section{Experiment Details}
\subsection{Environments}\label{subsec:environments}
\begin{figure}[h!]
    \centering
    \begin{subfigure}[b]{0.22\textwidth}
        \centering
        \includegraphics[width=0.9\textwidth]{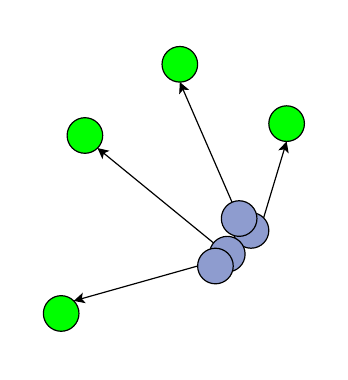}
        \caption{\textit{Dispersion}.}
    \end{subfigure}
    \hfill
    \begin{subfigure}[b]{0.22\textwidth}
        \centering
        \includegraphics[width=\textwidth]{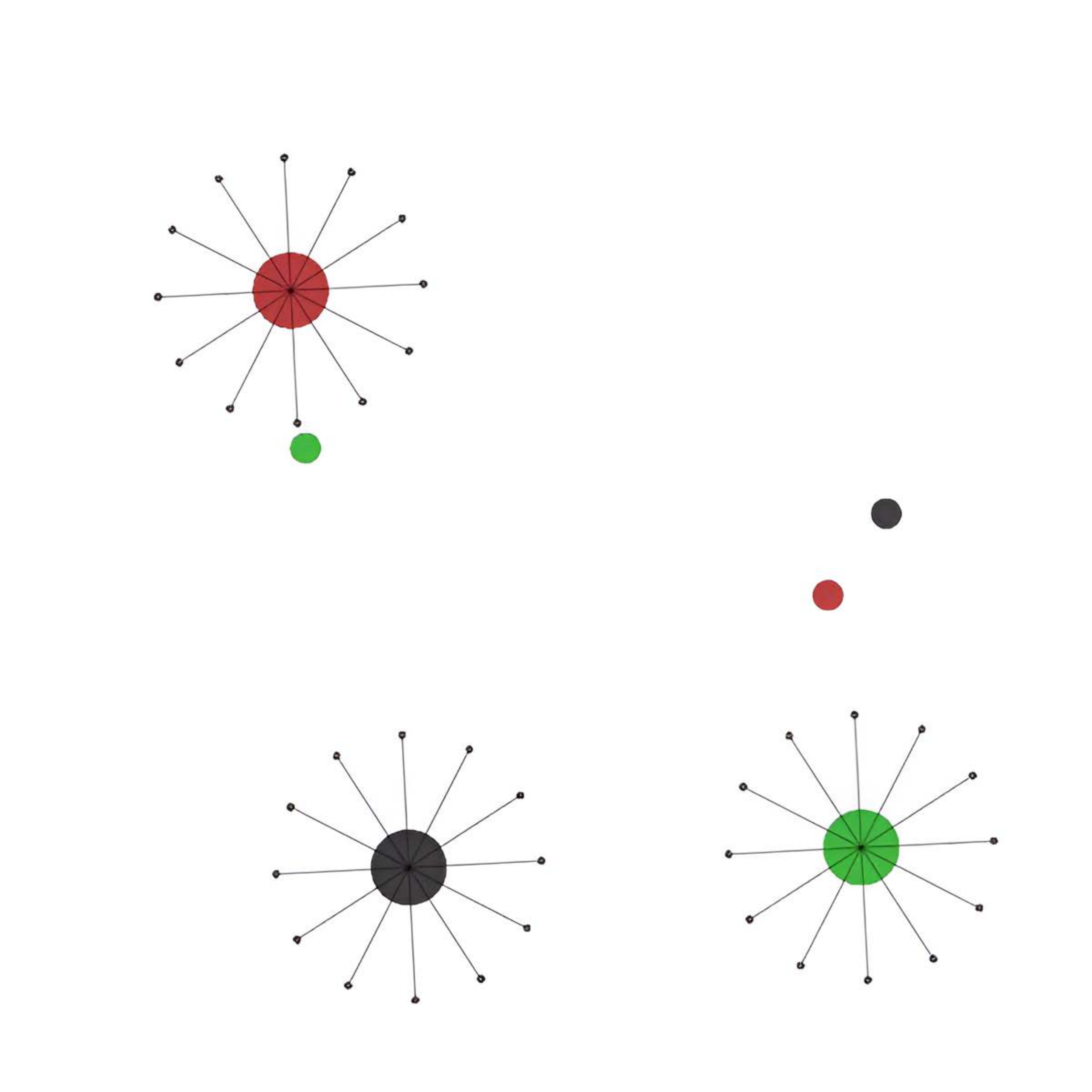}
        \caption{\textit{Navigation}.}
    \end{subfigure}
    \hfill
    \begin{subfigure}[b]{0.22\textwidth}
        \centering
        \includegraphics[width=\textwidth]{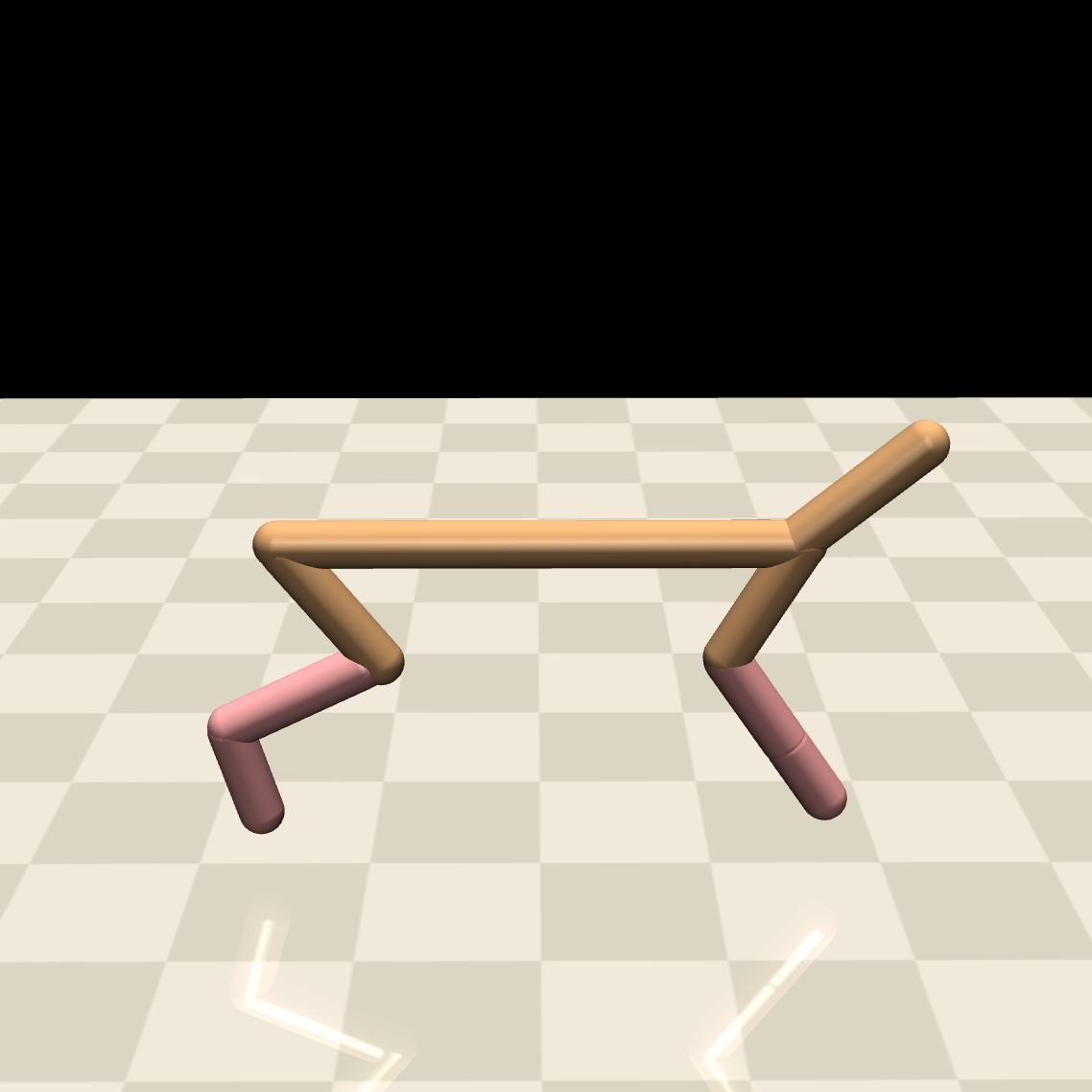}
        \caption{\textit{MAMuJoCo}.}
    \end{subfigure}
    \hfill
    \begin{subfigure}[b]{0.22\textwidth}
        \centering
        \includegraphics[width=\textwidth]{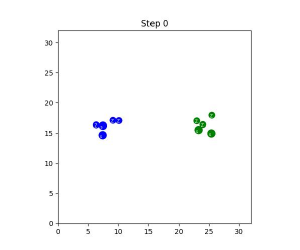}
        \caption{\textit{SMAX}.}
    \end{subfigure}
    \hfill
    \begin{subfigure}[b]{0.22\textwidth}
        \centering
        \includegraphics[width=\textwidth]{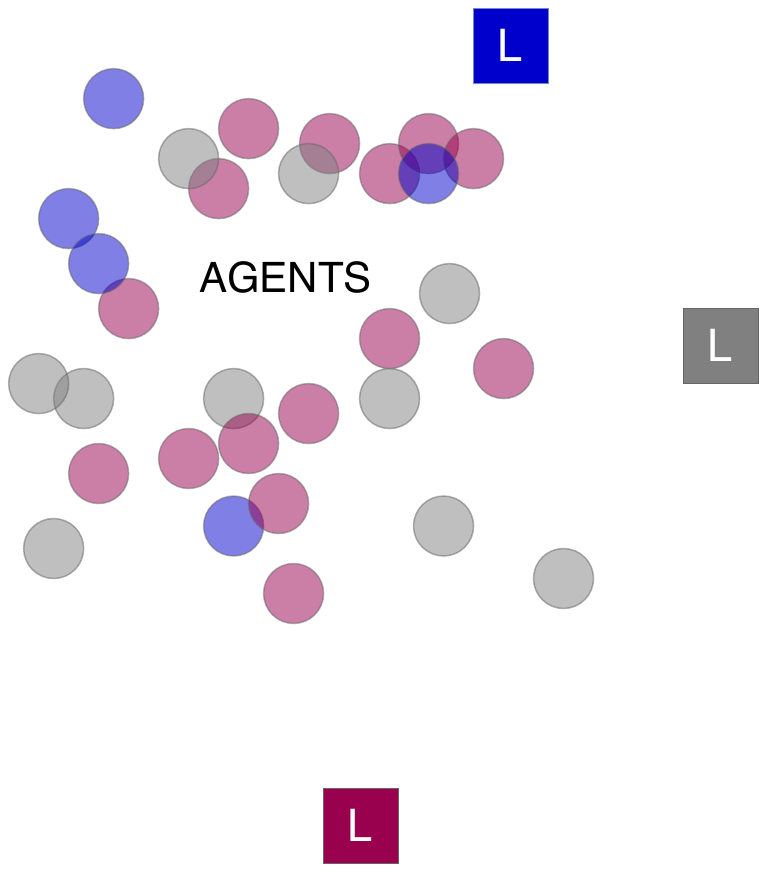}
        \caption{\textit{BPS}.}
    \end{subfigure}
    \caption{Multi-Agent environments used in our experiments.}
    \label{fig:marl_envs}
\end{figure}

\textbf{Dispersion (VMAS)~\citep{bettini2022vmas}}: A 2D environment where four agents collect unique food particles.  This task requires specialised \textit{heterogeneous} behaviours and resembles the Specialisation Game from Section \ref{sec:spec_game}. 

\textbf{Navigation (VMAS)~\citep{bettini2022vmas}}: Agents navigate in a 2D space to assigned goals, receiving dense rewards based on proximity. Configurations include shared goals (\textit{homogeneous}), unique goals (\textit{heterogeneous}), and \textit{mixed} goals, where some agents share goals while others have unique ones.

\textbf{Multi-Agent MuJoCo (MAMuJoCo)~\citep{peng2021facmac}}: A multi-agent extension of MuJoCo, where robot body parts (e.g., a cheetah's legs) are modelled as different agents. Agents coordinate to perform efficient motion, requiring \textit{heterogeneous} policies~\citep{JMLR:v25:23-0488}.

\textbf{SMAX (JaxMARL)~\citep{rutherford2024jaxmarl}}: Discrete action tasks with 2 to 20 agents on SMACv1- and SMACv2-style maps. FuPS baselines have been shown optimal for these settings~\cite {yu2022surprising,fu2022revisiting} indicating \textit{homogeneous} behaviour is preferred here. 

\textbf{Blind-Particle Spread (BPS)~\citep{christianos2021scaling}} Blind-Particle Spread (BPS) is a modified variant of the Multi-Agent Particle Environment (MPE) Simple Spread task~\citep{lowe2017multi} with 15 to 30 agents. The environment contains landmarks of different colours, and each agent is assigned a colour. Agents are 'blind' as they cannot observe their own assigned colour or the colours of other agents. They must infer the correct landmark to navigate towards based only on their observations and rewards. The number of colours thus dictates the number of distinct roles or behaviours the team must learn to successfully complete the task.

\newcommand{\xmark}{\ding{55}}%

\begin{table}[t]
    \centering
    \caption{Baseline Methods Selection and Justification. Selected methods (\checkmark) were chosen based on their relevance to parameter sharing and specialisation/generalisation in MARL, while excluded methods (\xmark) did not align with our research objectives or had implementation constraints. \textbf{Our experimental design systematically addresses key questions on agent specialisation and homogeneity, therefore we selected baselines with demonstrated strong performance in their respective settings, ensuring fair and rigorous comparison.}}
     \label{tab:baseline_selection}
    \small
    \begin{tabular}{p{2.5cm}p{1.2cm}cp{8cm}} %
        \toprule
        \textbf{Method} & \textbf{Category} & \textbf{Selected} & \textbf{Justification \& Experimental Settings} \\
        \midrule
        \textbf{IPPO}~\citep{de2020independent}\\(NoPS, FuPS+ID) & NoPS/FuPS & \checkmark & Established MARL baseline implementing both independent (NoPS) and fully shared (FuPS+ID) policy configurations. \textit{Tested in:} Dispersion, Navigation, SMAX (two SMACv1 maps and two SMACv2 maps, with 10 and 20 agents). \\
        \addlinespace[0.5ex]
        
        \textbf{MAPPO}~\citep{yu2022surprising}\\(NoPS, FuPS+ID) & NoPS/FuPS & \checkmark & Strong baseline with centralized critics for both NoPS and FuPS+ID architectures. \textit{Tested in:} Dispersion, MAMuJoCo, SMAX (two SMACv1 maps and two SMACv2 maps, with 10 and 20 agents). \\
        \addlinespace[0.5ex]
        
        \textbf{DiCo}~\citep{bettini2024controlling} & Architectural Diversity & \checkmark & Provides comparison with a method employing preset diversity levels that balances shared and non-shared parameters. \textit{Tested in:} Dispersion and Navigation (as per original paper). Original hyperparameters used for $n=2$ agents; parameter sweeps conducted for $n>2$ to identify optimal diversity levels. \\
        \addlinespace[0.5ex]
        
        \textbf{HAPPO}~\citep{JMLR:v25:23-0488} & Sequential Updates & \checkmark & Enables comparison with a method designed for heterogeneous behaviours using sequential policy updates with agent-specific parameters. \textit{Tested in:} MAMuJoCo, selecting 4/6 scenarios from the original paper, including the challenging 17-agent humanoid task. Walker and Hopper variants were excluded as MAPPO and HAPPO performed similarly in these environments. \\
        \addlinespace[0.5ex]
        
        \textbf{Kaleidoscope}~\citep{li2024kaleidoscope} & Architectural Pruning & \checkmark & Implemented for off-policy evaluation using its MATD3 implementation with tuned MaMuJoCo hyperparameters. \textit{Tested in:} MAMuJoCo environments Ant-v2, HalfCheetah-v2, Walker2d-v2 (overlapping with our IPPO experiments), and Swimmer-v2-10x2 (highest agent count variant). Included to evaluate HyperMARL's competitiveness against a method with ensemble critics and diversity loss, in an off-policy setting. \\
        \addlinespace[0.5ex]

        \textbf{SePS}~\citep{christianos2021scaling} & Selective Parameter Sharing & \checkmark & Although this requires a pre-training phase, it is a strong baseline for parameter sharing approaches. \textit{Tested in:} Blind-Particle Spread environments, with 15 to 30 agents.  \\
        \addlinespace[0.5ex]
        
        \textbf{SEAC}~\citep{christianos2020shared} & Shared Experience & \xmark & Focuses primarily on experience sharing rather than parameter sharing architecture, falling outside our research scope. \\
        \addlinespace[0.5ex]
        
        \textbf{CDAS}~\citep{li2021celebrating} & Intrinsic Reward & \xmark & Only implemented for off-policy methods and has been demonstrated to underperform FuPS/NoPS architectures~\citep{fu2022revisiting}, making it less suitable for our primary on-policy comparisons. \\
        \addlinespace[0.5ex]
        
        \textbf{ROMA/RODE}~\citep{wang2020roma,wang2020rode} & Role-based & \xmark & Shows limited practical performance advantages in comparative studies~\citep{christianos2020shared}, suggesting other baselines provide more rigorous comparison points. \\
        \addlinespace[0.5ex]
        
        \textbf{SNP-PS}~\citep{kim2023parameter} & Architectural Pruning & \xmark & No publicly available implementation, preventing direct, reproducible comparison. \\
        \bottomrule
    \end{tabular}
\end{table}

\subsection{HyperMARL Architecture Details}\label{app:arch_details}
For the Dispersion and Navigation results (Sec.~\ref{results:dispersion}) we use feedforward architectures, where we use HyperMARL to generate both the actor and critic networks. For the MAPPO experiments in Section \ref{sec:mamujoco_results}, for fairness in comparisons with HAPPO and MAPPO, we maintain the centralised critic conditioned on the global state and only use HyperMARL to generate the weights of the actors. For the recurrent IPPO experiments in Section \ref{results:smax}, HyperMARL only generates the actor and critic feedforward weights, not the GRU weights.   

\subsubsection{Training and Evaluation}\label{subsec:train_and_eval} 
\begin{itemize}
    \item \textbf{Training:}
    \begin{itemize}
        \item For Dispersion (\ref{results:dispersion}), we run 10 seeds and train for 20 million timesteps.
        \item For Navigation (\ref{results:navigation}), SMAX (\ref{results:smax}), and MaMuJoCo (\ref{sec:mamujoco_results}), we run 5 seeds and train for 10 million timesteps, consistent with the baselines.
        \item For Blind-Particle Spread (BPS), we run 5 seeds and train for 20 million timesteps, consistent with baselines.
    \end{itemize}
    \item \textbf{Evaluation:}
    \begin{itemize}
        \item For Dispersion (\ref{results:dispersion}), evaluation is performed every 100k timesteps across 32 episodes.
        \item For Navigation (\ref{results:navigation}), following the baselines, evaluation is performed every 120k timesteps across 200 episodes.
        \item For SMAX (\ref{results:smax}), evaluation is performed every 500k timesteps across 32 episodes.
        \item For MaMuJoCo (\ref{sec:mamujoco_results}), following the baselines, evaluation is performed every 25 training episodes over 40 episodes.
    \end{itemize}
\end{itemize}

\subsubsection{Measuring Policy Diversity Details} \label{app:measure_policy_diversity}
We measure team diversity using the System Neural Diversity (SND) metric (Equation~\ref{eq:snd}~\cite{bettini2023system}, details Section \ref{app:beh_diversity}) with Jensen-Shannon distance. SND ranges from 0 (identical policies across all agents) to 1 (maximum diversity). We collect a dataset of observations from IPPO-NoPS and IPPO-FuPS policies checkpointed at 5 and 20 million training steps. Each policy is rolled out for 10,000 episodes, generating 16 million observations. We then sample 1 million observations from this dataset to calculate the SND for each method tested.

\clearpage

\section{Detailed Results} \label{append:detailed_results}

\subsection{Comparison with Kaleidoscope using Off-Policy -- MATD3}
\label{app:kaleidoscope_offpolicy}

Our comparison with Kaleidoscope~\citep{li2024kaleidoscope}, mentioned in Section~\ref{sec:experiments}, is conducted using off-policy methods due to its original design. Kaleidoscope incorporates intricate mechanisms (e.g., learnable masks, an ensemble of five critics, a specific diversity loss) and numerous specialised hyperparameters (e.g., for critic ensembling: `critic\_deque\_len`, `critic\_div\_coef`, `reset\_interval`; for mask and threshold parameters: `n\_masks`, `threshold\_init\_scale`, `threshold\_init\_bias`, `weighted\_masks`, `sparsity\_layer\_weights`, etc.). Porting this full complexity to an on-policy PPO backbone would constitute a significant research deviation rather than a direct benchmark of the established method.

Therefore, we utilised Kaleidoscope's original off-policy implementation to ensure a meaningful comparison. We adopted MATD3 as the algorithmic backbone for this evaluation, as it was the only publicly available Kaleidoscope variant with tuned hyperparameters for Multi-Agent MuJoCo (MaMuJoCo). The MaMuJoCo tasks were chosen for alignment with our primary on-policy (IPPO) results and Kaleidoscope's original evaluation, specifically: Ant-v2, HalfCheetah-v2, Walker2d-v2 (overlapping with our IPPO experiments), and Swimmer-v2-10x2 ( which represents the MaMuJoCo variant with the highest number of agents). Comparative results in Table~\ref{tab:iqm-mujoco-offpolicy} show that HyperMARL achieves competitive results with Kaleidoscope, while only using two critics (standard MATD3) and without additional diversity objectives.

\begin{table}[t] 
\small
\centering
\caption{\textit{Mean episode return in MAMuJoCo for off-policy MATD3 variants.} IQM of the mean episode returns with 95\% stratified bootstrap CI. \textbf{Bold} indicates the highest IQM score; * indicates scores whose confidence intervals overlap with the highest. Although Kaleidoscope employs an ensemble of five critics and an explicit diversity loss, HyperMARL (using a standard MATD3 setup with two critics) achieves competitive results without these additional mechanisms.}
\label{tab:iqm-mujoco-offpolicy}
\resizebox{0.92\textwidth}{!}{%
\begin{tabular}{lccc}
\toprule
\textbf{Environment}      & \textbf{Ind. Actors, Shared Critic}      & \textbf{Kaleidoscope}         & \textbf{HyperMARL}             \\
\midrule
Ant-v2                  & 5270.38 (4329.73, 5719.78)      & \textbf{6160.70} (5798.02, 6463.83)    & *5886.58 (5840.00, 5920.66)         \\
HalfCheetah-v2          & *6777.04 (3169.11, 8233.94)      & *6901.00 (3609.73, 8192.38)          & \textbf{7057.44} (3508.70, 8818.11) \\
Walker2d-v2             & *5771.87 (5144.84, 8103.34)      & *6664.32 (5408.95, 8828.11)          & \textbf{7057.68} (5976.50, 8166.09) \\
Swimmer-v2-10x2         & *453.74 (427.24, 487.86)         & *462.48 (444.22, 475.64)             & \textbf{465.91} (410.82, 475.77)    \\
\bottomrule
\end{tabular}
}
\end{table}

\subsection{Comparison with Selective Parameter Sharing (SePS) - A2C}
\label{app:seps}

We benchmark HyperMARL against Selective Parameter Sharing (SePS)~\citep{christianos2021scaling}, a method that shares weights among pre-determined groups of similar agents. SePS identifies these groups by clustering agent trajectories via an autoencoder during a pre-training phase. Following the original authors' setup, we use their provided source code and hyperparameters, evaluating all A2C~\citep{mnih2016asynchronous} variants on the Blind Particle Spread (BPS) benchmarks (v1-4), which scale up to 30 agents. For HyperMARL, we sweep the learning rates (1e-4, 3e-4) and hypernetwork hidden dimension (16, 64) per scenario.

As shown in Table~\ref{tab:bps_results}, HyperMARL demonstrates competitive performance against both NoPS and SePS in these challenging settings, which feature up to 30 agents and five distinct agent roles. Crucially, HyperMARL achieves this result without resorting to agent-specific weights (like NoPS) or requiring a pre-training and clustering pipeline (like SePS).

\begin{table}[h]
\centering
\caption{\textit{Mean episode reward in Blind Particle Spread for A2C variants.} We report the IQM of the final total reward (with 95\% confidence intervals) across 5 seeds on the BPS environments. \textbf{Bold} indicates the highest IQM score; * indicates scores whose 95\% confidence intervals overlap with the highest score. HyperMARL achieves results competitive with SePS without requiring a pre-training phase or explicit agent clustering.}
\label{tab:bps_results}
\resizebox{\textwidth}{!}{%
\begin{tabular}{lcccc}
\toprule
\textbf{Environment} & \textbf{NoPS} & \textbf{FuPS+ID} & \textbf{SePS} & \textbf{HyperMARL} \\
\midrule
BPS-1 (15 agents, 3 groups, 5–5–5) & -216.8 (-227.7, -197.8) & -228.2 (-247.5, -213.3) & *-201.8 (-221.4, -186.8) & \textbf{-190.8 (-220.0, -178.5)} \\
BPS-2 (30 agents, 3 groups, 10–10–10) & *-415.4 (-459.4, -366.3) & -429.7 (-507.1, -411.7) & *-407.1 (-453.7, -387.6) & \textbf{-397.8 (-423.7, -376.6)} \\
BPS-3 (30 agents, 5 groups, 6–6–6–6–6) & \textbf{-403.4 (-421.0, -381.0)} & -835.2 (-1445.6, -534.1) & *-422.1 (-466.2, -387.6) & *-417.7 (-448.2, -381.1) \\
BPS-4 (30 agents, 5 uneven groups, 2–2–2–15–9) & *-410.8 (-436.6, -368.3) & -780.5 (-1044.1, -593.8) & *-411.6 (-446.6, -346.7) & \textbf{-389.5 (-457.6, -366.0)} \\
\bottomrule
\end{tabular}
}
\end{table}

\subsection{Dispersion Detailed Results}

\subsubsection{Interval Estimates Dispersion}\label{append:dispersion_dynamics}
\begin{figure*}[ht]
    \centering

    \begin{subfigure}[b]{0.38\linewidth}
        \centering
        \includegraphics[width=\linewidth]{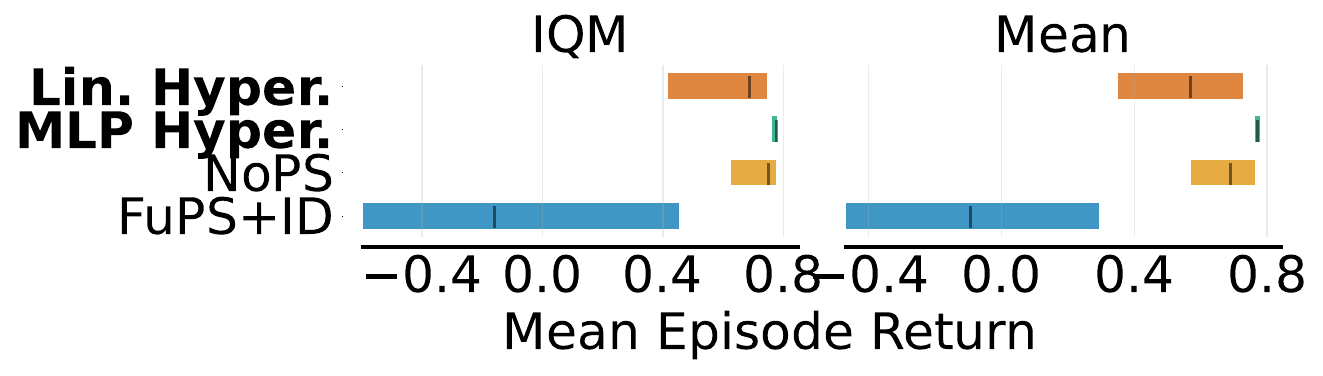}
        \caption{IPPO}
        \label{fig:ippo-interval}
    \end{subfigure}
    \hspace{0.05\linewidth} %
    \begin{subfigure}[b]{0.38\linewidth}
        \centering
        \includegraphics[width=\linewidth]{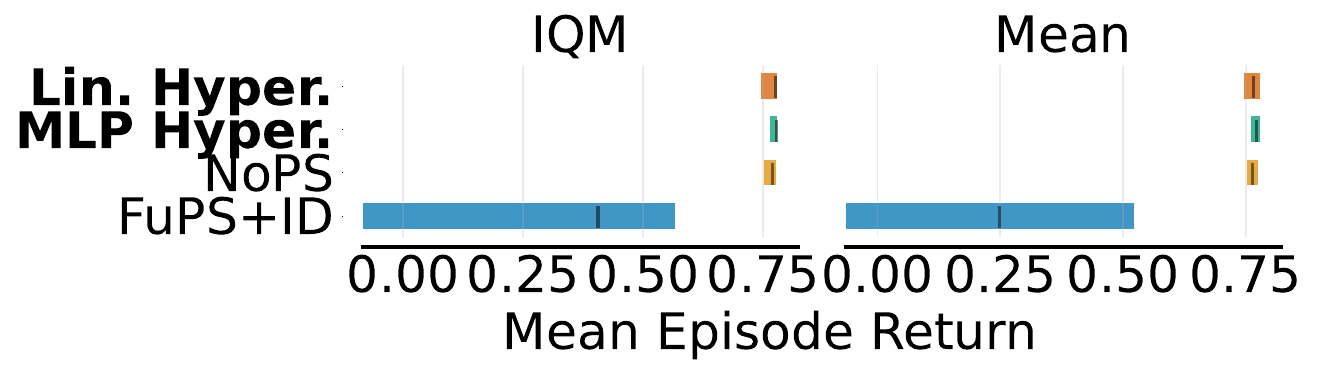}
        \caption{MAPPO}
        \label{fig:mappo-interval}
    \end{subfigure}

    \caption{\textit{Performance of IPPO and MAPPO on Dispersion after 20 million timesteps.} We show the Interquartile Mean (IQM) of the Mean Episode Return and the 95\% Stratified Bootstrap Confidence Intervals (CI) using~\cite{agarwal2021deep}. Hypernetworks achieve comparable performance to NoPS, while FuPS struggle with specialisation.}
    \label{fig:performance-comparison-interval}
\end{figure*}

\begin{figure}[h]
    \centering
    \includegraphics[width=0.45\linewidth]{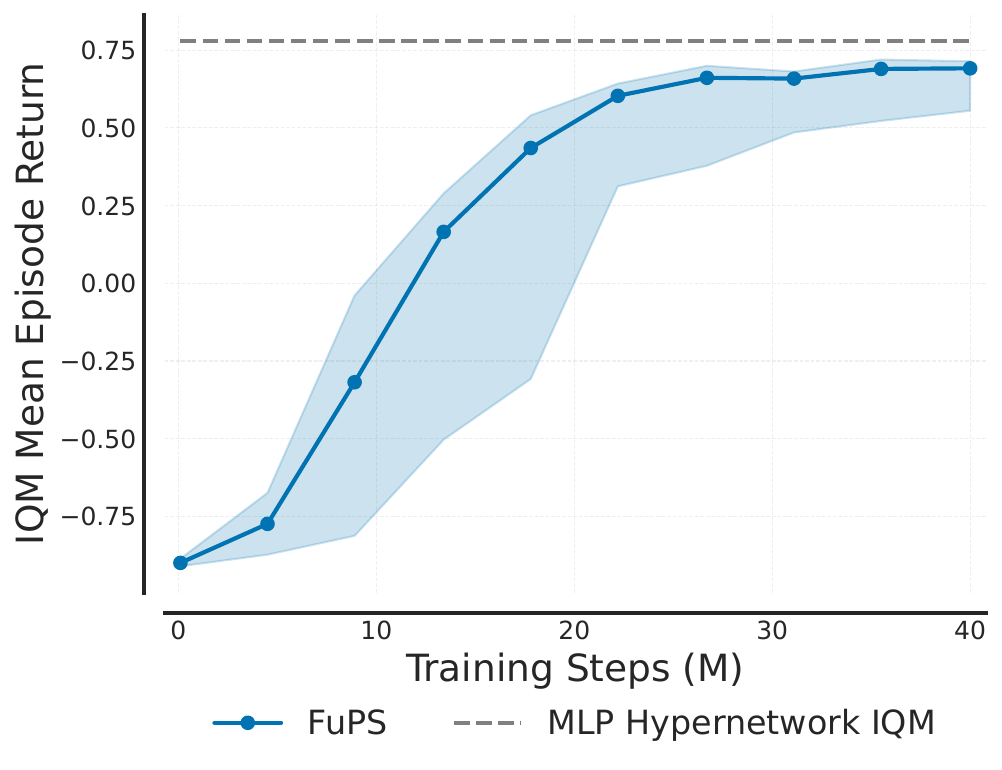}    
    \caption{We see that even if we run MAPPO-FuPS on Dispersion for 40 million timesteps (double the timesteps of MLP Hypernetwork), it converges to suboptimal performance.}
    \label{fig:run_longer}
\end{figure}

\subsection{Detailed MAMujoco Plots}

\begin{figure}[htbp]
    \centering
    \includegraphics[width=0.8\linewidth]{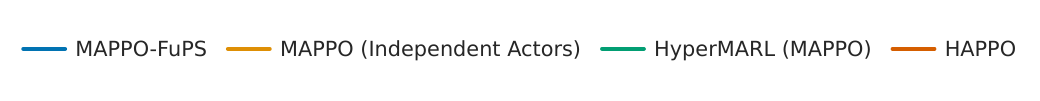}
    \begin{subfigure}[b]{0.48\textwidth}
        \centering
        \includegraphics[width=\textwidth]{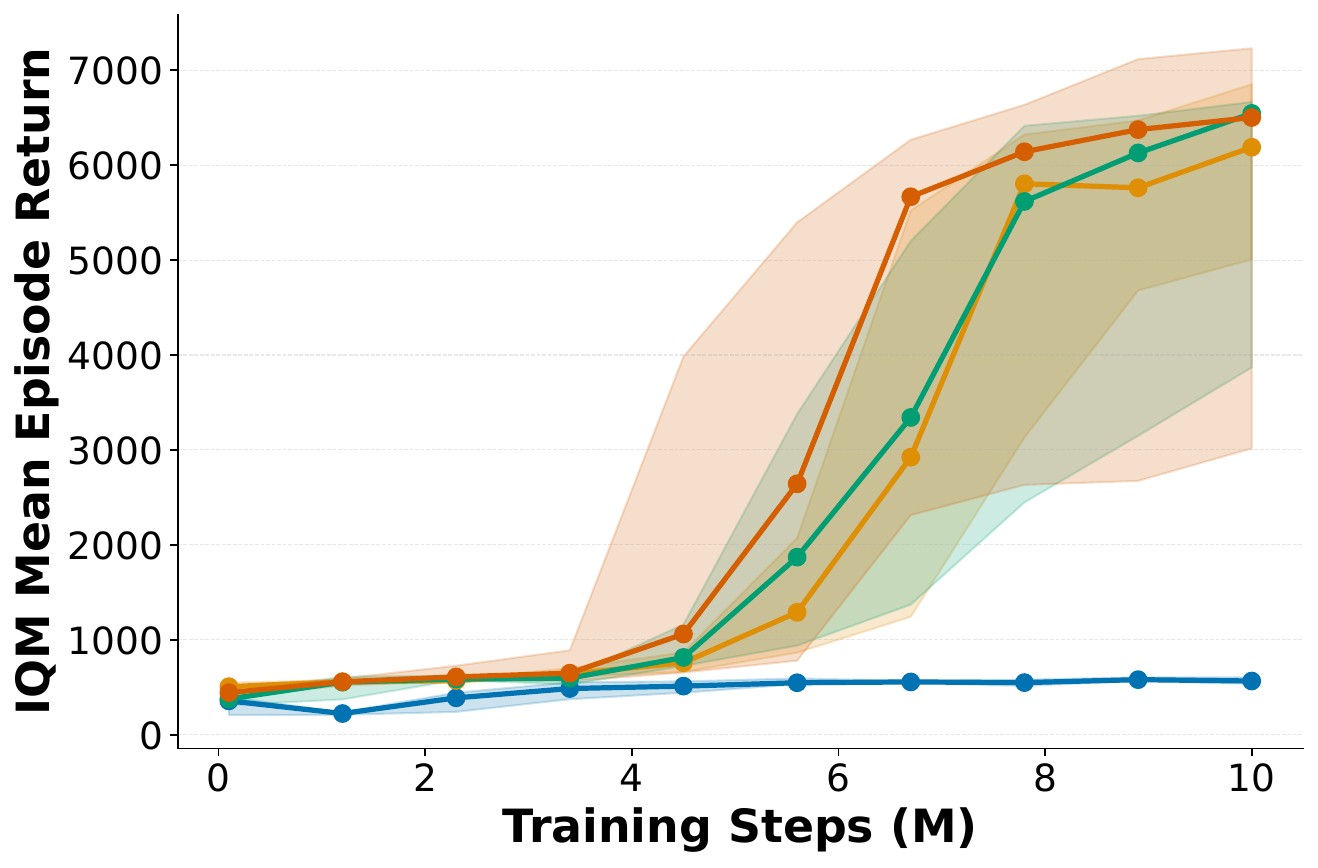}
        \caption{Humanoid-v2 17x1}
    \end{subfigure}
    \hfill
    \begin{subfigure}[b]{0.48\textwidth}
        \centering
        \includegraphics[width=\textwidth]{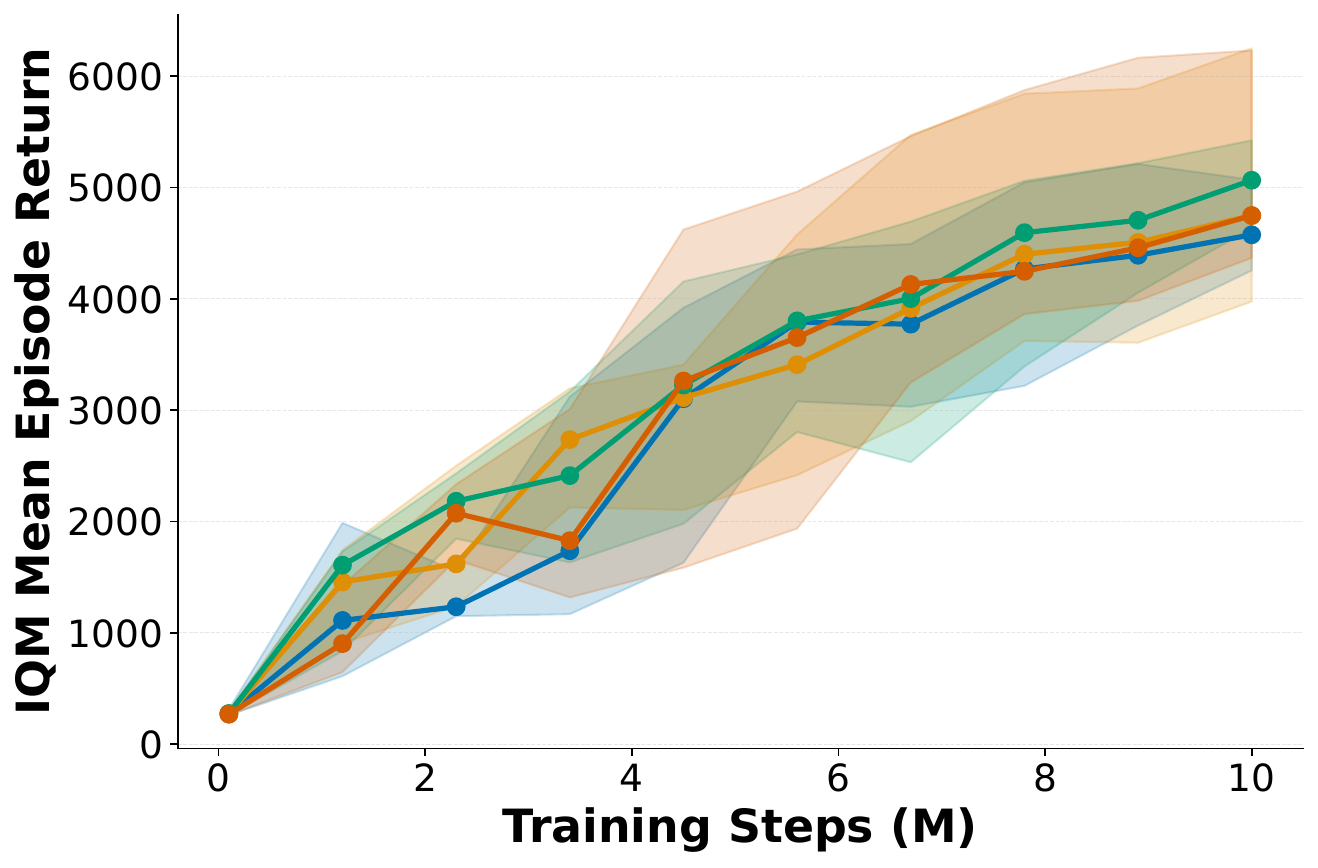}
        \caption{Walker2d-v2 2x3}
    \end{subfigure}
    \hfill
    \begin{subfigure}[b]{0.48\textwidth}
        \centering
        \includegraphics[width=\textwidth]{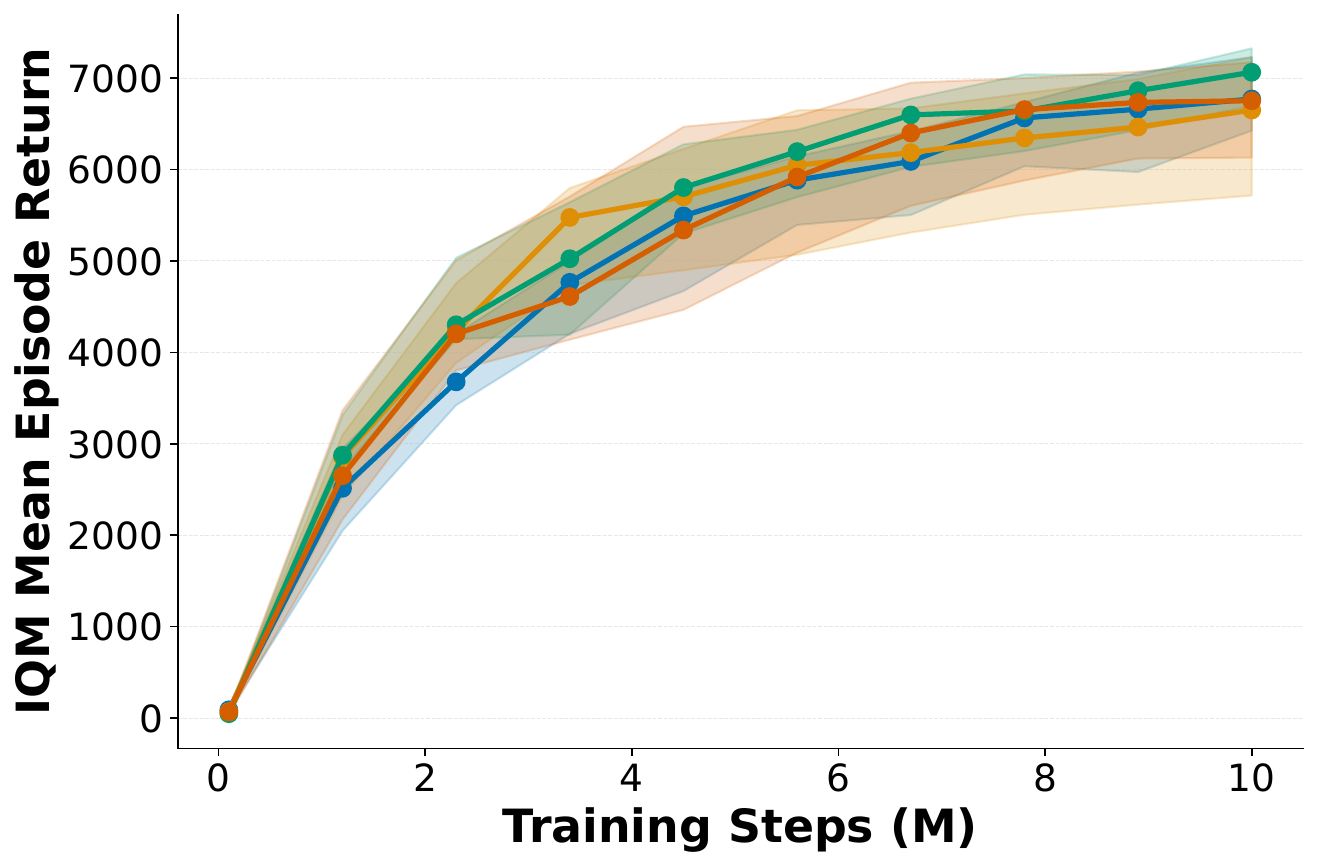}
        \caption{HalfCheetah-v2 2x3}
    \end{subfigure}
    \hfill 
    \begin{subfigure}[b]{0.48\textwidth}
        \centering
        \includegraphics[width=\textwidth]{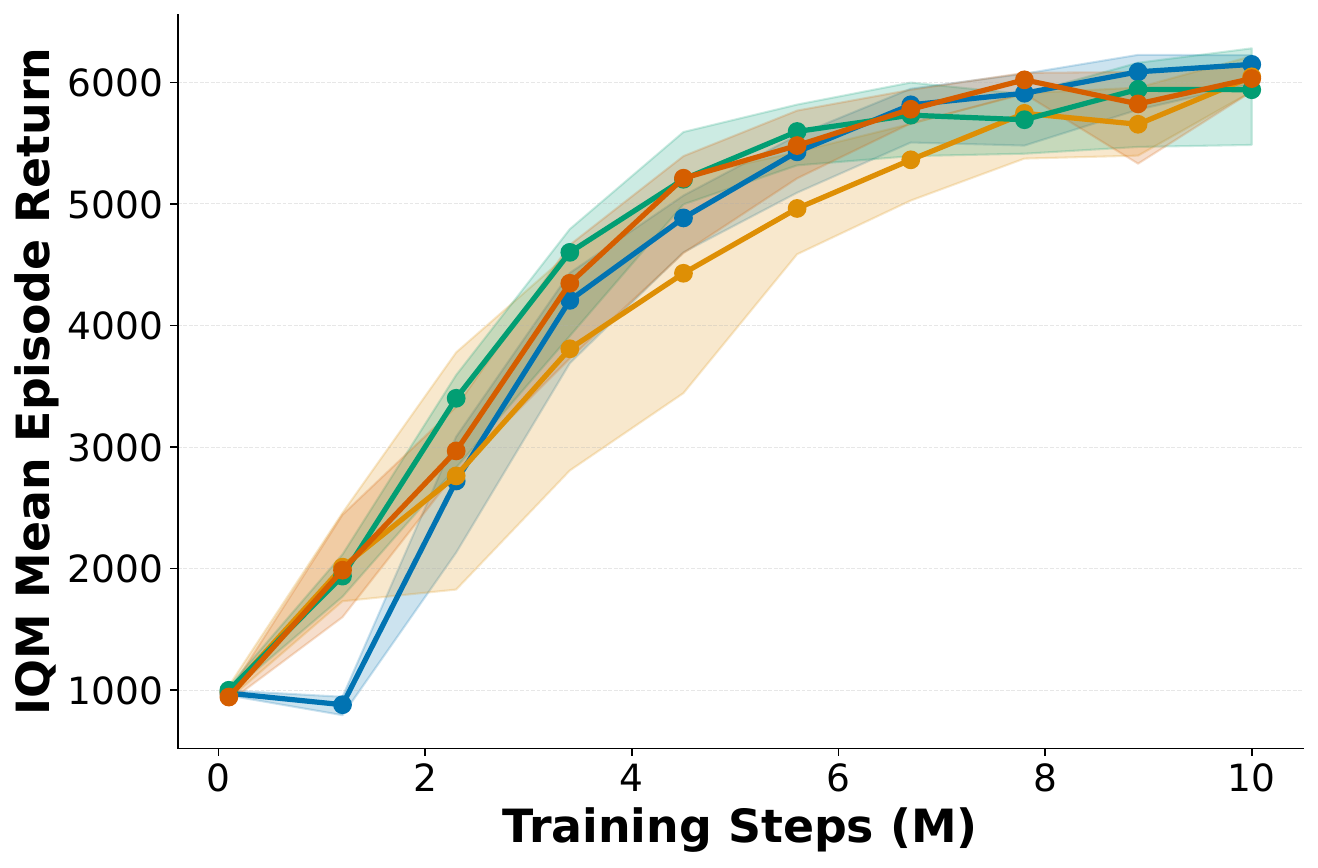}
        \caption{Ant-v2 4x2}
    \end{subfigure}
    \caption{Performance of Recurrent IPPO and MAPPO on MaMoJoCo. HyperMARL performs comparably to these baselines, and is the only method with shared actors to
demonstrate stable learning in the notoriously difficult 17-agent Humanoid environment.}
    \label{fig:performance_mamujoco}
\end{figure}

\clearpage

\subsection{Detailed Navigation Plots}\label{app:dico_results}

\begin{figure}[htbp] %
    \centering

    \begin{subfigure}[b]{0.32\textwidth} %
        \centering
        \includegraphics[width=\linewidth]{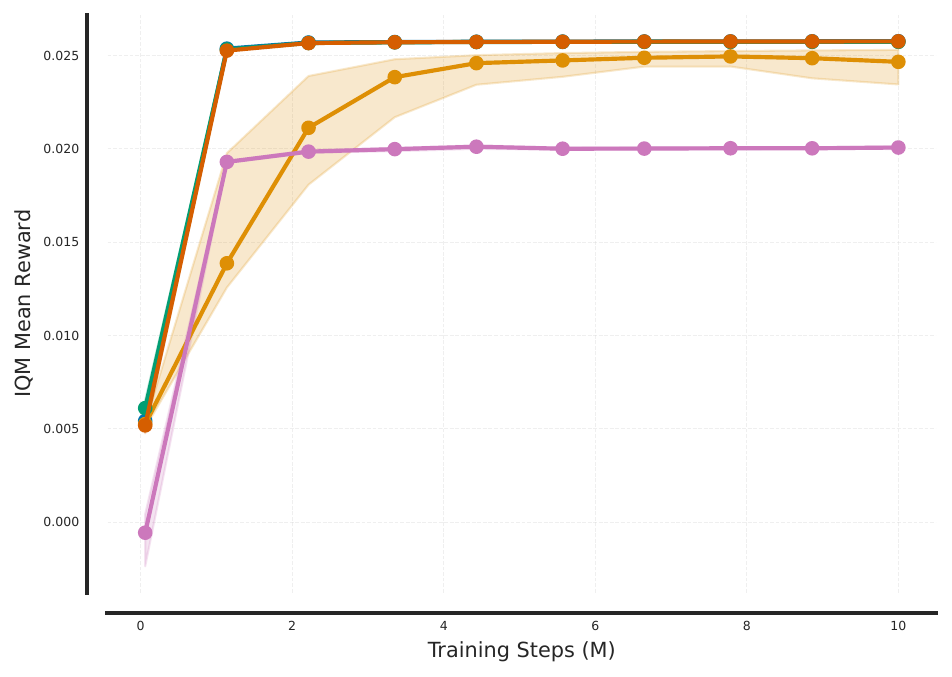}
        \caption{8 Agents, Same Goal}
        \label{fig:nav_8_same}
    \end{subfigure}\hfill
    \begin{subfigure}[b]{0.32\textwidth}
        \centering
        \includegraphics[width=\linewidth]{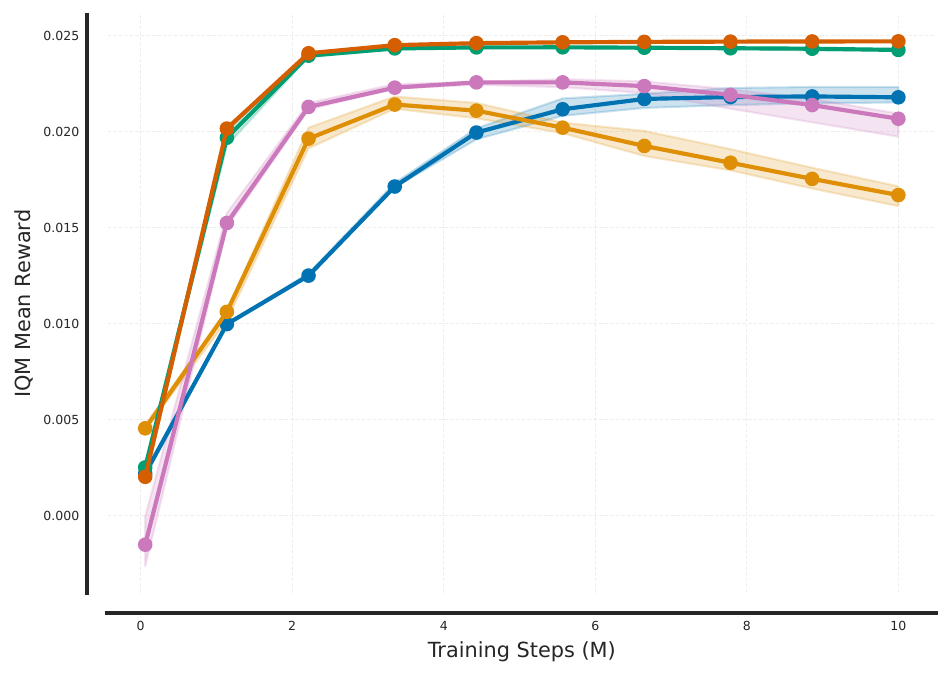}
        \caption{8 Agents, Different Goals}
        \label{fig:nav_8_diff}
    \end{subfigure}\hfill
    \begin{subfigure}[b]{0.32\textwidth}
        \centering
        \includegraphics[width=\linewidth]{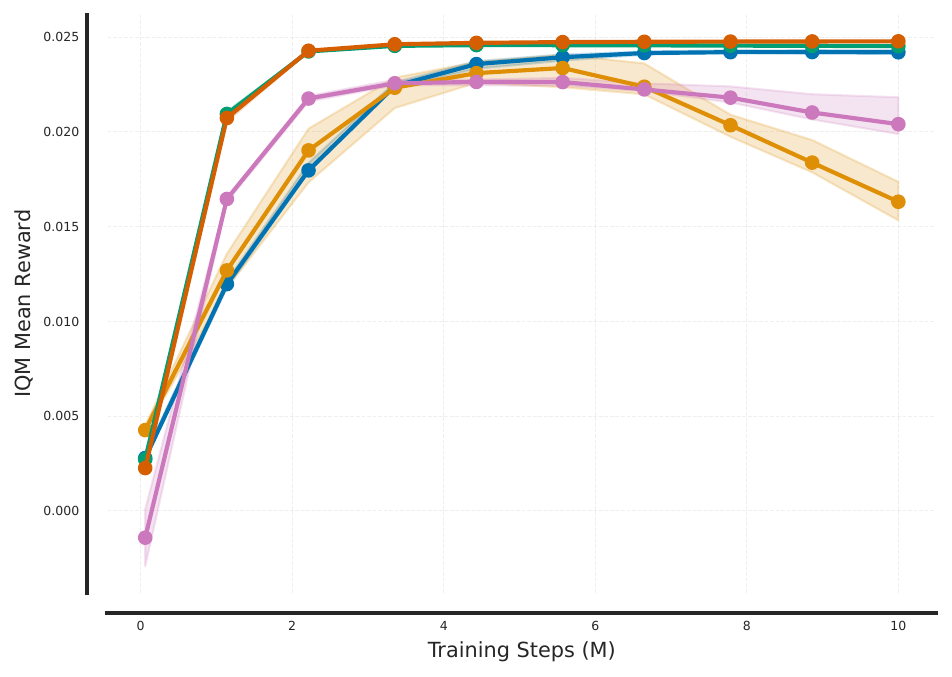}
        \caption{8 Agents, Four Goals}
        \label{fig:nav_8_four}
    \end{subfigure}

    \begin{subfigure}[b]{0.32\textwidth}
        \centering
        \includegraphics[width=\linewidth]{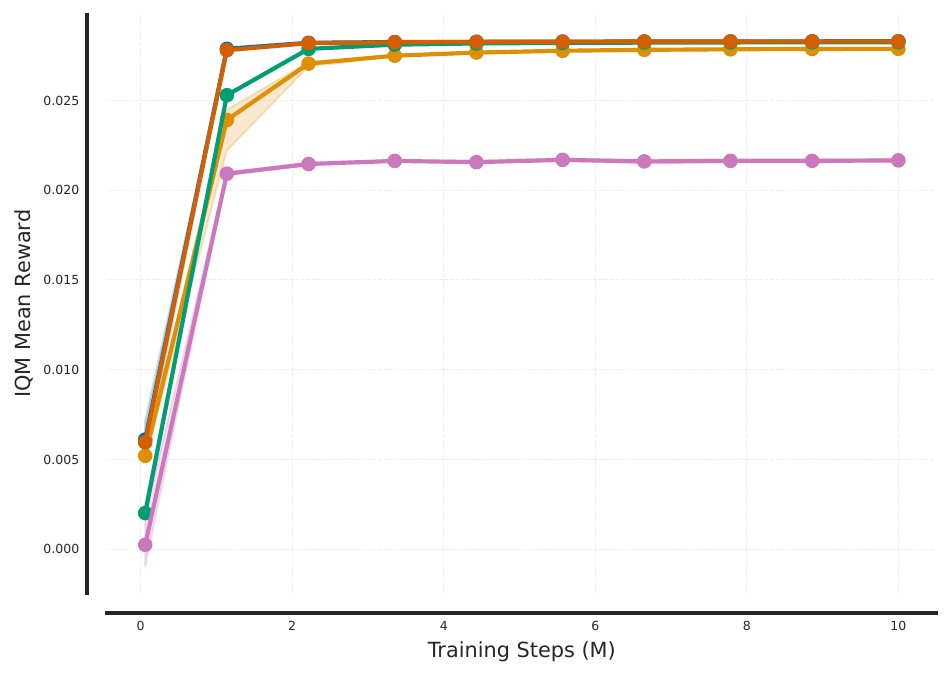}
        \caption{4 Agents, Same Goal}
        \label{fig:nav_4_same}
    \end{subfigure}\hfill
    \begin{subfigure}[b]{0.32\textwidth}
        \centering
        \includegraphics[width=\linewidth]{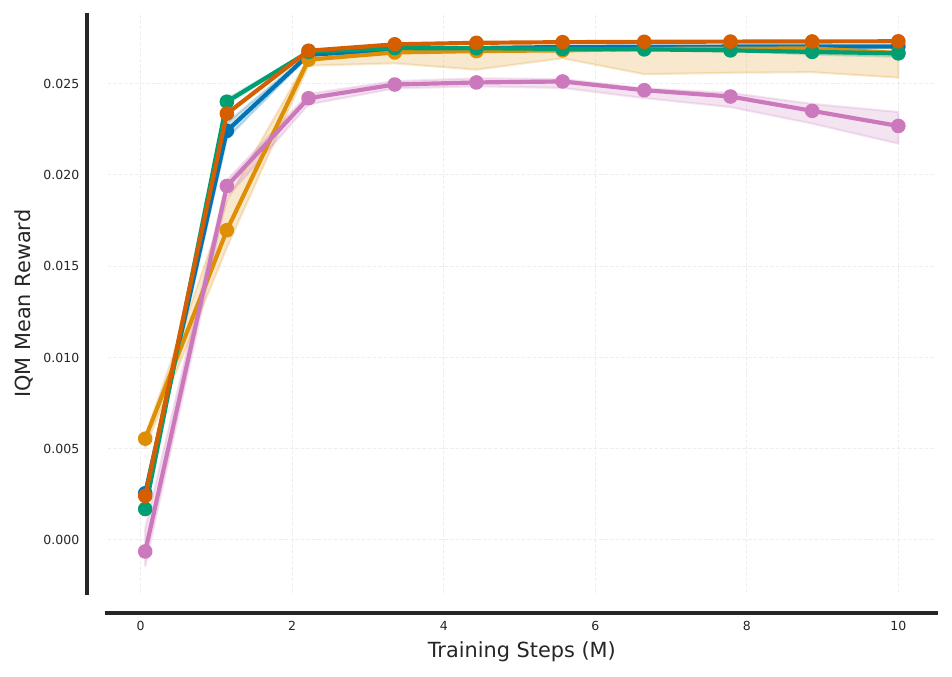}
        \caption{4 Agents, Different Goals}
        \label{fig:nav_4_diff}
    \end{subfigure}\hfill
    \begin{subfigure}[b]{0.32\textwidth}
        \centering
        \includegraphics[width=\linewidth]{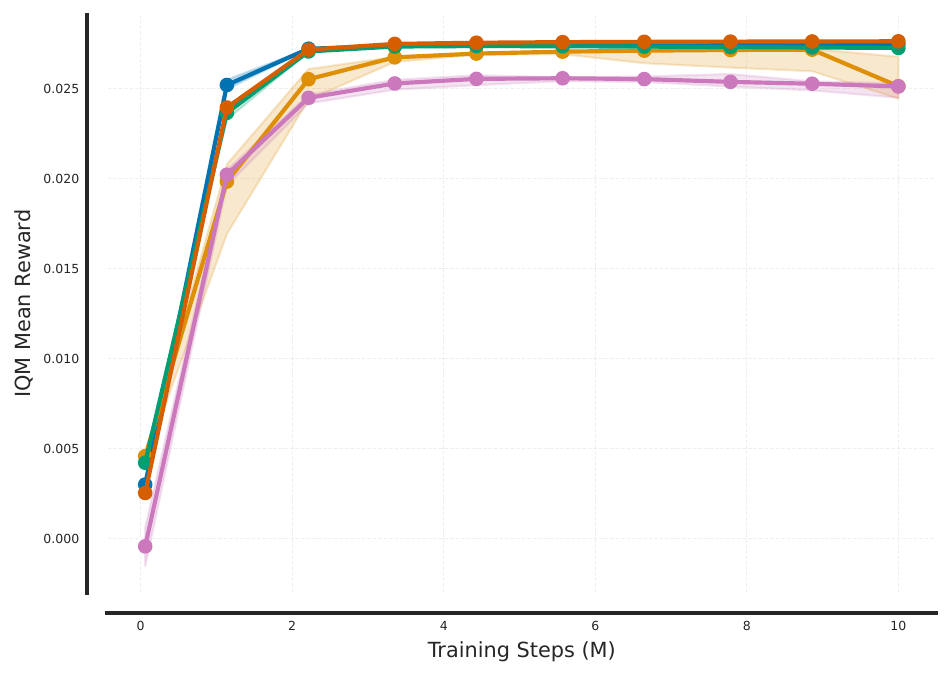}
        \caption{4 Agents, Two Goals}
        \label{fig:nav_4_two}
    \end{subfigure}

    \begin{minipage}{\textwidth} %
    \centering
        \begin{subfigure}[b]{0.32\textwidth}
            \centering
            \includegraphics[width=\linewidth]{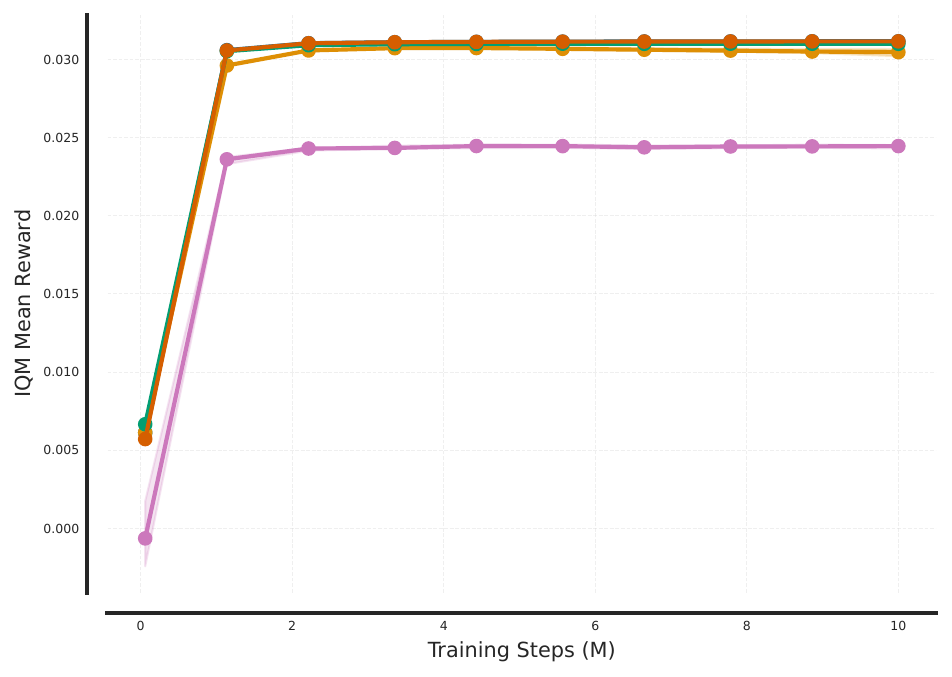}
            \caption{2 Agents, Same Goal}
            \label{fig:nav_2_same}
        \end{subfigure}\hspace{0.02\textwidth} %
        \begin{subfigure}[b]{0.32\textwidth}
            \centering
            \includegraphics[width=\linewidth]{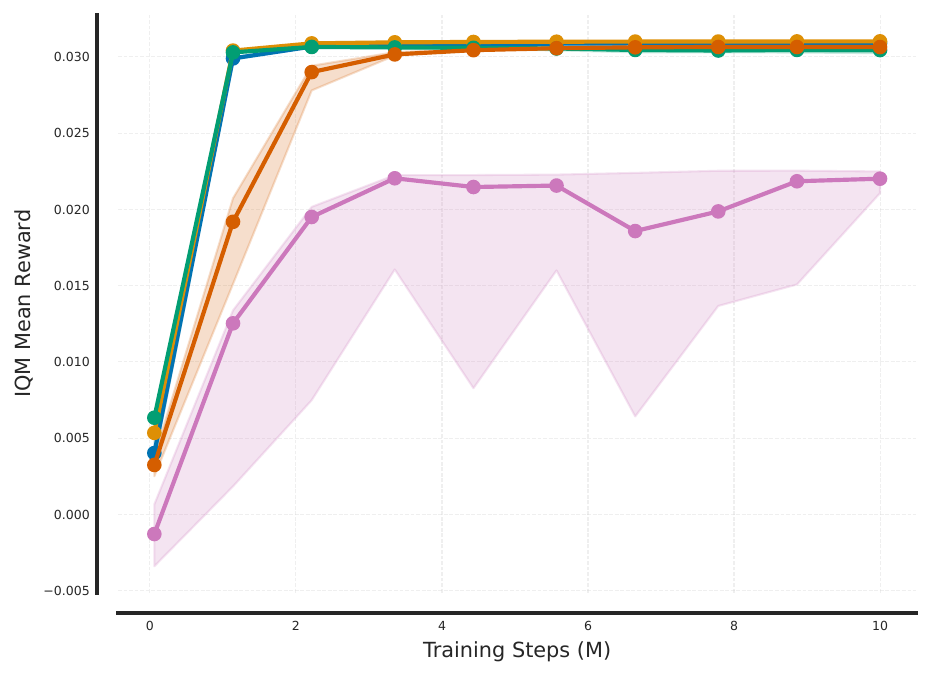}
            \caption{2 Agents, Different Goals}
            \label{fig:nav_2_diff}
        \end{subfigure}
    \end{minipage}

    \begin{center} %
        \includegraphics[width=0.8\textwidth]{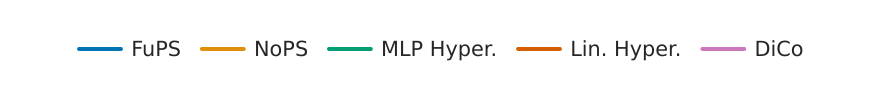} %
    \end{center}

    \caption{Sample efficiency of IPPO variants in the VMAS Navigation environment. Plots show IQM and 95\% CI (shaded regions) of mean episode return against training steps for different agent counts (rows: 8, 4, 2 agents) and goal configurations (columns, where applicable: Same, Different, Specific Goal Counts). Legend shown at bottom applies to all subplots.}
    \label{fig:sample_efficiency_ippo}
    \vspace{-10pt} %
\end{figure}

\subsection{Interval Estimates - SMAX}\label{app:interval_est_smax}
\begin{figure*}[ht]
    
    \begin{subfigure}[b]{0.45\linewidth}
        \centering
        \includegraphics[width=\linewidth]{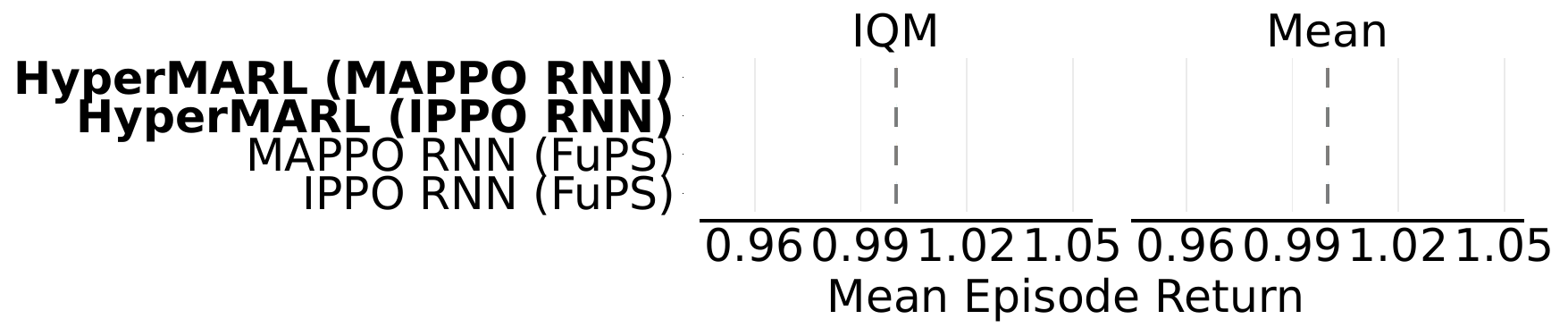}
        \caption{2s3z}
        \label{fig:smax_2s3z_interval}
    \end{subfigure}
    \hspace{0.05\linewidth}
    \begin{subfigure}[b]{0.45\linewidth}
        \centering
        \includegraphics[width=\linewidth]{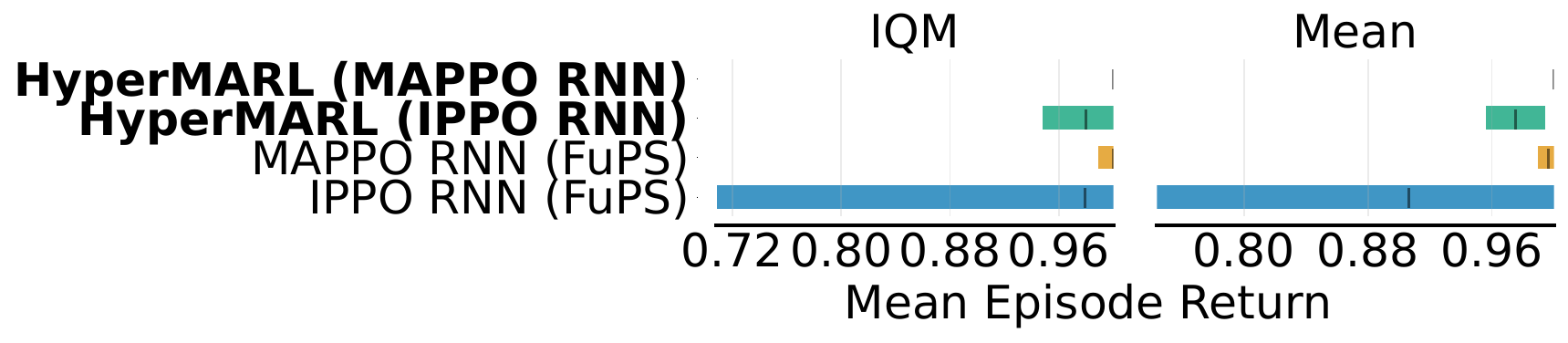}
        \caption{3s5z}
        \label{fig:smax_3s5z_interval}
    \end{subfigure}
    
    \vspace{1em}
    
    \begin{subfigure}[b]{0.45\linewidth}
        \centering
        \includegraphics[width=\linewidth]{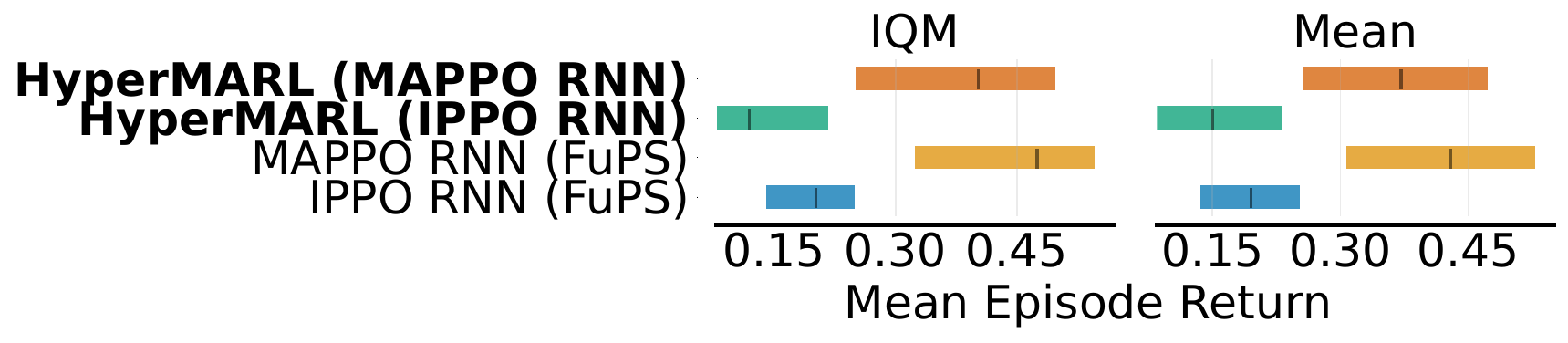}
        \caption{SMACv2 10 Units}
        \label{fig:smax_10_units_interval}
    \end{subfigure}
    \hspace{0.05\linewidth}
    \begin{subfigure}[b]{0.45\linewidth}
        \centering
        \includegraphics[width=\linewidth]{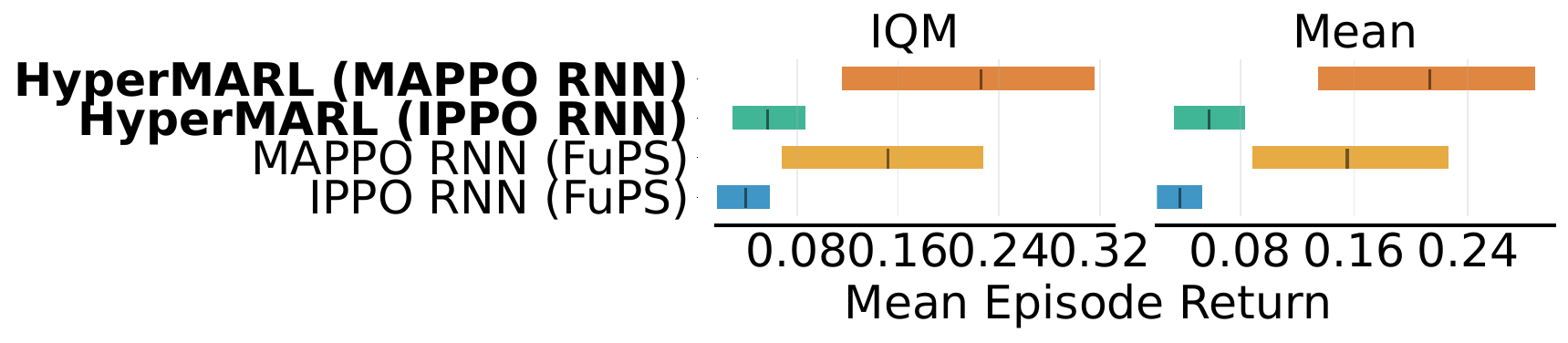}
        \caption{SMACv2 20 Units}
        \label{fig:smax_20_units_interval}
    \end{subfigure}

    \caption{\textit{Performance comparison in SMAX environments after 10 million timesteps.} We show the Interquartile Mean (IQM) of the Mean Win Rate and the 95\% Stratified Bootstrap Confidence Intervals (CI). HyperMARL performs comparably to FuPS baselines across all environments, demonstrating its effectiveness in tasks requiring homogeneous behaviours and using recurrency.}
    \label{fig:smax_interval_estimates}
\end{figure*}

\subsection{Additional Ablations}\label{app:additional_ablations}
\begin{figure}[h]
  \centering
 \includegraphics[width=0.6\linewidth]{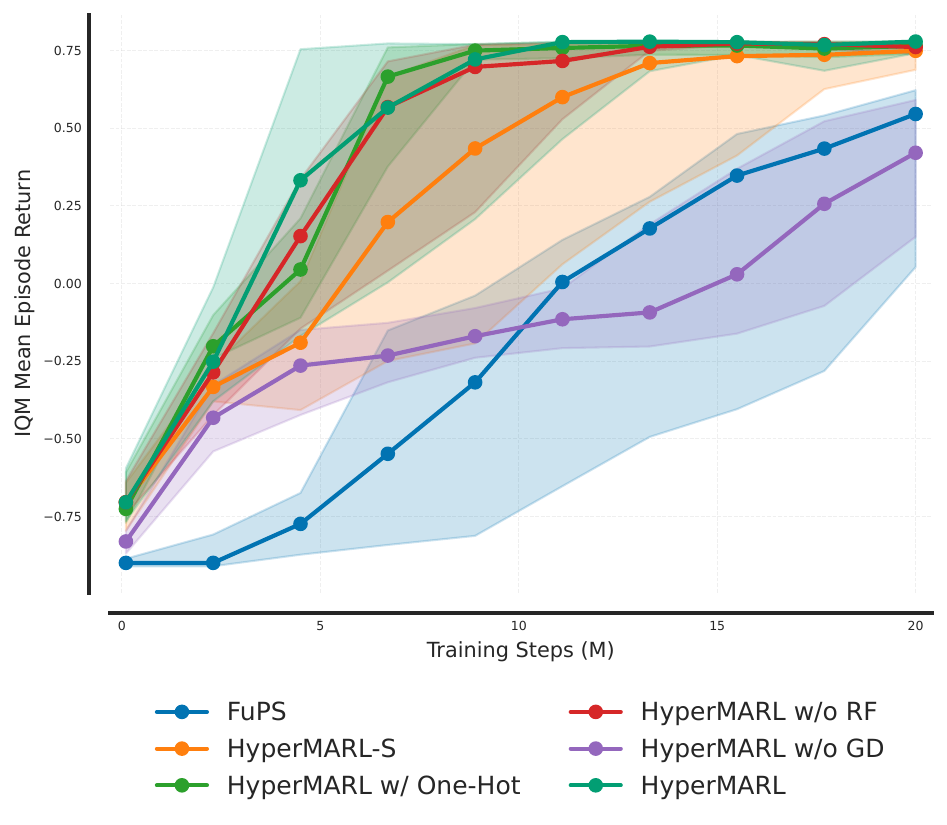}
  \caption{\textit{Ablation results comparing HyperMARL with its variants in Dispersion.} The results highlight that gradient decoupling is essential for maintaining HyperMARL's performance.}
  \label{fig:ablations_dispersion}
\end{figure}

\clearpage
\section{Hyperparameters}\label{append:hyperparams}

\begin{table}[htbp]
  \centering
  \caption{Hyperparameters, Training and Evaluation for Specialisation and Synchronisation Game}
  \label{tab:hyperparameters_spec_syn_game}
  \begin{tabular}{@{}ll@{}}
    \toprule
    \textbf{Hyperparameter}       & \textbf{Value} \\ 
    \midrule
    \multicolumn{2}{@{}l}{\textit{Environment Configuration}} \\ 
    Number of agents              & 2, 4, 8, 16    \\
    Maximum steps per episode     & 10             \\
    \midrule
    \multicolumn{2}{@{}l}{\textit{Training Protocol}} \\ 
    Number of seeds               & 10             \\
    Training steps                & 10,000         \\
    Evaluation episodes           & 100            \\
    Evaluation interval           & 1,000 steps    \\
    Batch size                    & 32             \\
    \midrule
    \multicolumn{2}{@{}l}{\textit{Model Architecture}} \\ 
    Hidden layer size             & 8, 16, 32, 64  \\
    Activation function           & ReLU           \\
    Output activation             & Softmax        \\
    \midrule
    \multicolumn{2}{@{}l}{\textit{Optimization}} \\ 
    Learning rate                 & 0.01           \\
    Optimizer                     & SGD            \\
    \midrule
    \multicolumn{2}{@{}l}{\textit{Model Parameter Count}} \\ 
    \addlinespace
    2 Agents                      & 
     \begin{tabular}[t]{@{}l@{}}
        NoPS: 60 \\ 
        FuPS: 58 \\ 
        FuPS+ID: 74 \\ 
        FuPS+ID (No State): 42
      \end{tabular} \\
    \addlinespace
    4 Agents                      & 
      \begin{tabular}[t]{@{}l@{}}
        NoPS: 352 \\ 
        FuPS: 240 \\ 
        FuPS+ID: 404 \\ 
        FuPS+ID (No State): 148
      \end{tabular} \\
    \addlinespace
    8 Agents                      & 
      \begin{tabular}[t]{@{}l@{}}
        NoPS: 2400 \\ 
        FuPS: 2344 \\ 
        FuPS+ID: 2600 \\ 
        FuPS+ID (No State): 552
      \end{tabular} \\
    \addlinespace
    16 Agents                     & 
      \begin{tabular}[t]{@{}l@{}}
        NoPS: 17728 \\ 
        FuPS: 17488 \\ 
        FuPS+ID: 18512 \\ 
        FuPS+ID (No State): 2128
      \end{tabular} \\
    \bottomrule
  \end{tabular}
\end{table}

\begin{table}[htbp]
\centering
\caption{IPPO and MAPPO Hyperparameters in Dispersion}
\begin{tabular}{lc}
\hline
\textbf{Hyperparameter} & \textbf{Value} \\
\hline
LR & 0.0005 \\
GAMMA & 0.99 \\
VF\_COEF & 0.5 \\
CLIP\_EPS & 0.2 \\
ENT\_COEF & 0.01 \\
NUM\_ENVS & 16 \\
NUM\_STEPS & 128 \\
GAE\_LAMBDA & 0.95 \\
NUM\_UPDATES & 9765 \\
EVAL\_EPISODES & 32 \\
EVAL\_INTERVAL & 100000 \\
MAX\_GRAD\_NORM & 0.5 \\
UPDATE\_EPOCHS & 4 \\
NUM\_MINIBATCHES & 2 \\
TOTAL\_TIMESTEPS & 20000000 \\
ANNEAL\_LR & false \\
ACTOR\_LAYERS & [64, 64] \\
CRITIC\_LAYERS & [64, 64] \\
ACTIVATION & relu \\
SEEDS & 30,1,42,72858,2300658 \\
ACTION\_SPACE\_TYPE & discrete \\
\hline
\end{tabular}
\end{table}

\begin{table}[htbp]
\centering
\caption{MLP Hypernet Hyperparameters in Dispersion}
\small
\begin{tabular}{@{}lcc@{}}
\hline
\textbf{Parameter} & \textbf{IPPO} & \textbf{MAPPO} \\
\hline
HYPERNET\_EMBEDDING\_DIM & 4 & 8 \\
EMBEDDING\_DIM Sweep & [4, 16, 64] & [4, 8, 16, 64] \\
HYPERNET\_HIDDEN\_DIMS & 64 & 64 \\
\hline
\end{tabular}
\end{table}

\begin{table}[htbp]
\centering
\caption{Dispersion Settings}
\begin{tabular}{ll}
\hline
\textbf{Setting} & \textbf{Value} \\
\hline
n\_food & 4 \\
n\_agents & 4 \\
max\_steps & 100 \\
food\_radius & 0.08 \\
share\_reward & false \\
penalise\_by\_time & true \\
continuous\_actions & false \\
\hline
\end{tabular}
\end{table}

\begin{table}[htbp]
\centering
\caption{IPPO Hyperparameters for Navigation}
\small
\begin{tabular}{@{}ll@{}}
\hline
\textbf{Hyperparameters} & \textbf{Value} \\
\hline
LR & 0.00005 \\
NUM\_ENVS & 600 \\
NUM\_STEPS & 100 \\
TOTAL\_TIMESTEPS & $10^6$ \\
UPDATE\_EPOCHS & 45 \\
NUM\_MINIBATCHES & 30 \\
GAMMA & 0.9 \\
GAE\_LAMBDA & 0.9 \\
CLIP\_EPS & 0.2 \\
ENT\_COEF & 0.0 \\
VF\_COEF & 1.0 \\
MAX\_GRAD\_NORM & 5 \\
ACTIVATION & tanh \\
ANNEAL\_LR & False \\
ACTOR\_LAYERS & [256, 256] \\
CRITIC\_LAYERS & [256, 256] \\
ACTION\_SPACE\_TYPE & continuous \\
\hline
\end{tabular}
\end{table}

\begin{table}[htbp]
\centering
\caption{MLP Hypernet Hyperparameters in Navigation}
\small
\begin{tabular}{@{}lcc@{}}
\hline
\textbf{Parameter} & \textbf{IPPO} & \textbf{MAPPO} \\
\hline
HYPERNET\_EMBEDDING\_DIM & 4 & 8 \\
EMBEDDING\_DIM Sweep & [4, 16, 64] & [4, 8, 16, 64] \\
HYPERNET\_HIDDEN\_DIMS & 64 & 64 \\
\hline
\end{tabular}
\end{table}

\begin{table}[htbp]
\centering
\caption{DiCo Algorithm $SND\_des$ Hyperparameter}
\label{tab:dico-snd-des}
\small
\begin{tabular}{@{}llc@{}}
\hline
\textbf{Goal Configuration} & \textbf{Number of Agents} & \textbf{SND\_des} \\
\hline
\multirow{3}{*}{All agents same goal} & 2 & 0 \\
 & 4 & 0 \\
 & 8 & 0 \\
\hline
\multirow{3}{*}{All agents different goals} & 2 & 1.2 (From DiCo paper) \\
 & 4 & [-1,1.2,2.4] $\Rightarrow$ -1 (Best) \\
 & 8 & [-1,1.2,4.8] $\Rightarrow$ -1 (Best)\\
\hline
\multirow{2}{*}{Some agents share goals} & 4 & [-1,1.2] $\Rightarrow$ -1 (Best)\\
 & 8 & [-1,2.4,1.2] $\Rightarrow$ -1 (Best) \\
\hline
\end{tabular}
\end{table}

\begin{table*}[ht] \label{}
\centering
\caption{Parameter Sweeps for IPPO Variants in Navigation Tasks with Four and Eight Agents}
\label{tab:param_sweeps_navigation}
\begin{tabular}{ll}
\toprule
\textbf{Parameter Sweeps} & \\
\midrule
CLIP\_EPS & 0.2, 0.1 \\
LR        & 5e-5, 5e-4, 2.5e-4 \\
\bottomrule
\end{tabular}

\vspace{1em}

\begin{tabular}{llc}
\toprule
\textbf{Algorithm} & \textbf{Setting}                     & \textbf{Selected Values} \\
\midrule
IPPO-FuPS            & 8 Agents (Same Goals)               & 0.2, 5e-5 \\
                     & 8 Agents (Different Goals)          & 0.1, 5e-5 \\
                     & 8 Agents (Four Goals)               & 0.1, 5e-5 \\
                     & 4 Agents (Same Goals)               & 0.2, 5e-5 \\
                     & 4 Agents (Different Goals)          & 0.2, 5e-5 \\
                     & 4 Agents (Two Goals)                & 0.2, 5e-5 \\
\midrule
IPPO-Linear Hypernetwork & 8 Agents (Same Goals)           & 0.2, 5e-5 \\
                         & 8 Agents (Different Goals)      & 0.1, 5e-5 \\
                         & 8 Agents (Four Goals)           & 0.1, 5e-5 \\
                         & 4 Agents (Same Goals)           & 0.2, 5e-5 \\
                         & 4 Agents (Different Goals)      & 0.1, 5e-5 \\
                         & 4 Agents (Two Goals)            & 0.1, 5e-5 \\
\midrule
IPPO-MLP Hypernetwork    & 8 Agents (Same Goals)           & 0.2, 5e-5 \\
                         & 8 Agents (Different Goals)      & 0.1, 5e-5 \\
                         & 8 Agents (Four Goals)           & 0.1, 5e-5 \\
                         & 4 Agents (Same Goals)           & 0.1, 5e-5 \\
                         & 4 Agents (Different Goals)      & 0.1, 5e-5 \\
                         & 4 Agents (Two Goals)            & 0.1, 5e-5 \\
\midrule
IPPO-NoPS                & 8 Agents (Same Goals)           & 0.1, 5e-5 \\
                         & 8 Agents (Different Goals)      & 0.2, 5e-5 \\
                         & 8 Agents (Four Goals)           & 0.1, 5e-5 \\
                         & 4 Agents (Same Goals)           & 0.1, 5e-5 \\
                         & 4 Agents (Different Goals)      & 0.2, 5e-5  \\
                         & 4 Agents (Two Goals)            & 0.1, 5e-5   \\
\midrule
IPPO-Dico                & 8 Agents (Same Goals)           & 0.2, 5e-5 \\
                         & 8 Agents (Different Goals)      & 0.1, 2.5e-4 \\
                         & 8 Agents (Four Goals)           & 0.1, 2.5e-4 \\
                         & 4 Agents (Same Goals)           & 0.2, 5e-5 \\
                         & 4 Agents (Different Goals)      & 0.1, 2.5e-4 \\
                         & 4 Agents (Two Goals)            & 0.1, 5e-4 \\
\bottomrule
\end{tabular}
\end{table*}

\begin{table}[htbp]
\centering
\caption{Environment Settings for Navigation Task}
\small
\begin{tabular}{@{}ll@{}}
\hline
\textbf{Parameter} & \textbf{Value} \\
\hline
n\_agents & 2,4,8 \\
agents\_with\_same\_goal & 1, n\_agents/2, n\_agents \\
max\_steps & 100 \\
collisions & False \\
split\_goals & False \\
observe\_all\_goals & True \\
shared\_rew & False \\
lidar\_range & 0.35 \\
agent\_radius & 0.1 \\
continuous\_actions & True \\
\hline
\end{tabular}
\end{table}

\begin{table*}[ht]
    \centering
    \caption{Default algorithm and model hyperparameters for the Ant-v2-4x2 environment (from ~\citep{JMLR:v25:23-0488}).}
    \label{tab:algo_model_config}
    \begin{tabular}{ll}
        \toprule
        \textbf{Parameter} & \textbf{Value} \\
        \midrule
        \multicolumn{2}{l}{\textbf{--- Algorithm Parameters ---}} \\
        action\_aggregation & prod \\
        actor\_num\_mini\_batch & 1 \\
        clip\_param & 0.1 \\
        critic\_epoch & 5 \\
        critic\_num\_mini\_batch & 1 \\
        entropy\_coef & 0 \\
        fixed\_order & true \\
        gae\_lambda & 0.95 \\
        gamma & 0.99 \\
        huber\_delta & 10.0 \\
        max\_grad\_norm & 10.0 \\
        ppo\_epoch & 5 \\
        share\_param & false \\
        use\_clipped\_value\_loss & true \\
        use\_gae & true \\
        use\_huber\_loss & true \\
        use\_max\_grad\_norm & true \\
        use\_policy\_active\_masks & true \\
        value\_loss\_coef & 1 \\
        \midrule
        \multicolumn{2}{l}{\textbf{--- Model Parameters ---}} \\
        activation\_func & relu \\
        critic\_lr & 0.0005 \\
        data\_chunk\_length & 10 \\
        gain & 0.01 \\
        hidden\_sizes & [128, 128, 128] \\
        initialization\_method & orthogonal\_ \\
        lr & 0.0005 \\
        opti\_eps & 1e-05 \\
        recurrent\_n & 1 \\
        std\_x\_coef & 1 \\
        std\_y\_coef & 0.5 \\
        use\_feature\_normalization & true \\
        use\_naive\_recurrent\_policy & false \\
        use\_recurrent\_policy & false \\
        weight\_decay & 0 \\
        \bottomrule
    \end{tabular}
\end{table*}

\begin{table*}[ht]
    \centering
    \caption{Default algorithm and model hyperparameters for the Humanoid-v2-17x1 environment (from ~\citep{JMLR:v25:23-0488}).}
    \label{tab:algo_model_config_humanoid}
    \begin{tabular}{ll}
        \toprule
        \textbf{Parameter} & \textbf{Value} \\
        \midrule
        \multicolumn{2}{l}{\textbf{--- Algorithm Parameters ---}} \\
        action\_aggregation & prod \\
        actor\_num\_mini\_batch & 1 \\
        clip\_param & 0.1 \\
        critic\_epoch & 5 \\
        critic\_num\_mini\_batch & 1 \\
        entropy\_coef & 0 \\
        fixed\_order & false \\
        gae\_lambda & 0.95 \\
        gamma & 0.99 \\
        huber\_delta & 10.0 \\
        max\_grad\_norm & 10.0 \\
        ppo\_epoch & 5 \\
        share\_param & false \\
        use\_clipped\_value\_loss & true \\
        use\_gae & true \\
        use\_huber\_loss & true \\
        use\_max\_grad\_norm & true \\
        use\_policy\_active\_masks & true \\
        value\_loss\_coef & 1 \\
        \midrule
        \multicolumn{2}{l}{\textbf{--- Model Parameters ---}} \\
        activation\_func & relu \\
        critic\_lr & 0.0005 \\
        data\_chunk\_length & 10 \\
        gain & 0.01 \\
        hidden\_sizes & [128, 128, 128] \\
        initialization\_method & orthogonal\_ \\
        lr & 0.0005 \\
        opti\_eps & 1e-05 \\
        recurrent\_n & 1 \\
        std\_x\_coef & 1 \\
        std\_y\_coef & 0.5 \\
        use\_feature\_normalization & true \\
        use\_naive\_recurrent\_policy & false \\
        use\_recurrent\_policy & false \\
        weight\_decay & 0 \\
        \bottomrule
    \end{tabular}
\end{table*}

\begin{table*}[ht]
    \centering
    \caption{Default algorithm and model hyperparameters for the Walker2d-v2-2x3 environment (from ~\citep{JMLR:v25:23-0488}).}
    \label{tab:algo_model_config_walker2d}
    \begin{tabular}{ll}
        \toprule
        \textbf{Parameter} & \textbf{Value} \\
        \midrule
        \multicolumn{2}{l}{\textbf{--- Algorithm Parameters ---}} \\
        action\_aggregation & prod \\
        actor\_num\_mini\_batch & 2 \\
        clip\_param & 0.05 \\
        critic\_epoch & 5 \\
        critic\_num\_mini\_batch & 2 \\
        entropy\_coef & 0 \\
        fixed\_order & false \\
        gae\_lambda & 0.95 \\
        gamma & 0.99 \\
        huber\_delta & 10.0 \\
        max\_grad\_norm & 10.0 \\
        ppo\_epoch & 5 \\
        share\_param & false \\
        use\_clipped\_value\_loss & true \\
        use\_gae & true \\
        use\_huber\_loss & true \\
        use\_max\_grad\_norm & true \\
        use\_policy\_active\_masks & true \\
        value\_loss\_coef & 1 \\
        \midrule
        \multicolumn{2}{l}{\textbf{--- Model Parameters ---}} \\
        activation\_func & relu \\
        critic\_lr & 0.001 \\
        data\_chunk\_length & 10 \\
        gain & 0.01 \\
        hidden\_sizes & 128, 128, 128 \\
        initialization\_method & orthogonal\_ \\
        lr & 0.001 \\
        opti\_eps & 1e-05 \\
        recurrent\_n & 1 \\
        std\_x\_coef & 1 \\
        std\_y\_coef & 0.5 \\
        use\_feature\_normalization & true \\
        use\_naive\_recurrent\_policy & false \\
        use\_recurrent\_policy & false \\
        weight\_decay & 0 \\
        \bottomrule
    \end{tabular}
\end{table*}

\begin{table*}[ht]
    \centering
    \caption{Default algorithm and model hyperparameters for the HalfCheetah-v2-2x3 environment (from ~\citep{JMLR:v25:23-0488}).}
    \label{tab:algo_model_config_halfcheetah}
    \begin{tabular}{ll}
        \toprule
        \textbf{Parameter} & \textbf{Value} \\
        \midrule
        \multicolumn{2}{l}{\textbf{--- Algorithm Parameters ---}} \\
        action\_aggregation & prod \\
        actor\_num\_mini\_batch & 1 \\
        clip\_param & 0.05 \\
        critic\_epoch & 15 \\
        critic\_num\_mini\_batch & 1 \\
        entropy\_coef & 0.01 \\
        fixed\_order & false \\
        gae\_lambda & 0.95 \\
        gamma & 0.99 \\
        huber\_delta & 10.0 \\
        max\_grad\_norm & 10.0 \\
        ppo\_epoch & 15 \\
        share\_param & false \\
        use\_clipped\_value\_loss & true \\
        use\_gae & true \\
        use\_huber\_loss & true \\
        use\_max\_grad\_norm & true \\
        use\_policy\_active\_masks & true \\
        value\_loss\_coef & 1 \\
        \midrule
        \multicolumn{2}{l}{\textbf{--- Model Parameters ---}} \\
        activation\_func & relu \\
        critic\_lr & 0.0005 \\
        data\_chunk\_length & 10 \\
        gain & 0.01 \\
        hidden\_sizes & 128, 128, 128 \\
        initialization\_method & orthogonal\_ \\
        lr & 0.0005 \\
        opti\_eps & 1e-05 \\
        recurrent\_n & 1 \\
        std\_x\_coef & 1 \\
        std\_y\_coef & 0.5 \\
        use\_feature\_normalization & true \\
        use\_naive\_recurrent\_policy & false \\
        use\_recurrent\_policy & false \\
        weight\_decay & 0 \\
        \bottomrule
    \end{tabular}
\end{table*}

\begin{table*}[ht]
    \centering
    \footnotesize %
    \caption{HyperMARL Hyperparameters Across MaMuJoCo Environments}
    \label{tab:hypermarl_hyperparams_all_environments}
    \begin{tabular}{@{}lccccc@{}}
        \toprule
        \textbf{Parameter} & \textbf{Humanoid} & \textbf{Walker2d} & \textbf{HalfCheetah} & \textbf{Ant} & \textbf{Sweeps} \\
        & \textbf{v2-17x1} & \textbf{v2-2x3} & \textbf{v2-2x3} & \textbf{v2-4x2} & \\
        \midrule
        AGENT\_ID\_DIM & 64 & 64 & 64 & 64 & None \\
        HNET\_HIDDEN\_DIMS & 64 & 64 & 64 & 64 & None\\
        clip\_param & 0.075 & 0.0375 & 0.0375 & 0.075 & [0.1,0.075,0.05,0.0375] \\ 
        \bottomrule
    \end{tabular}
\end{table*}

\begin{table}[htbp]
\centering
\caption{Recurrent IPPO and MAPPO Hyperparameters in SMAX (from JaxMARL paper)}
\begin{tabular}{lcc}
\hline
\textbf{Hyperparameter} & \textbf{IPPO Value} & \textbf{MAPPO Value} \\
\hline
LR & 0.004 & 0.002 \\
NUM\_ENVS & 128 & 128 \\
NUM\_STEPS & 128 & 128 \\
GRU\_HIDDEN\_DIM & 128 & 128 \\
FC\_DIM\_SIZE & 128 & 128 \\
TOTAL\_TIMESTEPS & 1e7 & 1e7 \\
UPDATE\_EPOCHS & 4 & 4 \\
NUM\_MINIBATCHES & 4 & 4 \\
GAMMA & 0.99 & 0.99 \\
GAE\_LAMBDA & 0.95 & 0.95 \\
CLIP\_EPS & 0.05 & 0.2 \\
SCALE\_CLIP\_EPS & False & False \\
ENT\_COEF & 0.01 & 0.0 \\
VF\_COEF & 0.5 & 0.5 \\
MAX\_GRAD\_NORM & 0.25 & 0.25 \\
ACTIVATION & relu & relu \\
SEED & 30,1,42,72858,2300658 & 30,1,42,72858,2300658 \\
ANNEAL\_LR & True & True \\
OBS\_WITH\_AGENT\_ID & - & True \\
\hline
\end{tabular}
\label{tab:base_hparams}
\end{table}

\clearpage
\newpage
\begin{table}[htbp]
\centering
\scriptsize
\caption{Hyperparameter Sweeps and Final Values for Different Maps in SMAX. H‐ refers to HyperMARL MLP Hypernetworks.}
\label{tab:smax_sweeps}
\resizebox{\columnwidth}{!}{%
\begin{tabular}{@{}l c c c c c@{}}
\toprule
\textbf{Map} & \textbf{Algorithm} & \textbf{LR Range} & \textbf{Chosen LR} & \textbf{HNET Embedding Dim} & \textbf{HNET Hidden Dims} \\
\midrule
\multirow{4}{*}{\textbf{2s3z}} 
 & IPPO    & 0.004                         & 0.004 & \multicolumn{2}{c}{--} \\
 & MAPPO   & 0.002                         & 0.002 & \multicolumn{2}{c}{--} \\
 & H-IPPO  & 0.004                         & 0.004 & 4                          & 32 \\
 & H-MAPPO & 0.002                         & 0.002 & 64                         & 16 \\
\addlinespace
\multirow{4}{*}{\textbf{3s5z}} 
 & IPPO    & 0.004                         & 0.004 & \multicolumn{2}{c}{--} \\
 & MAPPO   & 0.002, 0.005, 0.0003          & 0.002 & \multicolumn{2}{c}{--} \\
 & H-IPPO  & 0.004                         & 0.004 & 64                         & 16 \\
 & H-MAPPO & 0.002, 0.005, 0.0003          & 0.0003& 64                         & 16 \\
\addlinespace
\multirow{4}{*}{\textbf{smacv2\_10\_units}} 
 & IPPO    & 0.005, 0.001, 0.0003, 0.004   & 0.001 & \multicolumn{2}{c}{--} \\
 & MAPPO   & 0.002, 0.005, 0.0003          & 0.0003& \multicolumn{2}{c}{--} \\
 & H-IPPO  & 0.005, 0.001, 0.0003, 0.004   & 0.005 & 4                          & 64 \\
 & H-MAPPO & 0.002, 0.005, 0.0003          & 0.0003& 64                         & 16 \\
\addlinespace
\multirow{4}{*}{\textbf{smacv2\_20\_units}} 
 & IPPO    & 0.002, 0.005, 0.0003          & 0.005 & \multicolumn{2}{c}{--} \\
 & MAPPO   & 0.002, 0.005, 0.0003          & 0.0003& \multicolumn{2}{c}{--} \\
 & H-IPPO  & 0.002, 0.005, 0.0003          & 0.005 & 64                         & 64 \\
 & H-MAPPO & 0.002, 0.005, 0.0003          & 0.0003& 4                          & 64 \\
\bottomrule
\end{tabular}%
}
\vspace{0.5em}
\footnotesize
\textit{Note:} HNET Embedding Dim refers to the hypernetwork embedding dimension, chosen from the range \{4, 16, 64\}. HNET Hidden Dims refers to the hidden layer dimensions of the hypernetwork, chosen from the range \{16, 32, 64\}.
\end{table}

\section{Computational Resources}
\label{appendix:computational_resources}

\begin{table}[ht]
    \centering
    \footnotesize
    \caption{Computational Resources by Experiment Type}
    \label{tab:compute_resources}
    \begin{tabular}{@{}lccc@{}}
        \toprule
        \textbf{Experiment Category} & \textbf{Hardware Configuration} & \textbf{Execution Time} & \textbf{Total Hours} \\
        \midrule
        Specialisation, & 8 cores on AMD EPYC 7H12 & 2-6 hours per run & \multirow{2}{*}{250} \\
        Synchronisation \& Dispersion & 64-Core Processor & (agent-count dependent) & \\
        \midrule
        Navigation & 8 cores on AMD EPYC 7H12 & \multirow{2}{*}{4-10 hours per run} & \multirow{2}{*}{320} \\
        Experiments & 64-Core Processor + NVIDIA RTX A4500 & & \\
        \midrule
        MaMuJoCo & 8 cores on AMD EPYC 7H12 & 8-24 hours per run & \multirow{2}{*}{1,680} \\
        Experiments & 64-Core Processor + NVIDIA RTX A4500 & (scenario \& algorithm dependent) & \\
        \midrule
        SMAX & 8 cores on AMD EPYC 7H12 & \multirow{2}{*}{2-5 hours per run} & \multirow{2}{*}{160} \\
        Experiments & 64-Core Processor + NVIDIA RTX A4500 & & \\
        \bottomrule
    \end{tabular}
\end{table}

\clearpage
\newpage

\newpage
\section*{NeurIPS Paper Checklist}

\begin{enumerate}

\item {\bf Claims}
    \item[] Question: Do the main claims made in the abstract and introduction accurately reflect the paper's contributions and scope?
    \item[] Answer: \answerYes %
    \item[] Justification: Please refer to Sections~\ref{sec:motivation},~\ref{sec:experiments} and \ref{sec:ablations}.
    \item[] Guidelines:
    \begin{itemize}
        \item The answer NA means that the abstract and introduction do not include the claims made in the paper.
        \item The abstract and/or introduction should clearly state the claims made, including the contributions made in the paper and important assumptions and limitations. A No or NA answer to this question will not be perceived well by the reviewers. 
        \item The claims made should match theoretical and experimental results, and reflect how much the results can be expected to generalize to other settings. 
        \item It is fine to include aspirational goals as motivation as long as it is clear that these goals are not attained by the paper. 
    \end{itemize}

\item {\bf Limitations}
    \item[] Question: Does the paper discuss the limitations of the work performed by the authors?
    \item[] Answer: \answerYes{} %
    \item[] Justification: Please see Sections~\ref{app:limiations} and \ref{sec:scaling}.
    \item[] Guidelines:
    \begin{itemize}
        \item The answer NA means that the paper has no limitation while the answer No means that the paper has limitations, but those are not discussed in the paper. 
        \item The authors are encouraged to create a separate "Limitations" section in their paper.
        \item The paper should point out any strong assumptions and how robust the results are to violations of these assumptions (e.g., independence assumptions, noiseless settings, model well-specification, asymptotic approximations only holding locally). The authors should reflect on how these assumptions might be violated in practice and what the implications would be.
        \item The authors should reflect on the scope of the claims made, e.g., if the approach was only tested on a few datasets or with a few runs. In general, empirical results often depend on implicit assumptions, which should be articulated.
        \item The authors should reflect on the factors that influence the performance of the approach. For example, a facial recognition algorithm may perform poorly when image resolution is low or images are taken in low lighting. Or a speech-to-text system might not be used reliably to provide closed captions for online lectures because it fails to handle technical jargon.
        \item The authors should discuss the computational efficiency of the proposed algorithms and how they scale with dataset size.
        \item If applicable, the authors should discuss possible limitations of their approach to address problems of privacy and fairness.
        \item While the authors might fear that complete honesty about limitations might be used by reviewers as grounds for rejection, a worse outcome might be that reviewers discover limitations that aren't acknowledged in the paper. The authors should use their best judgment and recognize that individual actions in favor of transparency play an important role in developing norms that preserve the integrity of the community. Reviewers will be specifically instructed to not penalize honesty concerning limitations.
    \end{itemize}

\item {\bf Theory assumptions and proofs}
    \item[] Question: For each theoretical result, does the paper provide the full set of assumptions and a complete (and correct) proof?
    \item[] Answer: \answerYes{} %
    \item[] Justification: Please see Section~\ref{app:proof_fups_limitations}.
    \item[] Guidelines:
    \begin{itemize}
        \item The answer NA means that the paper does not include theoretical results. 
        \item All the theorems, formulas, and proofs in the paper should be numbered and cross-referenced.
        \item All assumptions should be clearly stated or referenced in the statement of any theorems.
        \item The proofs can either appear in the main paper or the supplemental material, but if they appear in the supplemental material, the authors are encouraged to provide a short proof sketch to provide intuition. 
        \item Inversely, any informal proof provided in the core of the paper should be complemented by formal proofs provided in appendix or supplemental material.
        \item Theorems and Lemmas that the proof relies upon should be properly referenced. 
    \end{itemize}

    \item {\bf Experimental result reproducibility}
    \item[] Question: Does the paper fully disclose all the information needed to reproduce the main experimental results of the paper to the extent that it affects the main claims and/or conclusions of the paper (regardless of whether the code and data are provided or not)?
    \item[] Answer: \answerYes{} %
    \item[] Justification: Please see Sections~\ref{sec:hypermarl} and \ref{alg:hypermarl}. We also provide all the code.
    \item[] Guidelines:
    \begin{itemize}
        \item The answer NA means that the paper does not include experiments.
        \item If the paper includes experiments, a No answer to this question will not be perceived well by the reviewers: Making the paper reproducible is important, regardless of whether the code and data are provided or not.
        \item If the contribution is a dataset and/or model, the authors should describe the steps taken to make their results reproducible or verifiable. 
        \item Depending on the contribution, reproducibility can be accomplished in various ways. For example, if the contribution is a novel architecture, describing the architecture fully might suffice, or if the contribution is a specific model and empirical evaluation, it may be necessary to either make it possible for others to replicate the model with the same dataset, or provide access to the model. In general. releasing code and data is often one good way to accomplish this, but reproducibility can also be provided via detailed instructions for how to replicate the results, access to a hosted model (e.g., in the case of a large language model), releasing of a model checkpoint, or other means that are appropriate to the research performed.
        \item While NeurIPS does not require releasing code, the conference does require all submissions to provide some reasonable avenue for reproducibility, which may depend on the nature of the contribution. For example
        \begin{enumerate}
            \item If the contribution is primarily a new algorithm, the paper should make it clear how to reproduce that algorithm.
            \item If the contribution is primarily a new model architecture, the paper should describe the architecture clearly and fully.
            \item If the contribution is a new model (e.g., a large language model), then there should either be a way to access this model for reproducing the results or a way to reproduce the model (e.g., with an open-source dataset or instructions for how to construct the dataset).
            \item We recognize that reproducibility may be tricky in some cases, in which case authors are welcome to describe the particular way they provide for reproducibility. In the case of closed-source models, it may be that access to the model is limited in some way (e.g., to registered users), but it should be possible for other researchers to have some path to reproducing or verifying the results.
        \end{enumerate}
    \end{itemize}

\item {\bf Open access to data and code}
    \item[] Question: Does the paper provide open access to the data and code, with sufficient instructions to faithfully reproduce the main experimental results, as described in supplemental material?
    \item[] Answer: \answerYes{} %
    \item[] Justification: We provide all the code and scripts to reproduce results.
    \item[] Guidelines:
    \begin{itemize}
        \item The answer NA means that paper does not include experiments requiring code.
        \item Please see the NeurIPS code and data submission guidelines (\url{https://nips.cc/public/guides/CodeSubmissionPolicy}) for more details.
        \item While we encourage the release of code and data, we understand that this might not be possible, so “No” is an acceptable answer. Papers cannot be rejected simply for not including code, unless this is central to the contribution (e.g., for a new open-source benchmark).
        \item The instructions should contain the exact command and environment needed to run to reproduce the results. See the NeurIPS code and data submission guidelines (\url{https://nips.cc/public/guides/CodeSubmissionPolicy}) for more details.
        \item The authors should provide instructions on data access and preparation, including how to access the raw data, preprocessed data, intermediate data, and generated data, etc.
        \item The authors should provide scripts to reproduce all experimental results for the new proposed method and baselines. If only a subset of experiments are reproducible, they should state which ones are omitted from the script and why.
        \item At submission time, to preserve anonymity, the authors should release anonymized versions (if applicable).
        \item Providing as much information as possible in supplemental material (appended to the paper) is recommended, but including URLs to data and code is permitted.
    \end{itemize}

\item {\bf Experimental setting/details}
    \item[] Question: Does the paper specify all the training and test details (e.g., data splits, hyperparameters, how they were chosen, type of optimizer, etc.) necessary to understand the results?
    \item[] Answer: \answerYes{} %
    \item[] Justification: Please see Sections~\ref{sec:experiments} and App.~\ref{append:hyperparams}.
    \item[] Guidelines:
    \begin{itemize}
        \item The answer NA means that the paper does not include experiments.
        \item The experimental setting should be presented in the core of the paper to a level of detail that is necessary to appreciate the results and make sense of them.
        \item The full details can be provided either with the code, in appendix, or as supplemental material.
    \end{itemize}

\item {\bf Experiment statistical significance}
    \item[] Question: Does the paper report error bars suitably and correctly defined or other appropriate information about the statistical significance of the experiments?
    \item[] Answer: \answerYes{} %
    \item[] Justification: Following best practices, we report statistical significance in all results in Section~\ref{sec:experiments}.
    \item[] Guidelines:
    \begin{itemize}
        \item The answer NA means that the paper does not include experiments.
        \item The authors should answer "Yes" if the results are accompanied by error bars, confidence intervals, or statistical significance tests, at least for the experiments that support the main claims of the paper.
        \item The factors of variability that the error bars are capturing should be clearly stated (for example, train/test split, initialization, random drawing of some parameter, or overall run with given experimental conditions).
        \item The method for calculating the error bars should be explained (closed form formula, call to a library function, bootstrap, etc.)
        \item The assumptions made should be given (e.g., Normally distributed errors).
        \item It should be clear whether the error bar is the standard deviation or the standard error of the mean.
        \item It is OK to report 1-sigma error bars, but one should state it. The authors should preferably report a 2-sigma error bar than state that they have a 96\% CI, if the hypothesis of Normality of errors is not verified.
        \item For asymmetric distributions, the authors should be careful not to show in tables or figures symmetric error bars that would yield results that are out of range (e.g. negative error rates).
        \item If error bars are reported in tables or plots, The authors should explain in the text how they were calculated and reference the corresponding figures or tables in the text.
    \end{itemize}

\item {\bf Experiments compute resources}
    \item[] Question: For each experiment, does the paper provide sufficient information on the computer resources (type of compute workers, memory, time of execution) needed to reproduce the experiments?
    \item[] Answer: \answerYes{} %
    \item[] Justification: Please see Section~\ref{appendix:computational_resources}.
    \item[] Guidelines:
    \begin{itemize}
        \item The answer NA means that the paper does not include experiments.
        \item The paper should indicate the type of compute workers CPU or GPU, internal cluster, or cloud provider, including relevant memory and storage.
        \item The paper should provide the amount of compute required for each of the individual experimental runs as well as estimate the total compute. 
        \item The paper should disclose whether the full research project required more compute than the experiments reported in the paper (e.g., preliminary or failed experiments that didn't make it into the paper). 
    \end{itemize}
    
\item {\bf Code of ethics}
    \item[] Question: Does the research conducted in the paper conform, in every respect, with the NeurIPS Code of Ethics \url{https://neurips.cc/public/EthicsGuidelines}?
    \item[] Answer: \answerYes{} %
    \item[] Justification: Read and confirmed.
    \item[] Guidelines:
    \begin{itemize}
        \item The answer NA means that the authors have not reviewed the NeurIPS Code of Ethics.
        \item If the authors answer No, they should explain the special circumstances that require a deviation from the Code of Ethics.
        \item The authors should make sure to preserve anonymity (e.g., if there is a special consideration due to laws or regulations in their jurisdiction).
    \end{itemize}

\item {\bf Broader impacts}
    \item[] Question: Does the paper discuss both potential positive societal impacts and negative societal impacts of the work performed?
    \item[] Answer: \answerYes{} %
    \item[] Justification: See App.~\ref{app:braoder_impact}.
    \item[] Guidelines:
    \begin{itemize}
        \item The answer NA means that there is no societal impact of the work performed.
        \item If the authors answer NA or No, they should explain why their work has no societal impact or why the paper does not address societal impact.
        \item Examples of negative societal impacts include potential malicious or unintended uses (e.g., disinformation, generating fake profiles, surveillance), fairness considerations (e.g., deployment of technologies that could make decisions that unfairly impact specific groups), privacy considerations, and security considerations.
        \item The conference expects that many papers will be foundational research and not tied to particular applications, let alone deployments. However, if there is a direct path to any negative applications, the authors should point it out. For example, it is legitimate to point out that an improvement in the quality of generative models could be used to generate deepfakes for disinformation. On the other hand, it is not needed to point out that a generic algorithm for optimizing neural networks could enable people to train models that generate Deepfakes faster.
        \item The authors should consider possible harms that could arise when the technology is being used as intended and functioning correctly, harms that could arise when the technology is being used as intended but gives incorrect results, and harms following from (intentional or unintentional) misuse of the technology.
        \item If there are negative societal impacts, the authors could also discuss possible mitigation strategies (e.g., gated release of models, providing defenses in addition to attacks, mechanisms for monitoring misuse, mechanisms to monitor how a system learns from feedback over time, improving the efficiency and accessibility of ML).
    \end{itemize}
    
\item {\bf Safeguards}
    \item[] Question: Does the paper describe safeguards that have been put in place for responsible release of data or models that have a high risk for misuse (e.g., pretrained language models, image generators, or scraped datasets)?
    \item[] Answer: \answerNA{} %
    \item[] Justification: 
    \item[] Guidelines:
    \begin{itemize}
        \item The answer NA means that the paper poses no such risks.
        \item Released models that have a high risk for misuse or dual-use should be released with necessary safeguards to allow for controlled use of the model, for example by requiring that users adhere to usage guidelines or restrictions to access the model or implementing safety filters. 
        \item Datasets that have been scraped from the Internet could pose safety risks. The authors should describe how they avoided releasing unsafe images.
        \item We recognize that providing effective safeguards is challenging, and many papers do not require this, but we encourage authors to take this into account and make a best faith effort.
    \end{itemize}

\item {\bf Licenses for existing assets}
    \item[] Question: Are the creators or original owners of assets (e.g., code, data, models), used in the paper, properly credited and are the license and terms of use explicitly mentioned and properly respected?
    \item[] Answer: \answerNA{} %
    \item[] Justification: 
    \item[] Guidelines:
    \begin{itemize}
        \item The answer NA means that the paper does not use existing assets.
        \item The authors should cite the original paper that produced the code package or dataset.
        \item The authors should state which version of the asset is used and, if possible, include a URL.
        \item The name of the license (e.g., CC-BY 4.0) should be included for each asset.
        \item For scraped data from a particular source (e.g., website), the copyright and terms of service of that source should be provided.
        \item If assets are released, the license, copyright information, and terms of use in the package should be provided. For popular datasets, \url{paperswithcode.com/datasets} has curated licenses for some datasets. Their licensing guide can help determine the license of a dataset.
        \item For existing datasets that are re-packaged, both the original license and the license of the derived asset (if it has changed) should be provided.
        \item If this information is not available online, the authors are encouraged to reach out to the asset's creators.
    \end{itemize}

\item {\bf New assets}
    \item[] Question: Are new assets introduced in the paper well documented and is the documentation provided alongside the assets?
    \item[] Answer: \answerNA{} %
    \item[] Justification: 
    \item[] Guidelines:
    \begin{itemize}
        \item The answer NA means that the paper does not release new assets.
        \item Researchers should communicate the details of the dataset/code/model as part of their submissions via structured templates. This includes details about training, license, limitations, etc. 
        \item The paper should discuss whether and how consent was obtained from people whose asset is used.
        \item At submission time, remember to anonymize your assets (if applicable). You can either create an anonymized URL or include an anonymized zip file.
    \end{itemize}

\item {\bf Crowdsourcing and research with human subjects}
    \item[] Question: For crowdsourcing experiments and research with human subjects, does the paper include the full text of instructions given to participants and screenshots, if applicable, as well as details about compensation (if any)? 
    \item[] Answer: \answerNA{} %
    \item[] Justification:
    \item[] Guidelines:
    \begin{itemize}
        \item The answer NA means that the paper does not involve crowdsourcing nor research with human subjects.
        \item Including this information in the supplemental material is fine, but if the main contribution of the paper involves human subjects, then as much detail as possible should be included in the main paper. 
        \item According to the NeurIPS Code of Ethics, workers involved in data collection, curation, or other labor should be paid at least the minimum wage in the country of the data collector. 
    \end{itemize}

\item {\bf Institutional review board (IRB) approvals or equivalent for research with human subjects}
    \item[] Question: Does the paper describe potential risks incurred by study participants, whether such risks were disclosed to the subjects, and whether Institutional Review Board (IRB) approvals (or an equivalent approval/review based on the requirements of your country or institution) were obtained?
    \item[] Answer: \answerNA{} %
    \item[] Justification: 
    \item[] Guidelines:
    \begin{itemize}
        \item The answer NA means that the paper does not involve crowdsourcing nor research with human subjects.
        \item Depending on the country in which research is conducted, IRB approval (or equivalent) may be required for any human subjects research. If you obtained IRB approval, you should clearly state this in the paper. 
        \item We recognize that the procedures for this may vary significantly between institutions and locations, and we expect authors to adhere to the NeurIPS Code of Ethics and the guidelines for their institution. 
        \item For initial submissions, do not include any information that would break anonymity (if applicable), such as the institution conducting the review.
    \end{itemize}

\item {\bf Declaration of LLM usage}
    \item[] Question: Does the paper describe the usage of LLMs if it is an important, original, or non-standard component of the core methods in this research? Note that if the LLM is used only for writing, editing, or formatting purposes and does not impact the core methodology, scientific rigorousness, or originality of the research, declaration is not required.
    \item[] Answer: \answerNA{} %
    \item[] Justification: 
    \item[] Guidelines:
    \begin{itemize}
        \item The answer NA means that the core method development in this research does not involve LLMs as any important, original, or non-standard components.
        \item Please refer to our LLM policy (\url{https://neurips.cc/Conferences/2025/LLM}) for what should or should not be described.
    \end{itemize}

\end{enumerate}

\end{document}